%% file: main_arxiv.tex
\documentclass[10pt]{article}
\usepackage[margin=1in]{geometry}
\usepackage{yub}
\usepackage[round,compress]{natbib}
\usepackage{parskip}
\usepackage{enumitem}
\usepackage{tikz}

\usepackage{algorithm}
\usepackage[noend]{algorithmic}

\newcommand{\alglinelabel}{%
  \addtocounter{ALC@line}{-1}%
  \refstepcounter{ALC@line}%
  \label%
}

\usepackage{makecell}
\definecolor{light-gray}{gray}{0.85}
\usepackage{colortbl}

\allowdisplaybreaks

\newcommand{\D}{{{\rm D}^{\rm bal}}}
\newcommand{\Dbal}{{{\rm D}^{\rm bal}}}
\newcommand{\Dorig}{{{\rm D}}}
\newcommand{\Dkl}{{{\rm D}^{\rm kl}}}

\renewcommand{\setto}{\leftarrow}

\newcommand{\Reg}{\mathfrak{R}}

\newcommand{\cO}{\mc{O}}
\newcommand{\tO}{\wt{\mc{O}}}

\newcommand{\negap}{{\rm NEGap}}
\newcommand{\ccegap}{{\rm CCEGap}}

\newcommand{\md}{\textsc{Hedge}}
\newcommand{\regret}{{\rm Regret}}

\renewcommand{\epsilon}{\eps}
\newcommand{\imm}{{\rm imm}}
\renewcommand{\hat}{\widehat}
\newcommand{\Alg}{{\sf Alg}}
\newcommand{\update}{\textsc{Update}}
\newcommand{\regmatch}{\textsc{RegretMatching}}

\def\POMG{{\rm POMG}}
\def\cS{{\mathcal S}}
\def\cX{{\mathcal X}}
\def\cY{{\mathcal Y}}
\def\cA{{\mathcal A}}
\def\cB{{\mathcal B}}
\def\cF{{\mathcal F}}

\def\l{\ell}
\def\L{L}

\def\cC{{\mathcal C}}

\def\cB{{\mathcal B}}

\renewcommand{\th}{{t, (h)}}
\mathchardef\mhyphen="2D

\hypersetup{
    colorlinks,
    linkcolor={blue!50!black},
    citecolor={blue!50!black},
}
\colorlet{linkequation}{blue}

\def\shownotes{0}  %
\ifnum\shownotes=1
\newcommand{\authnote}[2]{{\scriptsize $\ll$\textsf{#1 notes: #2}$\gg$}}
\else
\newcommand{\authnote}[2]{}
\fi
\newcommand{\yub}[1]{{\color{red}\authnote{Yu}{#1}}}

\title{Near-Optimal Learning of Extensive-Form Games \\
with Imperfect Information}

\author{
Yu Bai\thanks{Salesforce Research. Email: \texttt{yu.bai@salesforce.com}}
  \and
  Chi Jin\thanks{Princeton University. Email: \texttt{chij@princeton.edu}}
  \and 
   Song Mei\thanks{UC Berkeley. Email: \texttt{songmei@berkeley.edu}}
  \and
  Tiancheng Yu\thanks{MIT. Email: \texttt{yutc@mit.edu}}
}

\begin{document}

\maketitle

\input{Sections_arxiv/abstract.tex}

\input{Sections_arxiv/intro.tex}

\input{Sections_arxiv/related-work.tex}

\input{Sections_arxiv/prelim.tex}

\input{Sections_arxiv/ixomd.tex}

\input{Sections_arxiv/cfr.tex}

\input{Sections_arxiv/multi-player.tex}

\input{Sections_arxiv/conclusion.tex}
\input{Sections_arxiv/acknowledgment.tex}

\bibliographystyle{plainnat}
\bibliography{bib}

\appendix

\makeatletter
\def\renewtheorem#1{%
  \expandafter\let\csname#1\endcsname\relax
  \expandafter\let\csname c@#1\endcsname\relax
  \gdef\renewtheorem@envname{#1}
  \renewtheorem@secpar
}
\def\renewtheorem@secpar{\@ifnextchar[{\renewtheorem@numberedlike}{\renewtheorem@nonumberedlike}}
\def\renewtheorem@numberedlike[#1]#2{\newtheorem{\renewtheorem@envname}[#1]{#2}}
\def\renewtheorem@nonumberedlike#1{  
\def\renewtheorem@caption{#1}
\edef\renewtheorem@nowithin{\noexpand\newtheorem{\renewtheorem@envname}{\renewtheorem@caption}}
\renewtheorem@thirdpar
}
\def\renewtheorem@thirdpar{\@ifnextchar[{\renewtheorem@within}{\renewtheorem@nowithin}}
\def\renewtheorem@within[#1]{\renewtheorem@nowithin[#1]}
\makeatother

\renewtheorem{theorem}{Theorem}[section]

\input{Sections_arxiv/tools.tex}

\input{Sections_arxiv/properties.tex}

\input{Sections_arxiv/proof-ixomd.tex}

\input{Sections_arxiv/proof-lower-bound.tex}
\input{Sections_arxiv/proof-cfr.tex}

\input{Sections_arxiv/full-feedback.tex}

\input{Sections_arxiv/regret-matching.tex}
\input{Sections_arxiv/multi-player-full.tex}

\end{document}

%% file: Sections_arxiv/abstract.tex
\begin{abstract}
This paper resolves the open question of designing near-optimal algorithms for learning imperfect-information extensive-form games from bandit feedback. We present the first line of algorithms that require only $\widetilde{\mathcal{O}}((XA+YB)/\varepsilon^2)$ episodes of play to find an $\varepsilon$-approximate Nash equilibrium in two-player zero-sum games, where $X,Y$ are the number of information sets and $A,B$ are the number of actions for the two players. This improves upon the best known sample complexity of $\widetilde{\mathcal{O}}((X^2A+Y^2B)/\varepsilon^2)$ by a factor of $\widetilde{\mathcal{O}}(\max\{X, Y\})$, and matches the information-theoretic lower bound up to logarithmic factors. We achieve this sample complexity by two new algorithms: Balanced Online Mirror Descent, and Balanced Counterfactual Regret Minimization. Both algorithms rely on novel approaches of integrating \emph{balanced exploration policies} into their classical counterparts. We also extend our results to learning Coarse Correlated Equilibria in multi-player general-sum games.
\end{abstract}

%% file: Sections_arxiv/intro.tex
\section{Introduction}

\input{Sections_arxiv/table.tex}

Imperfect Information Games---games where players can only make decisions based on partial information about the true underlying state of the game---constitute an important challenge for modern artificial intelligence. The celebrated notion of Imperfect-Information Extensive-Form games (IIEFGs)~\citep{kuhn201611} offers a formulation for games with both imperfect information and sequential play. IIEFGs have been widely used for modeling real-world imperfect information games such as Poker~\citep{heinrich2015fictitious,moravvcik2017deepstack,brown2018superhuman}, Bridge~\citep{tian2020joint}, Scotland Yard~\citep{schmid2021player}, etc, and achieving strong performances therein.

A central question in IIEFGs is the problem of finding a Nash equilibrium (NE)~\citep{nash1950equilibrium} in a two-player zero-sum IIEFG with perfect recall. There is an extensive line of work for solving this problem with full knowledge of the game (or full feedback), by either reformulating as a linear program~\citep{koller1992complexity,von1996efficient,koller1996efficient}, first-order optimization methods~\citep{hoda2010smoothing,kroer2015faster,kroer2018solving,munos2020fast,lee2021last}, or Counterfactual Regret Minimization~\citep{zinkevich2007regret,lanctot2009monte,johanson2012efficient,tammelin2014solving,schmid2019variance,burch2019revisiting}.

However, in the more challenging bandit feedback setting where the game is not known and can only be learned from random observations by repeated playing, the optimal sample complexity (i.e., the number of episodes required to play) for learning an NE in IIEFGs remains open. Various approaches have been proposed recently for solving this, including model-based exploration~\citep{zhou2019posterior,zhang2021finding}, Online Mirror Descent with loss estimation~\citep{farina2021bandit,kozuno2021model}, and Monte-Carlo Counterfactual Regret Minimization (MCCFR)~\citep{lanctot2009monte,farina2020stochastic,farina2021model}. In a two-player zero-sum IIEFG with $X$, $Y$ information sets (infosets) and $A$, $B$ actions for the two players respectively, the current best sample complexity for learning an $\eps$-approximate NE is $\tO((X^2A+Y^2B)/\eps^2)$ achieved by a sample-based variant of Online Mirror Descent with implicit exploration~\citep{kozuno2021model}. However, this sample complexity scales quadratically in $X$, $Y$ and still has a gap from the information-theoretic lower bound $\Omega((XA+YB)/\eps^2)$ which only scales linearly. This gap is especially concerning from a practical point of view as the number of infosets is often the dominating measure of the game size in large real-world IIEFGs~\citep{johanson2013measuring}.

In this paper, we resolve this open question by presenting the first line of algorithms that achieve $\tO((XA+YB)/\eps^2)$ sample complexity for learning $\eps$-NE in an IIEFG.

Our contributions can be summarized as follows.
\begin{itemize}[wide, topsep=0pt, itemsep=0pt]
\item We design a new algorithm Balanced Online Mirror Descent (Balanced OMD) that achieves $\tO(\sqrt{XAT})$ regret for the max player against adversarial opponents, and learns an $\eps$-NE within $\tO((XA+YB)/\eps^2)$ episodes of play when run by both players in a self-play fashion (Section~\ref{section:ixomd}). These improve over the best existing results by a factor of $\sqrt{X}$ and $\max\{X, Y\}$ respectively, and match the information-theoretic lower bounds (Section~\ref{section:lower-bound}) up to ${\rm poly}(H)$ and logarithmic factors. The main feature within Balanced OMD is a new \emph{balanced dilated KL} as the distance function in its mirror descent step.

\item We design another new algorithm Balanced Counterfactual Regret Minimization (Balanced CFR) that also achieves $\tO((XA+YB)/\eps^2)$ sample complexity for learning an $\eps$-NE (Section~\ref{section:cfr}). Balanced CFR can be seen as an instantiation of the MCCFR framework that integrates the \emph{balanced exploration policies} within both the sampling and the local regret minimization steps.

\item We extend our results to multi-player general-sum IIEFGs, where we show that both Balanced OMD and Balanced CFR can learn an approximate Normal-Form Coarse Correlated Equilibrium (NFCCE) sample-efficiently when run by all players simultaneously via self-play (Section~\ref{section:multi-player}).
\end{itemize}

%% file: Sections_arxiv/table.tex
\begin{table*}[t]
  \centering
      \renewcommand{\arraystretch}{1.5}
\begin{tabular}{|c|c|c|c|}
\hline
         \textbf{Algorithm}                
  &\textbf{OMD} & \textbf{CFR} & \textbf{Sample Complexity} \\ 
\hline
\citet{zhang2021finding}               &    \multicolumn{2}{c|}{- (model-based)}      &       $\widetilde{\mathcal{O}}\left( S^2AB/\varepsilon ^2 \right) $      
\\ 
\hline
\citet{farina2021model} &  &   \checkmark   &        $\tO(\mathrm{poly}\left( X,Y,A,B \right) /\varepsilon ^4)$           \\ \hline
\citet{farina2021bandit}  &   \checkmark     &    &    $\widetilde{\mathcal{O}}\left( \left( X^4A^3+Y^4B^3 \right) /\varepsilon ^2 \right) $               \\ \hline
\citet{kozuno2021model}              &  \checkmark    &     &       $\widetilde{\mathcal{O}}\left( \left( X^2A+Y^2B \right) /\varepsilon ^2 \right)$            \\ \hline
  \cellcolor{light-gray}  Balanced OMD (Algorithm~\ref{alg:IXOMD})            &   \checkmark   &     &     $\widetilde{\mathcal{O}}\left( \left( XA+YB \right) /\varepsilon ^2 \right)$               \\ \hline
  \cellcolor{light-gray}  Balanced CFR (Algorithm~\ref{algorithm:cfr})             &     &    \checkmark  &     $\widetilde{\mathcal{O}}\left( \left( XA+YB \right) /\varepsilon ^2 \right)$               \\ \hline
Lower bound (Theorem~\ref{theorem:lower-bound})          &   -   &   -   &     $\Omega \left( \left( XA+YB \right) /\varepsilon ^2 \right) $              \\ \hline
\end{tabular}
\caption{\label{table:rate} Sample complexity (number of episodes
  required) for learning $\epsilon$-NE in IIEFGs from bandit feedback.
}
\vspace{-1em}
\end{table*}

%% file: Sections_arxiv/related-work.tex
\subsection{Related work}

\paragraph{Computing NE from full feedback}
When the full game (transitions and rewards) is known, the problem of finding the NE is a min-max optimization problem over the two policies. Early works consider casting this min-max problem over the sequence-form policies as a linear program~\citep{koller1992complexity,von1996efficient,koller1996efficient}. First-order algorithms are later proposed for solving the min-max problem directly, in particular by using proper regularizers such as the \emph{dilated KL distance}~\citep{gilpin2008first,hoda2010smoothing,kroer2015faster,lee2021last}.
Another prevalent approach is Counterfactual Regret Minimization (CFR)~\citep{zinkevich2007regret}, which works by minimizing (local) \emph{counterfactual regrets} at each infoset separately using any regret minimization algorithm over the probability simplex such as Regret Matching or Hedge~\citep{tammelin2014solving,burch2019revisiting,zhou2020lazy,farina2020faster}.
As each CFR iteration involves traversing the entire game tree which can be slow or memory-inefficient, techniques based on sampling or approximation have been proposed to address this, such as Monte-Carlo CFR (MCCFR)~\citep{lanctot2009monte,gibson2012efficient,gibson2012generalized,johanson2012efficient,lisy2015online,schmid2019variance}, function approximation of counterfactual values~\citep{waugh2015solving,brown2019deep}, and pruning~\citep{brown2015regret,brown2017dynamic,brown2017reduced}.

\paragraph{Learning NE from bandit feedback}
The MCCFR framework~\citep{lanctot2009monte} provides a first line of approaches for learning an $\eps$-NE in IIEFGs from bandit feedback, by feeding in sample-based unbiased loss estimators to CFR algorithms. 
This framework is then generalized by~\citet{farina2020stochastic} to any regret minimization algorithm (not necessarily CFR). They analyze the concentration between the true regret and the regret on loss estimators, and propose to sample with a ``balanced strategy'' (equivalent to a special case of our balanced exploration policy) to enable a small concentration term. However, they do not bound the regret on loss estimators or give an end-to-end sample complexity guarantee.~\citet{farina2021model} instantiate this framework to give a sample complexity guarantee of $\tO({\rm poly}(X, Y, A, B)/\eps^4)$, by using an exploration rule that favors larger sub-games (similar to our balanced exploration policy but defined through the number of terminal \emph{states} instead of infosets). Our Balanced CFR algorithm (Section~\ref{section:cfr}) can be seen as an instantiation of this framework using a more general balanced exploration policy in both the sampling and the local regret minimization steps.

Another line of work considers sample-based variants of Online Mirror Descent (OMD) algorithms. \citet{farina2021bandit} provide an algorithm with $\tO((X^4A^3+Y^4B^3)/\eps^2)$ sample complexity\footnote{By plugging in an $\tO(X)$ upper bound for the dilated KL distance and optimizing the regret bound by setting $\eta = 1/\sqrt{X^2A^3T}$ in their Theorem 3.} by OMD with an unbiased loss estimator and the dilated KL distance.~\citet{kozuno2021model} propose the IXOMD algorithm that achieves $\widetilde{\mathcal{O}}((X^2A+Y^2B)/\eps^2)$ sample complexity using an implicit-exploration loss estimator. Our Balanced OMD (Section~\ref{section:ixomd}) can be seen as a variant of the IXOMD algorithm by using a new variant of the dilated KL distance.

Finally, \citet{zhou2019posterior,zhang2021finding} propose model-based exploration approaches combined with planning on the estimated models. Specifically,~\citet{zhou2019posterior} use posterior sampling to obtain an $\widetilde{\mathcal{O}}(SAB/\eps^2)$ sample complexity under the Bayesian setting assuming a correct prior.~\citet{zhang2021finding} achieve $\widetilde{\mathcal{O}}(S^2AB/\eps^2)$ sample complexity by constructing confidence bounds for the transition model. Both sample complexities are polynomial in $S$ (the number of underlying states) due to their need of estimating the full model, which could be much higher than ${\rm poly}(X, Y)$. A comparison between the above existing results and ours is given in Table~\ref{table:rate}.

\paragraph{Markov games without tree structure}
A related line of work considers learning equilibria in Markov Games (MGs)~\citep{shapley1953stochastic} with perfect information, but without the tree structure assumed in IIEFGs. Sample-efficient algorithms for learning MGs from bandit feedback have been designed for learning NE in two-player zero-sum MGs either assuming access to a ``simulator'' or certain reachability assumptions, e.g.~\citep{sidford2020solving,zhang2020model,daskalakis2020independent,wei2021last} or in the exploration setting, e.g.~\citep{wei2017online,bai2020provable,xie2020learning,bai2020near,liu2021sharp,chen2021almost,jin2021power,huang2021towards}, as well as learning (Coarse) Correlated Equilibria in multi-player general-sum MGs, e.g.~\citep{liu2021sharp,song2021can,jin2021v,mao2022provably}. As the settings of MGs in these work do not allow imperfect information, these results do not imply results for learning IIEFGs.

%% file: Sections_arxiv/prelim.tex
\section{Preliminaries}
\label{sec:prelim}

We consider two-player zero-sum IIEFGs using the formulation via Partially Observable Markov Games (POMGs), following~\citep{kozuno2021model}. In the following, $\Delta(\cA)$ denotes the probability simplex over a set $\cA$.

\paragraph{Partially observable Markov games}
We consider finite-horizon, tabular, two-player zero-sum Markov Games with partial observability, which can be described as a tuple ${\rm POMG}(H, \cS, \cX, \cY, \cA, \cB, \P, r)$, where
\begin{itemize}[wide, itemsep=0pt, topsep=0pt]
\item $H$ is the horizon length;
\item $\cS=\bigcup_{h\in[H]} \cS_h$ is the (underlying) state space with $|\cS_h|=S_h$ and $\sum_{h=1}^H S_h=S$;
\item $\cX=\bigcup_{h\in[H]} \cX_h$ is the space of information sets (henceforth \emph{infosets}) for the \emph{max-player} with $|\cX_h|=X_h$ and $X\defeq \sum_{h=1}^H X_h$. At any state $s_h\in\cS_h$, the max-player only observes the infoset $x_h=x(s_h)\in\cX_h$, where $x:\cS\to\cX$ is the emission function for the max-player;
\item  $\cY=\bigcup_{h\in[H]} \cY_h$ is the space of infosets for the \emph{min-player} with $\abs{\cY_h}=Y_h$ and $Y\defeq \sum_{h=1}^H Y_h$. An infoset $y_h$ and the emission function $y:\cS\to\cY$ are defined similarly.
\item $\cA$, $\cB$ are the action spaces for the max-player and min-player respectively, with $\abs{\cA}=A$ and $\abs{\cB}=B$\footnote{While this assumes the action space at each infoset have equal sizes, our results can be extended directly to the case where each infoset has its own action space with (potentially) unequal sizes.}.
\item $\P=\{p_0(\cdot)\in\Delta(\cS_1)\} \cup \{p_h(\cdot|s_h, a_h, b_h)\in \Delta(\cS_{h+1})\}_{(s_h, a_h, b_h)\in \cS_h\times \cA\times \cB,~h\in[H-1]}$ are the transition probabilities, where $p_1(s_1)$ is the probability of the initial state being $s_1$, and $p_h(s_{h+1}|s_h, a_h, b_h)$ is the probability of transitting to $s_{h+1}$ given state-action $(s_h, a_h, b_h)$ at step $h$;
\item $r=\set{r_h(s_h, a_h, b_h)\in[0,1]}_{(s_h, a_h, b_h)\in\cS_h\times\cA\times \cB}$ are the (random) reward functions with mean $\wb{r}_h(s_h, a_h, b_h)$. 
\end{itemize}

\paragraph{Policies, value functions}
As we consider partially observability, each player's policy can only depend on the infoset rather than the underlying state. A policy for the max-player is denoted by $\mu=\{\mu_h(\cdot|x_h)\in\Delta(\cA)\}_{h\in[H], x_h\in\cX_h}$, where $\mu_h(a_h|x_h)$ is the probability of taking action $a_h\in\cA$ at infoset $x_h\in\cX_h$. Similarly, a policy for the min-player is denoted by $\nu=\set{\nu_h(\cdot|y_h)\in\Delta(\cB)}_{h\in[H], y_h\in\cY_h}$. A trajectory for the max player takes the form $(x_1, a_1, r_1, x_2, \dots, x_H, a_H, r_H)$, where $a_h\sim \mu_h(\cdot|x_h)$, and the rewards and infoset transitions depend on the (unseen) opponent's actions and underlying state transition.

The overall game value for any (product) policy $(\mu, \nu)$ is denoted by $V^{\mu, \nu}\defeq \E_{\mu, \nu}\brac{ \sum_{h=1}^H r_h(s_h, a_h, b_h) }$. The max-player aims to maximize the value, whereas the min-player aims to minimize the value.

\paragraph{Tree structure and perfect recall}
We use a POMG with tree structure and the perfect recall assumption as our formulation for IIEFGs, following~\citep{kozuno2021model}\footnote{The class of tree-structured, perfece-recall POMGs is able to express all perfect-recall IIEFGs (defined in~\citep{osborne1994course}) that additionally satisfy the \emph{timeability} condition~\citep{jakobsen2016timeability}, a mild condition that roughly requires that infosets for all players combinedly could be partitioned into ordered ``layers", and is satisfied by most real-world games of interest~\citep{kovavrik2019rethinking}. Further, our algorithms can be directly generalized to any perfect-recall IIEFG (not necessarily timeable), as we only require \emph{each player's own game tree to be timeable} (which holds for any perfect-recall IIEFG), similar as existing OMD/CFR type algorithms~\citep{zinkevich2007regret,farina2020stochastic}.~\yub{no section in appendix}.}. We assume that our POMG has a \emph{tree structure}: For any $h$ and $s_h\in\cS_h$, there exists a unique history $(s_1, a_1, b_1, \dots, s_{h-1}, a_{h-1}, b_{h-1})$ of past states and actions that leads to $s_h$. We also assume that both players have \emph{perfect recall}: For any $h$ and any infoset $x_h\in\cX_h$ for the max-player, there exists a unique history $(x_1, a_1, \dots, x_{h-1}, a_{h-1})$ of past infosets and max-player actions that leads to $x_h$ (and similarly for the min-player). We further define $\cC_{h'}(x_h, a_h)\subset\cX_{h'}$ to be the set of all infosets in the $h'$-the step that are reachable from $(x_h, a_h)$, and define $\cC_{h'}(x_h)=\cup_{a_h\in\cA} \cC_{h'}(x_h, a_h)$. Finally, define $\cC(x_h, a_h)\defeq \cC_{h+1}(x_h, a_h)$ as a shorthand for immediately reachable infosets.

With the tree structure and perfect recall, under any product policy $(\mu, \nu)$, the probability of reaching state-action $(s_h, a_h, b_h)$ at step $h$ takes the form
\begin{align}
  \label{equation:reaching-probability-decomposition}
  \P^{\mu, \nu}(s_h, a_h, b_h) = p_{1:h}(s_h) \mu_{1:h}(x_h, a_h) \nu_{1:h}(y_h, b_h),
\end{align}
where we have defined the sequence-form transition probability as
\begin{align*}
  p_{1:h}(s_h)\defeq p_0(s_1)\prod_{h'\le h-1} p_{h'}(s_{h'+1}|s_{h'}, a_{h'}, b_{h'}),
\end{align*}
where $\set{s_{h'}, a_{h'}, b_{h'}}_{h'\le h-1}$ are the histories uniquely determined from $s_h$ by the tree structure, and the \emph{sequence-form policies} as
\begin{align*}
  & \mu_{1:h}(x_h, a_h) \defeq \prod_{h'=1}^h \mu_{h'}(a_{h'}|x_{h'}), \quad\nu_{1:h}(y_h, b_h) \defeq \prod_{h'=1}^h \nu_{h'}(b_{h'}|y_{h'}),
\end{align*}
where $x_{h'}=x(s_{h'})$ and $y_{h'}=y(s_{h'})$ are the infosets for the two players (with $\set{x_{h'}, a_{h'}}_{h\le h-1}$ are uniquely determined by $x_h$ by perfect recall, and similar for $\set{y_{h'}, b_{h'}}_{h\le h-1}$).

We let $\Pi_{\max}$ denote the set of all possible policies for the max player ($\Pi_{\min}$ for the min player). In the sequence form representation, $\Pi_{\max}$ is a convex compact subset of $\R^{XA}$ specified by the constraints $\mu_{1:h}(x_h, a_h)\ge 0$ and $\sum_{a_h\in\cA}\mu_{1:h}(x_h, a_h)=\mu_{1:h-1}(x_{h-1}, a_{h-1})$ for all $(h, x_h, a_h)$, where $(x_{h-1}, a_{h-1})$ is the unique pair of prior infoset and action that reaches $x_h$ (understanding $\mu_0(x_0, a_0)=\mu_0(\emptyset)=1$).

\paragraph{Regret and Nash Equilibrium}
We consider two standard learning goals: Regret and Nash Equilibrium. For the regret, we focus on the max-player, and assume there is an arbitrary (potentially adversarial) opponent as the min-player who may determine her policy $\nu^t$ based on all past information (including knowledge of $\mu^t$) before the $t$-th episode starts. Then, the two players play the $t$-th episode jointly using $(\mu^t,\nu^t)$. The goal for the max-player's is to design policies $\set{\mu^t}_{t=1}^T$ that minimizes the regret against the best fixed policy in hindsight:
\begin{align}
\label{equation:regret}
  \Reg^T \defeq \max_{\mu^\dagger\in\Pi_{\max}} \sum_{t=1}^T \paren{V^{\mu^\dagger, \nu^t} - V^{\mu^t, \nu^t}}.
\end{align}

We say a product policy $(\mu, \nu)$ is an $\epsilon$-approximate Nash equilibrium ($\epsilon$-NE) if
\begin{align*}
  \negap(\mu, \nu) \defeq \max_{\mu^\dagger\in\Pi_{\max}} V^{\mu^\dagger, \nu} - \min_{\nu^\dagger\in\Pi_{\min}} V^{\mu, \nu^\dagger} \le \epsilon,
\end{align*}
i.e. $\mu$ and $\nu$ are each other's $\epsilon$-approximate best response. 

Using online-to-batch conversion, it is a standard result that sublinear regret for both players ensures that the pair of average policies $(\wb{\mu}, \wb{\nu})$ is an approximate NE (see e.g.~\citep[Theorem 1]{kozuno2021model}):
\begin{proposition}[Regret to Nash conversion]
  \label{proposition:online-to-batch}
  For any sequence of policies $\set{\mu^t}_{t=1}^T\in\Pi_{\max}$ and $\set{\nu^t}_{t=1}^T\in\Pi_{\min}$, the average policies $\wb{\mu}\defeq \frac{1}{T}\sum_{t=1}^T \mu^t$ and $\wb{\nu}\defeq \frac{1}{T}\sum_{t=1}^T \nu^t$ (averaged in the sequence form, cf.~\eqref{equation:mu-avg}) satisfy
  \begin{align*}
    \negap(\wb{\mu}, \wb{\nu}) = \frac{\Reg_{\max}^T + \Reg_{\min}^T}{T},
  \end{align*}
  where $\Reg_{\max}^T \defeq \max_{\mu^\dagger\in\Pi_{\max}} \sum_{t=1}^T (V^{\mu^\dagger, \nu^t} - V^{\mu^t, \nu^t})$ and $\Reg_{\min}^T \defeq \max_{\nu^\dagger\in\Pi_{\min}} \sum_{t=1}^T (V^{\mu^t, \nu^t} - V^{\mu^t, \nu^\dagger})$ denote the regret for the two players respectively.
\end{proposition}
Therefore, an approximate NE can be learned by letting both players play some sublinear regret algorithm against each other in a self-play fashion.

\paragraph{Bandit feedback}
Throughout this paper, we consider the interactive learning (exploration) setting with bandit feedback, where the max-player determines the policy $\mu^t$, the opponent determines $\nu^t$ (either adversarially or by running some learning algorithm, depending on the context) unknown to the max-player, and the two players play an episode of the game using policy $(\mu^t, \nu^t)$. The max player observes the trajectory $(x_1^t, a_1^t, r_1^t, \dots, x_H^t, a_H^t, r_H^t)$ of her own infosets and rewards, but not the opponent's infosets, actions, or the underlying states.

\subsection{Conversion to online linear regret minimization}
The reaching probability decomposition~\eqref{equation:reaching-probability-decomposition} implies that the value function $V^{\mu, \nu}$ is \emph{bilinear} in (the sequence form of) $(\mu, \nu)$. Thus, fixing a sequence of opponent's policies $\set{\nu^t}_{t=1}^T$, we have the linear representation
\begin{align*}
  V^{\mu, \nu^t} =
    \sum_{h=1}^H \sum_{(x_h, a_h)\in \cX_h\times\cA} \mu_{1:h}(x_h, a_h)\sum_{s_h\in x_h, b_h\in\cB} p_{1:h}(s_h) \nu^t_{1:h}(y(s_h), b_h) \wb{r}_h(s_h, a_h, b_h).
\end{align*}
Therefore, defining the \emph{loss function} for round $t$ as
\begin{align}
  \quad \ell^t_h(x_h, a_h)    \defeq \sum_{s_h\in x_h, b_h\in\cB} p_{1:h}(s_h) \nu^t_{1:h}(y(s_h), b_h) (1 - \wb{r}_h(s_h, a_h, b_h)) \label{equation:loss}
\end{align}
the regret $\Reg^T$~\eqref{equation:regret} can be written as
\begin{align}
  \Reg^T = \max_{\mu^\dagger\in\Pi_{\max}} \sum_{t=1}^T \<\mu^t - \mu^\dagger, \ell^t\>,
\end{align}
where the inner product $\<\cdot, \cdot\>$ is over the sequence form: 
$\<\mu, \l^t\> \defeq \sum_{h=1}^H \sum_{x_h, a_h} \mu_{1:h}(x_h, a_h) \l^t_h(x_h, a_h)$ for any $\mu\in\Pi_{\max}$.

\subsection{Balanced exploration policy}

Our algorithms make crucial use of the following balanced exploration policies.

\begin{definition}[Balanced exploration policy]
  For any $1\le h\le H$, the (max-player's) \emph{balanced exploration policy for layer $h$}, denoted as $\mu^{\star, h}\in\Pi_{\max}$, is defined as 
  \begin{align}
    \label{equation:balanced-policy}
    \mu^{\star, h}_{h'}(a_{h'} | x_{h'}) \defeq \left\{
    \begin{aligned}
      & \frac{ \abs{\cC_{h}(x_{h'}, a_{h'})} }{ \abs{\cC_h(x_{h'})} },& h'\in\set{1,\dots,h-1}, \\
      &1/A, & h'\in\set{h,\dots,H}.
    \end{aligned}
        \right.
  \end{align}
  In words, at time steps $h'\le h-1$, the policy $\mu^{\star, h}$ plays actions proportionally to their number of descendants \emph{within the $h$-th layer} of the game tree. Then at time steps $h'\ge h$, it plays the uniform policy.
\end{definition}
Note that there are $H$ such balanced policies, one for each layer $h\in[H]$. The balanced policy for layer $h=H$ is equivalent to the balanced strategy of~\citet{farina2020stochastic} (cf. their Section 4.2 and Appendix A.3) which plays actions proportionally to their number descendants within the last (terminal) layer. The balanced policies for layers $h\le H-1$ generalize theirs by also counting the number of descendants within earlier layers. We remark in passing that the key feature of $\mu^{\star, h}$ for our analyses is its \emph{balancing property}, which we state in Lemma~\ref{lemma:balancing}.

\paragraph{Requirement on knowing the tree structure}
The construction of $\mu^{\star, h}$ requires knowing the number of descendants $\abs{\cC_h(x_{h'}, a_{h'})}$, which depends on the \emph{structure\footnote{By this ``structure'' we refer to the parenting structure of the game tree only (which $x_{h+1}$ is reachable from which $(x_h, a_h)$), not the transition probabilities and rewards.} of the game tree} for the max player. Therefore, our algorithms that use $\mu^{\star, h}$ requires knowing this tree structure beforehand. Although there exist algorithms that do not require knowing such tree structure beforehand~\citep{zhang2021finding,kozuno2021model}, this requirement is relatively mild as the structure can be extracted efficiently from just one tree traversal. We also remark our algorithms using the balanced policies do not impose any additional requirements on the game tree, such as the existence of a policy with lower bounded reaching probabilities at all infosets.

%% file: Sections_arxiv/ixomd.tex
\section{Online Mirror Descent}
\label{section:ixomd}

We now present our first algorithm Balanced Online Mirror Descent (Balanced OMD) and its theoretical guarantees.

\subsection{Balanced dilated KL}

At a high level, OMD algorithms work by designing loss estimators (typically using importance weighting) and solving a regularized optimization over the constraint set in each round that involves the loss estimator and a distance function as the regularizer. OMD has been successfully deployed for solving IIEFGs by using various \emph{dilated distance generating functions} over the policy set $\Pi_{\max}$~\citep{kroer2015faster}.

The main ingredient of our algorithm is the \emph{balanced dilated KL}, a new distance measure between policies in IIEFGs.

\begin{definition}[Balanced dilated KL]
The balanced dilated KL distance between two policies $\mu, \nu\in\Pi_{\max}$ is defined as
\begin{align}
  \label{equation:balanced-dilated-kl}
    \D(\mu \| \nu) \defeq \sum_{h=1}^H \sum_{x_h, a_h} \frac{\mu_{1:h}(x_h, a_h)}{\mu^{\star, h}_{1:h}(x_h, a_h)} \log \frac{\mu_h(a_h | x_h)}{\nu_h(a_h | x_h)}.
\end{align}
\end{definition}

The balanced dilated KL is a reweighted version of the \emph{dilated KL} (a.k.a. the \emph{dilated entropy distance-generating function}) that has been widely used for solving IIEFGs~\citep{hoda2010smoothing,kroer2015faster}:
\begin{align}
  \label{equation:dilated-kl}
    \Dorig(\mu \| \nu) = \sum_{h=1}^H \sum_{x_h, a_h} \mu_{1:h}(x_h, a_h) \log \frac{\mu_h(a_h | x_h)}{\nu_h(a_h | x_h)}.
\end{align}
Compared with~\eqref{equation:dilated-kl}, our balanced dilated KL~\eqref{equation:balanced-dilated-kl} introduces an additional reweighting term $1/\mu^{\star, h}_{1:h}(x_h, a_h)$ that depends on the balanced exploration policy $\mu^{\star, h}$~\eqref{equation:balanced-policy}. This reweighting term is in general different for each $(x_h, a_h)$, which at a high level will introduce a balancing effect into our algorithm. An alternative interpretation of the balanced dilated KL can be found in Appendix~\ref{appendix:interpretation-dbal}.

\subsection{Algorithm and theoretical guarantee}

We now describe our Balanced OMD algorithm in Algorithm~\ref{alg:IXOMD}. Our algorithm is a variant of the IXOMD algorithm of~\citet{kozuno2021model} by using the balanced dilated KL. At a high level, it consists of the following steps:

\begin{itemize}[wide, topsep=0pt, itemsep=0pt]
\item Line~\ref{line:omd-play} \&~\ref{line:omd-loss-estimator} (Sampling): Play an episode using policy $\mu^t$ (against the opponent $\nu^t$) and observe the trajectory. Then construct the loss estimator using importance weighting and IX bonus~\citep{neu2015explore}:
  \begin{align}
    \label{equation:loss-estimator}
    \widetilde{\ell}_h^{t}(x_h, a_h) \defeq \frac{ \indic{(x_h^t, a_h^t)=(x_h, a_h) } \cdot (1 - r_h^t)}{\mu_{1:h}^t (x_h, a_h) + \gamma  \mu _{1:h}^{\star ,h}(x_{h},a_{h})}.
  \end{align}
  Note that the IX bonus $\gamma\mu _{1:h}^{\star ,h}(x_{h},a_{h})$ on the denominator makes~\eqref{equation:loss-estimator} a slightly downward biased estimator of the true loss $\ell_h^t(x_h, a_h)$ defined in~\eqref{equation:loss}.
\item Line~\ref{line:omd-update} (Update policy): Update $\mu^{t+1}$ by OMD with loss estimator $\wt{\ell}^t$ and the balanced dilated KL distance function. Due to the sparsity of $\wt{\ell}^t$, this update admits an efficient implementation that updates the conditional form $\mu^t_h(\cdot|x_h)$ at the visited infoset $x_h=x_h^t$ only (described in Algorithm~\ref{algorithm:balanced-omd-implement}).
\end{itemize}

\begin{algorithm}[t]
  \caption{Balanced OMD (max-player)}
  \label{alg:IXOMD}
  \small
  \begin{algorithmic}[1]
    \REQUIRE Learning rate $\eta>0$; IX parameter $\gamma>0$.
    \STATE Initialize $\mu_h^1 (a_h|x_h) \setto 1  / A_h$  for all $(h, x_h, a_h)$.
    \FOR{Episode $t=1, \ldots, T$}
    \STATE Play an episode using $\mu^t$, observe a trajectory
    \begin{align*}
      (x_1^t, a_1^t, r_1^t, \dots, x_H^t, a_H^t, r_H^t).
    \end{align*}
    \alglinelabel{line:omd-play}
    \FOR{$h = H, \ldots, 1$} 
    \STATE Construct loss estimator $\set{\widetilde{\ell}_h^{t}(x_h, a_h)}_{(x_h, a_h)\in\cX_h\times\cA}$ by
    \begin{align*}
      \widetilde{\ell}_h^{t}(x_h, a_h) \setto \frac{ \indic{(x_h^t, a_h^t)=(x_h, a_h) } \cdot (1 - r_h^t)}{\mu_{1:h}^t (x_h, a_h) + \gamma  \mu _{1:h}^{\star ,h}(x_{h},a_{h})}.
    \end{align*}
    \alglinelabel{line:omd-loss-estimator}
    \STATE Update policy \alglinelabel{line:omd-update}
    \begin{align}
      \label{equation:reweighted-update}
      \mu^{t+1} \setto \argmin_{\mu \in \Pi_{\max}} \<\mu, \wt{\ell}^t\> + \frac{1}{\eta}\Dbal(\mu \| \mu^t)
    \end{align}
    using the efficient implementation in Algorithm~\ref{algorithm:balanced-omd-implement}.
    \ENDFOR
    \ENDFOR
  \end{algorithmic}
\end{algorithm}

We are now ready to present the theoretical guarantees for the Balanced OMD algorithm.

\begin{theorem}[Regret bound for Balanced OMD]
  \label{theorem:ixomd-regret}
  Algorithm~\ref{alg:IXOMD} with learning rate $\eta=\sqrt{XA\log A/(H^3T)}$ and IX parameter $\gamma=\sqrt{XA\iota/(HT)}$ achieves the following regret bound with probability at least $1-\delta$:
  \begin{align*}
    \Reg^T \le \cO\paren{\sqrt{H^3XAT\iota}},
  \end{align*}
  where $\iota\defeq \log(3HXA/\delta)$ is a log factor.
\end{theorem}

Letting both players run Algorithm~\ref{alg:IXOMD}, the following corollary for learning NE follows immediately from the regret-to-Nash conversion (Proposition~\ref{proposition:online-to-batch}).
\begin{corollary}[Learning NE using Balanced OMD]
  \label{corollary:ixomd-pac}
  Suppose both players run Algorithm~\ref{alg:IXOMD} (and its min player's version) against each other for $T$ rounds, with choices of $\eta,\gamma$ specified in Theorem~\ref{theorem:ixomd-regret}. Then, for any $\eps>0$, the average policy $(\wb{\mu}, \wb{\nu})=(\frac{1}{T}\sum_{t=1}^T\mu^t, \frac{1}{T}\sum_{t=1}^T\nu^t)$ achieves $\negap(\wb{\mu}, \wb{\nu})\le \eps$ with probability at least $1 - \delta$, as long as the number of episodes
  \begin{align*}
    T \ge \cO\paren{  H^3(XA+YB)\iota/\eps^2 },
  \end{align*}
  where $\iota\defeq \log(3H(XA+YB)/\delta)$ is a log factor.
\end{corollary}
Theorem~\ref{theorem:ixomd-regret} and Corollary~\ref{corollary:ixomd-pac} are the first to achieve $\tO({\rm poly}(H)\cdot \sqrt{XAT})$ regret and $\tO({\rm poly}(H)\cdot (XA+YB)/\eps^2)$ sample complexity for learning an $\eps$-approximate NE for IIEFGs. Notably, the sample complexity scales only linearly in $X$, $Y$ and improves significantly over the best known $\tO((X^2A+Y^2B)/\eps^2))$ achieved by the IXOMD algorithm of~\citep{kozuno2021model} by a factor of $\max\set{X, Y}$.

\paragraph{Overview of techniques}
The proof of Theorem~\ref{theorem:ixomd-regret} (deferred to Appendix~\ref{appendix:proof-ixomd}) follows the usual analysis of OMD algorithms where the key is to bound a distance term and an algorithm-specific ``stability'' like term (cf. Lemma~\ref{lem:regret_sample} and its proof). Compared with existing OMD algorithms using the original dilated KL~\citep{kozuno2021model}, our balanced dilated KL creates a ``balancing effect'' that preserves the distance term (Lemma~\ref{lemma:bound-balanced-dilated-kl}) and shaves off an $X$ factor in the stability term (Lemma~\ref{lem:bound_second_order_sample} \&~\ref{lem:bound_Xi_sum}), which combine to yield a $\sqrt{X}$ improvement in the final regret bound. This $X$ factor improvement in the stability term is the technical crux of the proof, which we do by bounding a certain log-partition function $\log Z_1^t$ using an intricate second-order Taylor expansion argument in lack of a closed-form formula (Lemma~\ref{lem:bound_second_order_sample} \& Appendix~\ref{appendix:proof-regret_sample}).

\subsection{Lower bound}
\label{section:lower-bound}

We accompany our results with information-theoretic lower bounds showing that our $\tO(\sqrt{H^3XAT})$ regret and $\tO(H^3(XA+YB)/\eps^2)$ sample complexity are both near-optimal modulo ${\rm poly}(H)$ and log factors. 

\begin{theorem}[Lower bound for learning IIEFGs]
  \label{theorem:lower-bound}
  For any $A\ge 2$, $H\ge 1$, we have ($c>0$ is an absolute constant)
  \begin{enumerate}[wide,label=(\alph*), topsep=0pt, itemsep=1pt]
  \item (Regret lower bound) For any algorithm that controls the max player and plays policies $\set{\mu^t}_{t=1}^T$ where $T\ge XA$, there exists a game with $B=1$ on which
    \begin{align*}
      \E\brac{\Reg^T} = \E\brac{\max_{\mu^\dagger\in\Pi_{\max}} \sum_{t=1}^T \<\mu^t - \mu^\dagger, \l^t\>} \ge c\cdot \sqrt{XAT}.
    \end{align*}
  \item (PAC lower bound for learning NE) For any algorithm that controls both players and outputs a final policy $(\what{\mu}, \what{\nu})$ with $T$ episodes of play, and any $\eps\in(0,1]$, there exists a game on which the algorithm suffers from $\E\brac{\negap(\what{\mu}, \what{\nu})}\ge \eps$, unless
    \begin{align*}
      T \ge c\cdot (XA+YB)/\eps^2.
    \end{align*}
  \end{enumerate}
\end{theorem}
The proof of Theorem~\ref{theorem:lower-bound} (deferred to Appendix~\ref{appendix:proof-lower-bound}) constructs a hard instance with $X=\Theta(X_H)=\Theta(A^{H-1})$ that is equivalent to $A^H$-armed bandit problems, and follows by a reduction to standard bandit lower bounds\footnote{Alternatively, a lower bound like Theorem~\ref{theorem:lower-bound} can also be shown by expressing stochastic contextual bandits (for the max-player) with $X$ contexts and $A$ actions as IIEFGs, and using the $\Omega(\sqrt{XAT})$ regret lower bound for stochastic contextual bandits e.g.~\citep[Chapter 19.1]{lattimore2020bandit}. This provides hard IIEFG instances with $H=1$ and any value of $X,A\ge 1$.}.  We remark that our lower bounds are tight in $X$ but did not explicitly optimize the $H$ dependence (which is typically lower-order compared to $X$).

%% file: Sections_arxiv/cfr.tex
\section{Counterfactual Regret Minimization}
\label{section:cfr}

Counterfactual Regret Minimization (CFR)~\citep{zinkevich2007regret} is another widely used class of algorithms for solving IIEFGs. In this section, we present a new variant Balanced CFR that also achieves sharp sample complexity guarantees.

Different from OMD, CFR-type algorithms maintain a ``local'' regret minimizer at each infoset $x_h$ that aims to minimize the \emph{immediate counterfactual regret} at that infoset:
\begin{align*}
  \Reg_{h}^{\imm, T}(x_{h})\defeq \max_{\mu\in\Delta(\cA)} \sum_{t=1}^T \<\mu_h^t(\cdot|x_h) - \mu(\cdot), L_h^t(x_h, \cdot)\>,
\end{align*}
where $L_h^t(x_h, a_h)$ is the \emph{counterfactual loss function} 
\begin{align}
  \label{equation:counterfactual-loss}
  \begin{aligned}
    \L_h^t(x_h, a_h) \defeq \l_h^t(x_h, a_h) + \sum_{h'=h+1}^H \sum_{(x_{h'}, a_{h'})\in \cC_{h'}(x_h, a_h)\times \cA} \mu^t_{(h+1):h'}(x_{h'}, a_{h'}) \l_{h'}^t(x_{h'}, a_{h'}).
  \end{aligned}
\end{align}
Controlling all the immediate counterfactual regrets $\Reg_{h}^{\imm, T}(x_{h})$ will also control the overall regret of the game $\Reg^T$, as guaranteed by the \emph{counterfactual regret decomposition}~\citep{zinkevich2007regret} (see also our Lemma~\ref{lemma:cfr-regret-decomposition}).

\subsection{Algorithm description}

Our Balanced CFR algorithm, described in Algorithm~\ref{algorithm:cfr}, can be seen as an instantiation of the Monte-Carlo CFR (MCCFR) framework~\citep{lanctot2009monte} that incorporates the balanced policies in its sampling procedure. Algorithm~\ref{algorithm:cfr} requires regret minimization algorithms $R_{x_h}$ for each $x_h$ as its input, and performs the following steps in each round:
\begin{itemize}[wide, topsep=0pt, itemsep=0pt]
\item Line~\ref{line:cfr-policy-begin}-\ref{line:cfr-loss-estimator} (Sampling): Play $H$ episodes using policies $\set{\mu^{\th}}_{h\in[H]}$, where each $\mu^{\th}=(\mu^{\star, h}_{1:h}\mu^t_{h+1:H})$ is a \emph{concatenation} of the balanced exploration policy $\mu^{\star, h}$ with the current maintained policy $\mu^t$ over time steps.
Then, compute $\wt{\L}_h^t(x_h, a_h)$ by~\eqref{equation:L-estimator} that are importance-weighted unbiased estimators of the true counterfactual loss $L_h^t(x_h, a_h)$ in \eqref{equation:counterfactual-loss}.
\item Line~\ref{line:cfr-md} (Update regret minimizers): For each $(h, x_h)$, send the loss estimators $\{\wt{\L}_h^t(x_h, a)\}_{a\in\cA}$ to the local regret minimizer $R_{x_h}$, and obtain the updated policy $\mu^{t+1}_h(\cdot|x_h)$. 
\end{itemize}

Similar as existing CFR-type algorithms, Balanced CFR has the flexibility of allowing different choices of regret minimization algorithms as $R_{x_h}$. We will consider two concrete instantiations of $R_{x_h}$ as Hedge and Regret Matching in the following subsection.

\subsection{Theoretical guarantee}

To obtain a sharp guarantee for Balanced CFR, we first instantiate $R_{x_h}$ as the Hedge algorithm (a.k.a. Exponential Weights, or mirror descent with the entropic regularizer; cf. Algorithm~\ref{algorithm:md}). Specifically, we let each $R_{x_h}$ be the Hedge algorithm with learning rate $\eta\mu^{\star, h}_{1:h}(x_h, a)$\footnote{Note that this quantity depends on $x_h$ but not $a$.}. With this choice, Line~\ref{line:cfr-md} of Algorithm~\ref{algorithm:cfr} takes the following explicit form:
\begin{align}
  \label{equation:cfr-hedge}
  \mu_h^{t+1}(a|x_h) \propto_a \mu_{h}^t(a|x_{h}) \cdot e^{-\eta \mu^{\star, h}_{1:h}(x_{h}, a) \cdot \wt{\L}^t_{h}(x_{h}, a)}.
\end{align}
We are now ready to present the theoretical guarantees for the Balanced CFR algorithm.
\begin{theorem}[``Regret'' bound for Balanced CFR]
  \label{theorem:cfr}
  Suppose the max player plays Algorithm~\ref{algorithm:cfr} where each $R_{x_h}$ is instantiated as the Hedge algorithm~\eqref{equation:cfr-hedge} with $\eta=\sqrt{XA\iota/(H^3T)}$. Then, the policies $\mu^t$ achieve the following ``regret'' bound with probability at least $1-\delta$: 
  \begin{align*}
    \wt{\Reg}^T \defeq \max_{\mu^\dagger\in\Pi_{\max}} \sum_{t=1}^T \<\mu^t - \mu^\dagger, \ell^t \> \le \cO(\sqrt{H^3XAT\iota}),
  \end{align*}
  where $\iota=\log(10XA/\delta)$ is a log factor. 
\end{theorem}
The $\tO(\sqrt{H^3XAT})$ ``regret'' achieved by Balanced CFR matches that of Balanced OMD. However, we emphasize that the quantity $\wt{\Reg}^T$ is \emph{not strictly speaking a regret}, as it measures performance of the policy $\set{\mu^t}$ \emph{maintained} in the Balanced CFR algorithm, not the sampling policy $\mu^{\th}$ that the Balanced CFR algorithm have \emph{actually played}. Nevertheless, we remark that such a form of ``regret'' bound is the common type of guarantee for all existing MCCFR type algorithms~\citep{lanctot2009monte,farina2020stochastic}.

\begin{algorithm}[t]
  \small
  \caption{Balanced CFR (max-player)}
  \label{algorithm:cfr}
  \begin{algorithmic}[1]
    \REQUIRE Regret minimization algorithm $R_{x_h}$ for all $(h, x_h)$.
    \STATE Initialize policy $\mu_{h}^1(a_{h}|x_{h})\setto 1/A$ for all $(h, x_h, a_h)$.
    \FOR{round $t=1,\dots,T$}
    \FOR{$h=1,\dots,H$}
    \STATE Set policy $\mu^{\th}\setto (\mu^{\star, h}_{1:h}\mu^t_{h+1:H})$. \alglinelabel{line:cfr-policy-begin}
    \STATE Play an episode using $\mu^{\th}\times \nu^t$, observe a trajectory
    \begin{align*}
      (x_{1}^{\th}, a_{1}^{\th}, r_{1}^{\th}, \cdots, x_{H}^{\th}, a_{H}^{\th}, r_{H}^{\th}).
    \end{align*} \alglinelabel{line:cfr-policy-end}
    \STATE Compute loss estimators for all $(h, x_h, a_h)$:
    \begin{equation}
      \label{equation:L-estimator}
      \begin{aligned}
        \wt{\L}^t_{h}(x_{h}, a_{h}) \defeq \frac{ \indic{(x_{h}^\th, a_{h}^\th) = (x_{h}, a_{h})} }{ \mu^{\star, h}_{1:h}(x_{h}, a_{h}) } \paren{ H-h+1 - \sum_{h'=h}^H r_{h'}^\th}.
      \end{aligned}
    \end{equation}
    \alglinelabel{line:cfr-loss-estimator}
    \ENDFOR
    \FOR{all $h\in[H]$ and $x_{h}\in\mc{X}_{h}$}
    \STATE Update the regret minimizer at $x_h$ and obtain policy:
    \begin{align}
      \label{equation:cfr-update}
      \mu^{t+1}_h(\cdot|x_h) \setto R_{x_h}.\update( \{\wt{\L}^t_{h}(x_{h}, a)\}_{a\in\cA} ).
    \end{align}
    \vspace{-1em}
    \alglinelabel{line:cfr-md}
    \ENDFOR
    \ENDFOR
  \end{algorithmic}
\end{algorithm}

\paragraph{Self-play of Balanced CFR}
Balanced CFR can also be turned into a PAC algorithm for learning $\eps$-NE, by letting the two players play Algorithm~\ref{algorithm:cfr} against each other for $T$ rounds of self-play using the following protocol: Within each round, the max player plays policies $\set{\mu^{\th}}_{h=1}^H$ while the min player plays the fixed policy $\nu^t$; then symmetrically the min player plays $\set{\nu^{\th}}_{h=1}^H$ while the max player plays the fixed $\mu^t$. Overall, each round plays $2H$ episodes.

Theorem~\ref{theorem:cfr} directly implies the following corollary for the above self-play algorithm on learning $\eps$-NE, by the regret-to-Nash conversion (Proposition~\ref{proposition:online-to-batch}).
\begin{corollary}[Learning NE using Balanced CFR]
  \label{corollary:cfr-pac}
  Let both players play Algorithm~\ref{algorithm:cfr} in a self-play fashion against each other for $T$ rounds, where each $R_{x_h}$ is instantiated as the Hedge algorithm~\eqref{equation:cfr-hedge} with $\eta$ specified in Theorem~\ref{theorem:cfr}. Then, for any $\eps>0$, the average policy $(\wb{\mu}, \wb{\nu})=(\frac{1}{T}\sum_{t=1}^T\mu^t, \frac{1}{T}\sum_{t=1}^T\nu^t)$ achieves $\negap(\wb{\mu}, \wb{\nu})\le \eps$ with probability at least $1-\delta$, as long as
  \begin{align*}
    T \ge \cO(H^3(XA+YB) \iota/\eps^2),
  \end{align*}
  where $\iota\defeq \log(10(XA+YB)/\delta)$ is a log factor. The total amount of episodes played is at most
  \begin{align*}
    2H\cdot T = \cO(H^4(XA+YB) \iota/\eps^2).
  \end{align*}
\end{corollary}
Corollary~\ref{corollary:cfr-pac} shows that Balanced CFR requires $\tO(H^4(XA+YB)/\eps^2)$ episodes for learning an $\eps$-NE, which is $H$ times larger than Balanced OMD but otherwise also near-optimal with respect to the lower bound (Theorem~\ref{theorem:lower-bound}) modulo an $\tO({\rm poly}(H))$ factor. 
This improves significantly over the current best sample complexity achieved by CFR-type algorithms, which are either ${\rm poly}(X, Y, A, B)/\eps^4$~\citep{farina2021model}, or potentially ${\rm poly}(X, Y, A, B)/\eps^2$ using the MCCFR framework of~\citep{lanctot2009monte,farina2020stochastic} but without any known such instantiation.

\paragraph{Overview of techniques}
The proof of Theorem~\ref{theorem:cfr} (deferred to Appendix~\ref{appendix:proof-theorem-cfr}) follows the usual analysis pipeline for MCCFR algorithms that decomposes the overall regret $\wt{\Reg}_T$ into combinations of immediate counterfactual regrets $\Reg^{\imm, T}_h(x_h)$, and bounds each by regret bounds (of the regret minimizer $R_{x_h}$) plus concentration terms. We adopt a sharp application of this pipeline by using a tight counterfactual regret decomposition (Lemma~\ref{lemma:cfr-regret-decomposition}), as well as using the balancing property of $\mu^{\star, h}$ to bouund both the regret and concentration terms (Lemma~\ref{lemma:cfr-bias1}-\ref{lemma:cfr-regret}). 
We remark that the way Algorithm~\ref{algorithm:cfr} uses the balanced policy $\mu^{\star, h}$ in both the sampling step (by \emph{concatenating} with the current policy $\mu^t$) and as the learning rate (for the Hedge regret minimizer $R_{x_h}$~\eqref{equation:cfr-hedge}) is novel, and required for the above sharp analysis.

We remark that our techniques can also be used for analyzing CFR type algorithms in the full-feedback setting. Concretely, we provide a sharp $\cO(\sqrt{H^3\lone{\Pi_{\max}}\log A\cdot T})$ regret bound for a ``vanilla" CFR algorithm in the full-feedback setting, matching the result of~\citep[Lemma 2]{zhou2020lazy}. For completeness, we provide a statement and proof of this result under our notation in Appendix~\ref{appendix:full-feedback}.

\paragraph{Balanced CFR with Regret Matching}
Many real-world applications of CFR-type algorithms use Regret Matching~\citep{hart2000simple} instead of Hedge as the regret minimizer, due to its practical advantages such as learning-rate free and pruning effects~\citep{tammelin2014solving,burch2019revisiting}. In Appendix~\ref{appendix:cfr-rm}, we show that Balanced CFR instantiated with Regret Matching enjoys $\tO(\sqrt{H^3XA^2T})$ ``regret'' and $\tO((H^4(XA^2+YB^2)/\eps^2)$ sample complexity for learning $\eps$-NE (Theorem~\ref{theorem:cfr-rm} \& Corollary~\ref{corollary:cfr-rm-pac}). The sample complexity is also sharp in $X,Y$, though is $A$ (or $B$) times worse than the Hedge version, which is expected due to the difference between the regret minimizers.

%% file: Sections_arxiv/multi-player.tex
\section{Extension to multi-player games}
\label{section:multi-player}

In this section, we show that our Balanced OMD and Balanced CFR generalize directly to learning Coarse Correlated Equilibria in multi-player general-sum games.

We consider an $m$-player general-sum IIEFG with $X_i$ infosets and $A_i$ actions for the $i$-th player. Let $V_i$ denote the game value (expected cumulative reward) for the $i$-th player. (More formal definitions can be found in Appendix~\ref{appendix:multi-player-definition}.)

\begin{definition}[NFCCE]
A joint policy $\pi$ (for all players) is an $\epsilon$-approximate Normal-Form Coarse Correlated Equilibrium (NFCCE) if
\begin{align*}
  \ccegap(\pi) \defeq \max_{i \in [m]}\Big( \max_{\pi_i^\dagger\in\Pi_{i}}V_i^{\pi_i^\dagger, \pi_{-i}} - V_i^{\pi} \Big) \le \epsilon,
\end{align*}
i.e., no player can gain more than $\eps$ in her own reward by deviating from $\pi$ and playing some other policy on her own.
\end{definition}

We remark that the NFCCE differs from other types of Coarse Correlated Equilibria in the literature such as the EFCCE\footnote{Such distinctions only exist for (Coarse) Correlated Equilibria and not for the NE studied in the previous sections.}~\citep{farina2020coarse}.
Using the known connection between no-regret and NFCCE~\citep{celli2019computing}, we can learn an $\eps$-NFCCE in an multi-player IIEFG sample-efficiently by letting all players run either Balanced CFR or Balanced OMD in a self-play fashion. In the following, we let $\set{\pi_i^t}_{t=1}^T$ denote the policies maintained by player $i$, and $\pi^t\defeq \prod_{i=1}^m \pi_i^t$ denote their joint policy in the $t$-th round.

\begin{theorem}[Learning NFCCE sample-efficiently using Balanced OMD / Balanced CFR]
  \label{theorem:nfcce}
  We have
  \begin{enumerate}[label=(\alph*), wide, topsep=0pt, itemsep=0pt]
  \item (Balanced OMD) Let all players play Algorithm~\ref{alg:IXOMD} for $T$ rounds with learning rate $\eta=\sqrt{X_iA_i\log A_i/(H^3T)}$ and IX parameter $\gamma=\sqrt{X_iA_i\iota/(HT)}$ for the $i$-th player. Then for any $\eps>0$, the average policy $\wb{\pi}$ uniformly sampled from $\{\pi^t\}_{t=1}^T $ satisfies $\ccegap(\wb{\pi})\le \eps$ with probability at least $1-\delta$, as long as the number of episodes
    \begin{align*}
      T \ge \cO \Big(  H^3\iota \Big(\max_{i\in[m]}X_iA_i\Big)/\varepsilon ^2 \Big),
    \end{align*}
    where $\iota\defeq \log(3H\sum_{i=1}^m{X_iA_i}/\delta )$ is a log factor.
  \item (Balanced CFR) Let all players play Algorithm~\ref{algorithm:cfr} in the same self-play fashion as Corollary~\ref{corollary:cfr-pac} for $T$ rounds, with $R_{x_h}$ instantiated as Hedge~\eqref{equation:cfr-hedge} with learning rate $\eta=\sqrt{X_iA_i \iota/(H^3T)}$ for the $i$-th player. Then for any $\eps>0$, 
the average policy $\wb{\pi}$ uniformly sampled from $\{\pi^t\}_{t=1}^T $  satisfies $\ccegap(\wb{\pi})\le \eps$ with probability at least $1-\delta$, as long as $T \ge \cO\paren{  H^3\iota (\max _{i\in[m]}X_iA_i)/\varepsilon ^2}$. The total number of episodes played is at most
    \begin{align*}
      mH\cdot T = \cO\Big( H^4m\iota \cdot \Big(\max_{i\in[m]} X_iA_i\Big)/\eps^2\Big).
    \end{align*}
    where $\iota\defeq \log(10\sum_{i=1}^m{X_iA_i}/\delta )$ is a log factor.
  \end{enumerate}
\end{theorem}
For both algorithms, the number of episodes for learning an $\eps$-NFCCE scales linearly with $\max_{i\in[m]} X_iA_i$ (with Balanced CFR having an additional $Hm$ factor than Balanced OMD), compared to the best existing $\max_{i\in[m]} X_i^2A_i$ dependence (e.g. by self-playing IXOMD~\citep{kozuno2021model}). The proof of Theorem~\ref{theorem:nfcce} is in Appendix~\ref{appendix:proof-nfcce}.

%% file: Sections_arxiv/conclusion.tex
\section{Conclusion}
This paper presents the first line of algorithms for learning an $\eps$-NE in two-player zero-sum IIEFGs with near-optimal $\tO((XA+YB)/\eps^2)$ sample complexity. We achieve this by new variants of both OMD and CFR type algorithms that incorporate suitable balanced exploration policies. We believe our work opens up many interesting future directions, such as empirical verification of our balanced algorithms, or how to learn IIEFGs with large state/action spaces efficiently using function approximation. 

%% file: Sections_arxiv/acknowledgment.tex
\section*{Acknowledgment}
The authors would like to thank Ziang Song for the helpful discussions. TY is partially supported by NSF CCF-2112665 (TILOS AI Research Institute).

%% file: Sections_arxiv/tools.tex
\section{Technical tools}

The following Freedman's inequality can be found in~\citep[Lemma 9]{agarwal2014taming}.
\begin{lemma}[Freedman's inequality]
  \label{lemma:freedman}
  Suppose random variables $\set{X_t}_{t=1}^T$ is a martingale difference sequence, i.e. $X_t\in\cF_t$ where $\set{\cF_t}_{t\ge 1}$ is a filtration, and $\E[X_t|\cF_{t-1}]=0$. Suppose $X_t\le R$ almost surely for some (non-random) $R>0$. Then for any $\lambda\in(0, 1/R]$, we have with probability at least $1-\delta$ that
  \begin{align*}
    \sum_{t=1}^T X_t \le \lambda \cdot \sum_{t=1}^T \E\brac{X_t^2 | \cF_{t-1}} + \frac{\log(1/\delta)}{\lambda}.
  \end{align*}
\end{lemma}

\section{Bounds for regret minimizers}

Here we collect regret bounds for various regret minimization algorithms on the probability simplex. For any algorithm that plays policy $p_t$ in the $t$-th round and observes loss vector $\set{\ell_t(a)}_{a\in[A]}\in\R_{\ge 0}^A$, define its regret as
\begin{align*}
  \regret(T) \defeq \max_{p^\star\in\Delta([A])} \sum_{t=1}^T \<p_t, \wt{\ell}_t\> - \<p^\star, \wt{\ell}_t\>.
\end{align*}

\subsection{Hedge}
\begin{algorithm}[h]
  \caption{Regret Minimization with Hedge (\md)}
  \label{algorithm:md}
  \begin{algorithmic}[1]
    \REQUIRE Learning rate $\eta>0$.
    \STATE Initialize $p_1(a)\setto 1/A$ for all $a\in[A]$.
    \FOR{iteration $t=1,\dots,T$}
    \STATE Receive loss vector $\set{\wt{\ell}_t(a)}_{a\in[A]}$.
    \STATE Update action distribution via mirror descent:
    \begin{align*}
      p_{t+1}(a) \propto_a p_t(a) \exp\paren{-\eta \wt{\ell}_t(a)}.
    \end{align*}
    \ENDFOR
  \end{algorithmic}
\end{algorithm}

The following regret bound for Hedge is standard, see, e.g.~\citep[Proposition 28.7]{lattimore2020bandit}.
\begin{lemma}[Regret bound for Hedge]
  \label{lemma:md}
  Algorithm~\ref{algorithm:md} with learning rate $\eta>0$ achieves regret bound
  \begin{align*}
    \regret(T) \le \frac{\log A}{\eta} + \frac{\eta}{2}\sum_{t=1}^T \sum_{a\in[A]} p_t(a) \wt{\ell}_t(a)^2.
  \end{align*}
\end{lemma}

\subsection{Regret Matching}

\begin{algorithm}[h]
  \caption{Regret Minimization with Regret Matching (\regmatch)}
  \label{algorithm:regmatch}
  \begin{algorithmic}[1]
    \STATE Initialize $p_1(a)\setto 1/A$ and $R_0(a)\setto 0$ for all $a\in[A]$.
    \FOR{iteration $t=1,\dots,T$}
    \STATE Receive loss vector $\set{\wt{\ell}_t(a)}_{a\in[A]}$.
    \STATE Update instantaneous regret and cumulative regret for all $a\in[A]$:
    \begin{align*}
      r_t(a) \setto \<p_t, \wt{\ell}_t\> - \wt{\ell}_t(a)~~~{\rm and}~~~
      R_t(a)\setto R_{t-1}(a)+r_t(a).
    \end{align*}
    \STATE Compute action distribution by regret matching:
    \begin{align*}
      p_{t+1}(a) \setto \frac{\brac{R_t(a)}_+}{ \sum_{a'\in[A]}\brac{R_t(a')}_+ } = \frac{\brac{ \sum_{t=1}^T \<p_t, \tilde \ell_t\> - \wt{\ell}_t(a) }_+}{ \sum_{a'\in[A]} \brac{ \sum_{t=1}^T \<p_t, \tilde \ell_t\> - \wt{\ell}_t(a') }_+ }.
    \end{align*}
    In the edge case where $\brac{R_t(a)}_+=0$ for all $a\in[A]$, set $p_{t+1}(a)\setto 1/A$ to be the uniform distribution.
    \ENDFOR
  \end{algorithmic}
\end{algorithm}

The following regret bound for Regret Matching is standard, see, e.g.~\citep{cesa2006prediction,brown2014regret}. For completeness, here we provide a proof along with an alternative form of bound useful for our purpose (Remark~\ref{rmk:regmatch}). Note that here $\eta$ is not the learning rate but rather an arbitrary positive value (i.e. the right-hand side is an upper bound on the regret for any $\eta>0$). Algorithm~\ref{algorithm:regmatch} itself does not require any learning rate.
\begin{lemma}[Regret bound for Regret Matching]
  \label{lemma:regmatch}
  Algorithm~\ref{algorithm:regmatch} achieves the following regret bound for \emph{any} $\eta>0$:
  \begin{align*}
    \regret(T) \le \Big[\sum_{t=1}^T \sum_{a\in[A]} \Big( \< p_t, \wt{\ell}_t \> - \wt{\ell}_t(a) \Big)^2 \Big]^{1/2} \le \frac{1}{\eta} + \frac{\eta}{4} \sum_{t=1}^T \sum_{a \in [A]} \Big( \< p_t, \wt{\ell}_t \> - \wt{\ell}_t(a) \Big)^2.
  \end{align*}
\end{lemma}
\begin{proof}
  By the fact that $(a + b)_+^2 \le a_+^2 + 2 a_+ b + b^2$, we have 
\begin{equation}\label{eqn:regmatch_eq1}
[R_{t}(a)]_+^2 \le [R_{t-1}(a)]_+^2 + 2 [R_{t-1}(a)]_+ r_t(a) + r_t(a)^2. 
\end{equation}
Then by the definition of $p_t(a)$ and $r_t(a)$, we have 
\begin{equation}\label{eqn:regmatch_eq2}
\begin{aligned}
\sum_{a \in [A]} [R_{t-1}(a)]_+ r_t(a) =&~ \sum_{a \in [A]} [R_{t-1}(a)]_+ \Big( \sum_{a' \in [A]} p_t(a') \wt{\ell}_t(a') - \wt{\ell}_t(a) \Big) \\
=&~ \sum_{a \in [A]} [R_{t-1}(a)]_+ \wt{\ell}_t(a) - \sum_{a \in [A]} [R_{t-1}(a)]_+ \wt{\ell}_t(a) = 0.
\end{aligned}
\end{equation}
Then summing over $a$ in Eq. (\ref{eqn:regmatch_eq1}) and using Eq. (\ref{eqn:regmatch_eq2}), we get 
\[
\begin{aligned}
\sum_{a \in [A]} [R_T(a)]_+^2 \le&~ \sum_{a \in [A]} [R_{T-1}(a)]_+^2 + 2 \sum_{a \in [A]} [R_{T-1}(a)]_+ r_T(a) + \sum_{a \in [A]} r_T(a)^2 \\
=&~ \sum_{a \in [A]} [R_{T-1}(a)]_+^2 + \sum_{a \in [A]} r_T(a)^2 \le \sum_{t = 1}^T \sum_{a \in [A]} r_t(a)^2. 
\end{aligned}
\]
Using that $\max_a R_T(a) \le \max_a [R_T(a)]_+ \le (\sum_{a \in [A]} [R_T(a)]_+^2)^{1/2}$ gives the regret bound
\begin{align*}
  \regret(T) = \max_{a\in[A]} R_T(a) \le \paren{\sum_{t = 1}^T \sum_{a \in [A]} r_t(a)^2}^{1/2} = \paren{\sum_{t = 1}^T \sum_{a \in [A]} \paren{\<p_t, \wt{\ell}_t\> - \wt{\ell}_t(a)}^2}^{1/2}.
\end{align*}
The claimed bound with $\eta$ follows directly from the inequality $\sqrt{z}\le 1/\eta + \eta z/4$ for any $\eta>0$, $z\ge 0$.
\end{proof}

\begin{remark} \label{rmk:regmatch}

The quantity $\sum_{a\in[A]}\paren{\<p_t, \wt{\ell}_t\> - \wt{\ell}_t(a)}^2$ above can be upper bounded as
\begin{align*}
  & \quad \sum_{a \in [A]} \Big( \< p_t, \wt{\ell}_t \> - \wt{\ell}_t(a) \Big)^2 \le \sum_{a\in[A]} \paren{ \<p_t, \wt{\ell}_t\>^2 + \wt{\ell}_t(a)^2 } \\
  & = A\<p_t, \wt{\ell}_t\>^2 + \ltwo{\wt{\ell}_t}^2 \le A\sum_{a\in[A]} \paren{ p_t(a)\wt{\ell}_t(a)^2 + (1/A)\wt{\ell}_t(a)^2 } \\
  & = 2 A \sum_{a \in [A]} \bar p_t(a) \wt{\ell}_t(a)^2, 
\end{align*}
where $\bar p_t(a) = [p_t(a) + (1/A)]/2$ is a probability distribution over $[A]$.

As a consequence, we get an upper bound on the regret of Regret Matching algorithm by
\[
\regret(T) \le \frac{1}{\eta} + \frac{\eta}{2}\sum_{t=1}^T \sum_{a\in[A]} (A \bar p_t(a)) \wt{\ell}_t(a)^2. 
\]
Comparing to the bound of Hedge (Lemma~\ref{lemma:md}), the above regret bound for Regret Matching has a similar form except for replacing $\log A$ by $1$ and replacing $p_t$ by $A \bar p_t$.
\end{remark}

%% file: Sections_arxiv/properties.tex
\section{Properties of the game}

\subsection{Basic properties}
For any opponent (min-player) policy $\nu\in\Pi_{\min}$, define
\begin{align*}
  p^{\nu}_{1:h}(x_h) \defeq \sum_{s_h\in x_h} p_{1:h}(s_h) \nu_{1:h-1}(y(s_{h-1}), b_{h-1})~~~\textrm{for all}~h\in[H],~x_h\in\cX_h.
\end{align*}
Intuitively, $p^\nu_{1:h}(x_h)$ measures the environment and the opponent's contribution in the reaching probability of $x_h$. 
\begin{lemma}[Properties of $p^\nu_{1:h}(x_h)$]
  \label{lemma:pnu}
  The following holds for any $\nu\in\Pi_{\min}$:
  \begin{enumerate}[label=(\alph*)]
  \item For any policy $\mu\in\Pi_{\max}$, we have
    \begin{align*}
      \sum_{(x_h, a_h)\in\cX_h\times \cA} \mu_{1:h}(x_h, a_h) p^\nu_{1:h}(x_h) = 1.
    \end{align*}
  \item $0\le p^\nu_{1:h}(x_h)\le 1$ for all $h, x_h$.
  \end{enumerate}
\end{lemma}
\begin{proof}
  For (a), notice that
  \begin{align*}
    & \quad \mu_{1:h}(x_h, a_h)p^\nu_{1:h}(x_h) = \sum_{s_h\in x_h} p_{1:h}(s_h) \cdot \mu_{1:h}(x_h, a_h) \cdot \nu_{1:h-1}(y(s_{h-1}), b_{h-1}) \\
    & = \sum_{s_h\in x_h} \P^{\mu, \nu}\paren{{\rm visit}~(s_h, a_h)} = \P^{\mu, \nu}\paren{{\rm visit}~(x_h, a_h)}.
  \end{align*}
  Summing over all $(x_h, a_h)\in\cX_h\times \cA$, the right hand side sums to one, thereby showing (a).

  For (b), fix any $x_h\in\cX_h$. Clearly $p^\nu_{1:h}(x_h)\ge 0$. Choose any $a_h\in\cA$, and choose policy $\mu^{x_h, a_h}\in\Pi_{\max}$ such that $\mu^{x_h, a_h}_{1:h}(x_h, a_h)=1$ (such $\mu^{x_h, a_h}$ exists, for example, by deterministically taking all actions prescribed in infoset $x_h$ at all ancestors of $x_h$). For this $\mu^{x_h, a_h}$, using (a), we have
  \begin{align*}
    p^\nu_{1:h}(x_h) = \mu^{x_h, a_h}_{1:h}(x_h, a_h) \cdot p^\nu_{1:h}(x_h) \le \sum_{(x_h', a_h')\in\cX_h\times \cA} \mu^{x_h, a_h}_{1:h}(x_h', a_h') \cdot p^\nu_{1:h}(x_h') = 1.
  \end{align*}
  This shows part (b).
\end{proof}

\begin{corollary}
\label{cor:averge_loss_bound}
For any policy $\mu\in\Pi_{\max}$ and $h \in [H]$, we have
$$
\sum_{(x_h, a_h)\in\cX_h\times \cA} \mu_{1:h}(x_h, a_h) \ell_h^t(x_h, a_h) \le 1.
$$
\end{corollary}
\begin{proof}
  Notice by definition
  $$
   \ell^t_h(x_h, a_h) = \sum_{s_h\in x_h, b_h\in\cB_h} p_{1:h}(s_h) \nu^t_{1:h}(y(s_h), b_h) (1 - r_h(s_h, a_h, b_h)) \le p^\nu_{1:h}(x_h),
  $$
  and the result is implied by Lemma~\ref{lemma:pnu} (b).
\end{proof}

\begin{lemma}
  \label{lemma:counterfactual-loss-bound}
  For any $h\in[H]$, the counterfactual loss function $\L_h^t$ defined in~\eqref{equation:counterfactual-loss} satisfies the bound
  \begin{enumerate}[label=(\alph*)]
  \item For any policy $\mu\in\Pi_{\max}$, we have
    \begin{align*}
      \sum_{(x_h, a_h)\in\cX_h\times\cA} \mu_{1:h}(x_h, a_h) \L_h^t(x_h, a_h) \le H-h+1.
    \end{align*}
  \item For any $(h, x_h, a_h)$, we have
  \begin{align*}
    0\le \L_h^t(x_h, a_h) \le p^{\nu^t}_{1:h}(x_h)\cdot (H-h+1).
  \end{align*}
\end{enumerate}
\end{lemma}
\begin{proof}
  Part (a) follows from the fact that
  \begin{align*}
    \sum_{(x_h, a_h)\in\cX_h\times\cA} \mu_{1:h}(x_h, a_h) \L_h^t(x_h, a_h) = \E_{\mu, \nu^t}\brac{ \sum_{h'=h}^H r_{h'} } \le H-h+1,
  \end{align*}
  where the first equality follows from the definition of the loss functions $\l_h$ and $\L_h$ in~\eqref{equation:loss},~\eqref{equation:counterfactual-loss}.

  For part (b), the nonnegativity follows clearly by definition. For the upper bound, take any policy $\mu^{x_h, a_h}\in\Pi_{\max}$ such that $\mu^{x_h, a_h}_{1:h}(x_h, a_h)=1$. We then have
  \begin{align*}
    & \quad L_h^t(x_h, a_h) = \mu^{x_h, a_h}_{1:h}(x_h, a_h) L_h^t(x_h, a_h) = \E_{\mu^{x_h, a_h}, \nu^t}\brac{ \indic{{\rm visit}~x_h, a_h} \cdot \sum_{h'=h}^H r_{h'} } \\
    & = \P_{\mu^{x_h, a_h}, \nu^t}\paren{{\rm visit}~x_h, a_h} \cdot \E_{\mu^{x_h, a_h}, \nu^t}\brac{ \sum_{h'=h}^H r_{h'} \bigg| {\rm visit}~x_h, a_h} \\
    & \le \mu^{x_h, a_h}_{1:h}(x_h, a_h) p^{\nu^t}_{1:h}(x_h) \cdot (H-h+1) = p^{\nu^t}_{1:h}(x_h) \cdot (H-h+1).
  \end{align*}
\end{proof}

\paragraph{Definition of average policies}

For two-player zero-sum IIEFGs, we define the average policy of the max-player $\wb{\mu} = \frac{1}{T}\sum_{t=1}^T{\mu ^t}$ (in conditional form) by 
\begin{align}
\wb{\mu}_h(a_h|x_h) \defeq \frac{\sum_{t=1}^T{\mu _{1:h}^{t}\left( x_h,a_h \right)}}{\sum_{t=1}^T{\mu _{1:h-1}^{t}\left( x_h \right)}},
\end{align}
for any $h$ and $(x_h,a_h) \in \cX_h \times \cA$. It is straightforward to check that this $\wb{\mu}$ is exactly the averaging of $\mu^t$ in the sequence-form representation (see e.g.~\citep[Theorem 1]{kozuno2021model}):
\begin{align}
\label{equation:mu-avg}
\wb{\mu}_{1:h}(x_h,a_h)=\frac{1}{T}\sum_{t=1}^T \mu^t_{1:h}(x_h,a_h)~~~\textrm{for all}~(h, x_h, a_h).
\end{align}
Both expressions above can be used as the definition interchangably.
The average policy of the min-player $\wb{\nu} = \frac{1}{T}\sum_{t=1}^T{\nu ^t}$ is defined similarly.

\subsection{Balanced exploration policy}
\label{appendix:proof-balancing}

\begin{lemma}[Balancing property of $\mu^{\star, h}$]
  \label{lemma:balancing}
  For any max-player's policy $\mu\in\Pi_{\max}$ and any $h\in[H]$, we have
  \begin{align*}
    \sum_{(x_h, a_h)\in \mc{X}_h\times \mc{A}} \frac{\mu_{1:h}(x_h, a_h)}{\mu^{\star, h}_{1:h}(x_h, a_h)} = X_hA.
  \end{align*}
\end{lemma}
Lemma~\ref{lemma:balancing} states that $\mu^{\star, h}$ is a good exploration policy in the sense that the distribution mismatch between it and \emph{any} $\mu\in\Pi_{\max}$ has bounded $L_1$ norm. Further, the bound $X_hA$ is non-trivial---For example, if we replace $\mu^{\star, h}_{1:h}$ with the uniform policy $\mu^{\rm unif}_{1:h}(x_h, a_h)=1/A^h$, the left-hand side can be as large as $X_hA^h$ in the worst case.

\begin{proof-of-lemma}[\ref{lemma:balancing}]
  We have
  \begin{align*}
    & \quad \sum_{x_h, a_h} \frac{\mu_{1:h}(x_h, a_h)}{\mu^{\star, h}_{1:h}(x_h, a_h)} \\
    & = \sum_{x_{h-1}, a_{h-1}} \sum_{(x_h, a_h)\in \cC(x_{h-1}, a_{h-1})\times \cA} \frac{\mu_{1:(h-1)}(x_{h-1}, a_{h-1}) \cdot \mu_h(a_h|x_h)}{\mu^{\star, h}_{1:(h-1)}(x_{h-1}, a_{h-1}) \cdot (1/A)} \\
    & \stackrel{(i)}{=} A \cdot \sum_{x_{h-1}, a_{h-1}} \sum_{x_h\in \cC(x_{h-1}, a_{h-1})} \frac{\mu_{1:(h-1)}(x_{h-1}, a_{h-1}) }{\mu^{\star, h}_{1:(h-1)}(x_{h-1}, a_{h-1})} \\
    & = A \cdot \sum_{x_{h-1}, a_{h-1}} \frac{\mu_{1:(h-1)}(x_{h-1}, a_{h-1}) }{\mu^{\star, h}_{1:(h-1)}(x_{h-1}, a_{h-1})} \cdot \abs{\cC_h(x_{h-1}, a_{h-1})} \\
    & \stackrel{(ii)}{=} A \cdot \sum_{x_{h-2}, a_{h-2}}\sum_{(x_{h-1}, a_{h-1})\in \cC(x_{h-2}, a_{h-2})\times \cA} \frac{\mu_{1:(h-2)}(x_{h-2}, a_{h-2})\mu_{h-1}(a_{h-1}|x_{h-1}) }{\mu^{\star, h}_{1:(h-2)}(x_{h-2}, a_{h-2}) \cdot \abs{\cC_h(x_{h-1}, a_{h-1})} / \abs{\cC_h(x_{h-1})}} \cdot \abs{\cC_h(x_{h-1}, a_{h-1})} \\
    & = A \cdot \sum_{x_{h-2}, a_{h-2}}\sum_{(x_{h-1}, a_{h-1})\in \cC(x_{h-2}, a_{h-2})\times \cA} \frac{\mu_{1:(h-2)}(x_{h-2}, a_{h-2})\mu_{h-1}(a_{h-1}|x_{h-1}) }{\mu^{\star, h}_{1:(h-2)}(x_{h-2}, a_{h-2}) } \cdot \abs{\cC_h(x_{h-1})} \\
    & = A \cdot \sum_{x_{h-2}, a_{h-2}}\sum_{(x_{h-1}, a_{h-1})\in \cC(x_{h-2}, a_{h-2})\times \cA} \frac{\mu_{1:(h-2)}(x_{h-2}, a_{h-2})\mu_{h-1}(a_{h-1}|x_{h-1}) }{\mu^{\star, h}_{1:(h-2)}(x_{h-2}, a_{h-2}) } \cdot \abs{\cC_h(x_{h-1})} \\
    & \stackrel{(iii)}{=} A \cdot \sum_{x_{h-2}, a_{h-2}} \frac{\mu_{1:(h-2)}(x_{h-2}, a_{h-2})}{\mu^{\star, h}_{1:(h-2)}(x_{h-2}, a_{h-2}) } \cdot \abs{\cC_h(x_{h-2}, a_{h-2})} \\
    & = \dots \\
    & = A\cdot \sum_{x_1, a_1} \frac{\mu_1(a_1|x_1)}{\abs{\cC_h(x_1, a_1)} / \abs{\cC_h(x_1)}} \cdot \abs{\cC_h(x_1, a_1)} \\
    & = A\cdot \sum_{x_1, a_1} \mu_1(a_1|x_1) \cdot \abs{\cC_h(x_1)} \\
    & = A\cdot \sum_{x_1} \abs{\cC_h(x_1)} = A\cdot \abs{\cC_h(\emptyset)} = X_hA.
  \end{align*}
  Above, (i) used the definition of $\mu^{\star,h}_h$ and the fact that $\sum_{a_h\in\mc{A}} \mu_h(a_h|x_h)=1$ for any $\mu$, $x_h$; (ii) used the definition of $\mu^{\star, h}_{h-1}$; (iii) used the fact that $\sum_{x_{h-1}\in \cC(x_{h-2}, a_{h-2})}  \abs{\cC_h(x_{h-1})} = \abs{\cC_h(x_{h-2}, a_{h-2})}$ which follows by the additivity of the number of descendants; and the rest followed by performing the same operations repeatedly.
\end{proof-of-lemma}

The following corollary is similar to the lower bound in~\citep[Appendix A.3]{farina2020stochastic}.
\begin{corollary}
  \label{cor:mu_star_lower_bound}
  We have
  \begin{align*}
    \mu^{\star, h}_{1:h}(x_h, a_h) \ge \frac{1}{X_hA}
  \end{align*}
  for any $h\in[H]$ and $(x_h,a_h)\in\cX_h\times\cA$.
\end{corollary}
\begin{proof}
  Choose some deterministic policy $\mu$ s.t. $\mu_{1:h}(x_h,a_h)=1$ in Lemma~\ref{lemma:balancing} and noticing each term in the summation is non-negative,
  $$
  \frac{\mu_{1:h}(x_h, a_h)}{\mu^{\star, h}_{1:h}(x_h, a_h)}  \le X_hA.
  $$
\end{proof}

\subsubsection{Interpretation as a transition probability}
\label{appendix:interpretation}
We now provide an intepretation of the balanced exploration policy $\mu^{\star, h}_{1:h}$: its inverse $1/\mu^{\star, h}_{1:h}$ can be viewed as the (product) of a ``transition probability'' over the game tree for the max player. As a consequence, this interpretation also provides an alternative proof of Lemma~\ref{lemma:balancing}.

For any $1 \le h \le H$ and $1 \le k \le h - 1$, denote $p^{\star, h}_k(x_{k+1} \vert x_k, a_k) = \vert \cC_h (x_{k+1}) \vert  / 
\vert \cC_h(x_{k}, a_k) \vert$ (we use the convention that $\vert \cC_h(x_h) \vert = 1$). By this definition, $p^{\star, h}_k(\cdot \vert x_k, a_k)$ is a probability distribution over $\cC_h(x_k, a_k)$ and can be interpreted as a balanced transition probability from $(x_k, a_k)$ to $x_{k+1}$. We further denote the sequence form of the balanced transition probability by 
\begin{align}\label{eqn:balanced_transition}
p^{\star, h}_{1:h}(x_h) = \frac{\vert \cC_h(x_1) \vert}{X_h} \prod_{k=1}^{h-1} p^{\star, h}_k(x_{k+1} \vert x_k, a_k) = \frac{\vert \cC_h(x_1) \vert}{X_h} \prod_{k = 1}^{h-1} \frac{\vert \cC_h (x_{k+1}) \vert}{\vert \cC_h(x_{k}, a_k) \vert}. 
\end{align}

\begin{lemma}\label{lem:balancing_transition_relation}
For any $(x_h, a_h) \in \cX_h \times \cA$, the sequence form of the transition $p^{\star, h}_{1:h}(x_h)$ and the sequence form of balanced exploration strategy $\mu^{\star, h}_{1:h}(x_h, a_h)$ are related by 
\begin{equation}\label{eqn:balanced_strategy_transition_relation}
p^{\star, h}_{1:h}(x_h) = \frac{1}{X_h A \cdot \mu^{\star, h}_{1:h}(x_h, a_h)}. 
\end{equation}
Furthermore, for any max player's policy $\mu \in \Pi_{\max}$ and any $h \in [H]$, we have 
\begin{equation}\label{eqn:balanced_transition_unity}
\sum_{(x_h, a_h) \in \cX_h \times \cA} \mu_{1:h}(x_h, a_h) p^{\star, h}_{1:h} (x_h) = 1. 
\end{equation}
\end{lemma}

\begin{proof-of-lemma}[\ref{lem:balancing_transition_relation}]
By the definition of the balanced transition probability as in Eq. (\ref{eqn:balanced_transition}) and the balanced exploration strategy as in Eq. (\ref{equation:balanced-policy}), we have 
\[
\frac{1}{X_h A \cdot \mu^{\star, h}_{1:h}(x_h, a_h)} = \frac{1}{X_h A} \prod_{k = 1}^{h-1} \frac{\vert \cC_h(x_k) \vert}{\vert \cC_h(x_k, a_k) \vert} \times A = \frac{\vert \cC_h(x_1) \vert}{X_h} \prod_{k = 1}^{h-1} \frac{\vert \cC_h(x_{k+1}) \vert}{\vert \cC_h(x_k, a_k) \vert} = p^{\star, h}_{1:h}(x_h). 
\]
where the second equality used the property that $\vert \cC_h(x_h)\vert = 1$. This proves Eq. (\ref{eqn:balanced_strategy_transition_relation}). The proof of Eq. (\ref{eqn:balanced_transition_unity}) is similar to the proof of Lemma \ref{lemma:pnu} (a). 
\end{proof-of-lemma}

\paragraph{Alternative proof of Lemma~\ref{lemma:balancing}}
Lemma~\ref{lemma:balancing} follows as a direct consequence of Eq. (\ref{eqn:balanced_strategy_transition_relation}) and (\ref{eqn:balanced_transition_unity}) in Lemma \ref{lem:balancing_transition_relation}.
\qed

\subsection{Balanced dilated KL}

\begin{lemma}[Bound on balanced dilated KL]
  \label{lemma:bound-balanced-dilated-kl}
  Let $\mu^{\rm unif}\in\Pi_{\max}$ denote the uniform policy: $\mu^{\rm unif}_h(a_h|x_h)=1/A$ for all $(h, x_h, a_h)$. Then we have
  \begin{align*}
    \max_{\mu^\dagger\in\Pi_{\max}} \Dbal(\mu^\dagger \| \mu^{\rm unif}) \le XA\log A.
  \end{align*}
\end{lemma}
\begin{proof}
We have
\begin{align*}
    \max_{\mu^\dagger \in\Pi_{\max}}\D( \mu ^{\dagger}\|\mu^{\rm unif} ) =&\max_{\mu^\dagger\in\Pi_{\max}}\sum_{h=1}^H{\sum_{x_h,a_h}{\frac{\mu ^{\dagger}_{1:h}(x_h,a_h)}{\mu _{1:h}^{\star ,h}(x_h,a_h)}}}\log \frac{\mu ^{\dagger}_h(a_h|x_h)}{\mu^{\rm unif}_h(a_h|x_h)}
\\
=&\max_{\mu^\dagger\in\Pi_{\max}}\sum_{h=1}^H{\sum_{x_h,a_h}{\frac{\mu ^{\dagger}_{1:h}(x_h,a_h)}{\mu _{1:h}^{\star ,h}(x_h,a_h)}}}\left( \log \mu ^{\dagger}_h(a_h|x_h)+\log A \right) 
\\
\overset{\left( i \right)}{\le}&\log A\sum_{h=1}^H \max_{\mu^\dagger\in\Pi_{\max}}{\sum_{x_h,a_h}{\frac{\mu ^{\dagger}_{1:h}(x_h,a_h)}{\mu _{1:h}^{\star ,h}(x_h,a_h)}}}
\\
\overset{\left( ii \right)}{=}&\log A\sum_{h=1}^H{X_hA} = XA \log A,
\end{align*}
where $(i)$ is because $\mu ^{\dagger}_h(a_h|x_h)\log \mu ^{\dagger}_h(a_h|x_h) \le 0$ (recalling that each sequence form $\mu_{1:h}^\dagger(x_h, a_h)$ contains the term $\mu_h^{\dagger}(a_h|x_h)$), and $(ii)$ uses the balancing property of $\mu^{\star, h}$ (Lemma~\ref{lemma:balancing}).
\end{proof}

\subsubsection{Interpretation of balanced dilated KL}
\label{appendix:interpretation-dbal}
We present an interpretation of the balanced dilated KL~\eqref{equation:balanced-dilated-kl} as a KL distance between the reaching probabilities under the ``balanced transition''~\eqref{eqn:balanced_transition} on the max player's game tree.

For any policy $\mu \in \Pi_{\max}$, we define its \emph{balanced transition reaching probability} $\P^{\mu, \star}_h(x_h, a_h)$ as
\begin{equation}\label{eqn:balanced_transition_reaching}
\P^{\mu, \star}_h(x_h, a_h) = \mu_{1:h}(x_h, a_h) p^{\star, h}_{1:h}(x_h).
\end{equation}
This is a probability measure on $\cX_h \times \cA$ ensured by Lemma \ref{lem:balancing_transition_relation}. For any two probability distribution $p$ and $q$, we denote $\mathrm{KL}(p \| q)$ to be their KL divergence. 
\begin{lemma}\label{lem:KL_interpretation}
For any tuple of max-player's policies $\mu, \nu \in \Pi_{\max}$, we have 
\begin{equation}
\Dbal(\mu \| \nu) = \sum_{h = 1}^H (X_h A) \mathrm{KL}( \P_h^{\mu_{1:h}, \star} \| \P_h^{\mu_{1:h-1} \nu_h, \star} ). 
\end{equation}
\end{lemma}

\begin{proof-of-lemma}[\ref{lem:KL_interpretation}]
By Eq. (\ref{eqn:balanced_transition_reaching}) and by the definition of KL divergence, we have 
\begin{equation}
\begin{aligned}
&~(X_h A) \Dkl( \P_h^{\mu_{1:h}, \star} \| \P_h^{\mu_{1:h-1} \nu_h, \star} ) \\
=&~ (X_h A) \sum_{(x_h, a_h) \in \cX_h \times \cA} \mu_{1:h}(x_h, a_h) p^{\star, h}_{1:h}(x_h) \log \Big[ \frac{\mu_{1:h}(x_h, a_h) p^{\star, h}_{1:h}(x_h)}{\mu_{1:h-1}(x_{h-1}, a_{h-1}) \nu_h(x_h \vert a_h) p^{\star, h}_{1:h}(x_h)} \Big]\\ =&~ \sum_{(x_h, a_h) \in \cX_h \times \cA} \frac{\mu_{1:h}(x_h, a_h)}{\mu^{\star, h}_{1:h}(x_h, a_h)} \log \Big[ \frac{\mu_{h}(a_h \vert x_h)}{ \nu_h(a_h \vert x_h)} \Big],
\end{aligned}
\end{equation}
where the last equality is by Lemma \ref{lem:balancing_transition_relation}. Comparing with the definition of $\Dbal$ as in Eq. (\ref{equation:balanced-dilated-kl}) concludes the proof. 
\end{proof-of-lemma}

%% file: Sections_arxiv/proof-ixomd.tex
\section{Proofs for Section~\ref{section:ixomd}}

\subsection{Efficient implementation for Update~\eqref{equation:reweighted-update}}
\label{appendix:proof-md_implement}

\begin{algorithm}[h]
  \caption{Implementation of Balanced OMD update}
  \label{algorithm:balanced-omd-implement}
  \small
    \begin{algorithmic}[1]
    \REQUIRE Current policy $\mu^t$; Trajectory $(x_1^t, a_1^t, \dots, x_H^t, a_H^t)$; learning rate $\eta>0$; \\
    \hspace{1.4em} Loss vector $\set{\wt{\ell}_h^t(x_h, a_h)}_{h, x_h, a_h}$ that is non-zero only on $(x_h, a_h)=(x_h^t, a_h^t)$.
    \STATE Set $Z^t_{H+1} \setto 1$. 
    \FOR{$h=H,\dots,1$}
    \STATE Compute normalization constant
    \begin{align*}
        Z_h^t\setto 1-\mu_h^t(a_h^t|x_h^t)+\mu_h^t(a_h^t|x_h^t)\cdot \exp\paren{-\eta \mu _{1:h}^{\star ,h}(x_{h}^{t},a_{h}^t)\widetilde{\ell }_{h}^{t}(x_h^t, a_h^t)+\frac{\mu _{1:h}^{\star ,h}(x_{h}^{t},a_{h}^t) \log Z_{h+1}^{t}}{\mu _{1:h+1}^{\star ,h+1}(x_{h+1}^{t},a_{h+1}^{t})}}.
    \end{align*} 
    \STATE Update policy at $x_h^t$:
    \begin{align*}
      \mu_h^{t+1} (a_h | x_h^t) \setto \left\{
      \begin{aligned}
        & \mu_h^t(a_h | x_h^t) \cdot \exp\paren{-\eta \mu _{1:h}^{\star ,h}(x_{h}^{t},a_{h}^t)\widetilde{\ell }_{h}^{t}(x_h^t, a_h^t)+\frac{\mu _{1:h}^{\star ,h}(x_{h}^{t},a_{h}^t) \log Z_{h+1}^{t}}{\mu _{1:h+1}^{\star ,h+1}(x_{h+1}^{t},a_{h+1}^{t})} - \log Z_h^t} & {\rm if}~a_h = a_h^t, \\
        & \mu_h^t(a_h | x_h^t) \cdot \exp(-\log Z_h^t) & {\rm otherwise}.
      \end{aligned}
      \right.
    \end{align*}
    \STATE Set $\mu_h^{t+1} (\cdot | x_h) \setto \mu_h^t (\cdot | x_h)$ for all $x_h\in \cX_h\setminus \set{x_h^t}$. 
    \ENDFOR
    \ENSURE Updated policy $\mu^{t+1}$.
    \end{algorithmic}
\end{algorithm}

\begin{lemma}
\label{lem:md_implement}
Algorithm~\ref{algorithm:balanced-omd-implement} indeed solves the optimization problem~\eqref{equation:reweighted-update}:
  \begin{align*}
    \mu^{t+1} \setto \argmin_{\mu \in \Pi_{\max}} \<\mu, \wt{\ell}^t\> + \frac{1}{\eta}\D(\mu \| \mu^t).
  \end{align*}
\end{lemma}
\begin{proof}
  First, by the sparsity of the loss estimator $\wt{\ell}^t$ (cf.~\eqref{equation:loss-estimator}), the above objective can be written succinctly as
\begin{align}
  & \quad \<\mu, \wt{\l}^t\> + \frac{1}{\eta}\D(\mu \| \mu^t) \\
  & =\sum_{h=1}^H{\sum_{x_h,a_h}{\mu _{1:h}(x_h,a_h)\left[ \widetilde{\ell }_{h}^{t}\left( x_h,a_h \right) +\frac{1}{\eta \mu _{1:h}^{\star ,h}(x_h,a_h)}\log \frac{\mu _h(a_h|x_h)}{\mu _{h}^{t}(a_h|x_h)} \right]}}
 \nonumber \\
& =\sum_{h=1}^H{\sum_{x_h}{\mu _{1:h-1}(x_h)\left[ \left< \mu _h(\cdot |x_h),\widetilde{\ell }_{h}^{t}\left( x_h,\cdot \right) \right> +\frac{\mathrm{KL}\left( \mu _h(\cdot |x_h)||\mu _{h}^{t}(\cdot |x_h) \right)}{\eta \mu _{1:h}^{\star ,h}(x_h,a_h)} \right]}} \nonumber
 \\
  &=\sum_{h=1}^H{\left\{ \mu _{1:h-1}(x_{h}^{t})\left[ \mu _h(a_{h}^{t}|x_{h}^{t})\widetilde{\ell }_{h}^{t}\left( x_{h}^{t},a_{h}^{t} \right) +\frac{\mathrm{KL}\left( \mu _h(\cdot |x_h)||\mu _{h}^{t}(\cdot |x_h) \right)}{\eta \mu _{1:h}^{\star ,h}(x_{h}^{t},a_h)} \right] +\sum_{x_h\ne x_{h}^{t}}{\mu _{1:h-1}}(x_h)\frac{\mathrm{KL}\left( \mu _h(\cdot |x_h)||\mu _{h}^{t}(\cdot |x_h) \right)}{\eta \mu _{1:h}^{\star ,h}(x_h,a_h)} \right\}}
. \label{equ:opt_object}
\end{align}
We now show the equivalence by backward induction over $h=H,\dots,1$. For $h=H$, we can optimize over the $H$-th layer directly to see 
\begin{align*}
  &  \quad \mu _{H}^{t+1}(a_H|x^t_H) \propto_{a_H} \mu _{H}^{t}(a_H|x^t_H)\exp \left\{ -\eta \mu _{1:h}^{\star ,h}(x_{h}^{t},a_{h}) \widetilde{\ell }_{H}^{t}(x_H^t, a_H) \right\} \\
  & = \mu _{H}^{t}(a_H|x^t_H)\exp \left\{ -\eta  \widetilde{\ell }_{H}^{t}(x_H^t, a_H) -\log Z_{H}^{t} \right\},
\end{align*}
where $Z_H^t>0$ is the normalization constant. For all non-visited $x_H\neq x_H^t$, by equation~\eqref{equ:opt_object} and non-negativity of KL divergence, the object must be minimized at $\mu^{t+1}_H(\cdot|x_H)=\mu^t_h(\cdot|x_H)$.

If the claim holds from layer $h+1$ to $H$, consider the $h$-th layer. Plug in the proved optimizer after layer $h$, the objective~\eqref{equ:opt_object} can be written as 
\begin{align*}
    &\sum_{h'=1}^H{\sum_{x_{h'},a_{h'}}{\mu _{1:h'}(x_{h'},a_{h'})\left[ \widetilde{\ell }_{h'}^{t}\left( x_{h'},a_{h'} \right) +\frac{1}{\eta \mu _{1:h'}^{\star ,h'}(x_{h'},a_{h'})}\log \frac{\mu _{h'}(a_{h'}|x_{h'})}{\mu _{h'}^{t}(a_{h'}|x_{h'})} \right]}}
\\
=&\sum_{h'=1}^H{\sum_{x_{h'}}{\mu _{1:h'-1}(x_{h'})\left[ \left< \mu _{h'}(\cdot |x_{h'}),\widetilde{\ell }_{h'}^{t}\left( x_{h'},\cdot \right) \right> +\frac{\mathrm{KL}\left( \mu _{h'}(\cdot |x_{h'})||\mu _{h'}^{t}(\cdot |x_{h'}) \right)}{\eta \mu _{1:h'}^{\star ,h'}(x_{h'},a_{h'})} \right]}}
\\
=&\sum_{h'=1}^h{\sum_{x_{h'}}{\mu _{1:h'-1}(x_{h'})\left[ \left< \mu _{h'}(\cdot |x_{h'}),\widetilde{\ell }_{h'}^{t}\left( x_{h'},\cdot \right) \right> +\frac{\mathrm{KL}\left( \mu _{h'}(\cdot |x_{h'})||\mu _{h'}^{t}(\cdot |x_{h'}) \right)}{\eta \mu _{1:h'}^{\star ,h'}(x_{h'},a_{h'})} \right]}}
\\
&+\sum_{h'=h+1}^H{\left[ \frac{\mu _{1:h'}(x_{h'}^{t},a_{h'}^{t})\log Z_{h'+1}^{t}}{\eta \mu _{1:h'+1}^{\star ,h'+1}(x_{h'+1}^{t},a_{h'+1}^{t})}-\frac{\mu _{1:h'-1}(x_{h'-1}^{t},a_{h'-1}^{t})\log Z_{h'}^{t}}{\eta \mu _{1:h'}^{\star ,h'}(x_{h'}^{t},a_{h'}^{t})} \right]}
\\
=&\sum_{h'=1}^h{\sum_{x_{h'}}{\mu _{1:h'-1}(x_{h'})\left[ \left< \mu _{h'}(\cdot |x_{h'}),\widetilde{\ell }_{h'}^{t}\left( x_{h'},\cdot \right) \right> +\frac{\mathrm{KL}\left( \mu _{h'}(\cdot |x_{h'})||\mu _{h'}^{t}(\cdot |x_{h'}) \right)}{\eta \mu _{1:h'}^{\star ,h'}(x_{h'},a_{h'})} \right]}}-\frac{\mu _{1:h}(x_{h}^{t},a_{h}^{t})\log Z_{h+1}^{t}}{\eta \mu _{1:h+1}^{\star ,h+1}(x_{h+1}^{t},a_{h+1}^{t})}
\\
=&\sum_{h'=1}^{h-1}{\sum_{x_{h'}}{\mu _{1:h'-1}(x_{h'})\left[ \left< \mu _{h'}(\cdot |x_{h'}),\widetilde{\ell }_{h'}^{t}\left( x_{h'},\cdot \right) \right> +\frac{\mathrm{KL}\left( \mu _{h'}(\cdot |x_{h'})||\mu _{h'}^{t}(\cdot |x_{h'}) \right)}{\eta \mu _{1:h'}^{\star ,h'}(x_{h'},a_{h'})} \right]}}
\\
&+\mu _{1:h-1}(x_{h}^{t})\left[ \mu _h(a_{h}^{t}|x_{h}^{t}) \Big( \widetilde{\ell }_{h}^{t}\left( x_{h}^{t},a_{h}^{t} \right) -\frac{\log Z_{h+1}^{t}}{\eta \mu _{1:h+1}^{\star ,h+1}(x_{h+1}^{t},a_{h+1}^{t})} \Big) +\frac{\mathrm{KL}\left( \mu _h(\cdot |x_{h}^{t})||\mu _{h}^{t}(\cdot |x_{h}^{t}) \right)}{\eta \mu _{1:h}^{\star ,h}(x_{h}^{t},a_h)} \right] 
\\
&+\sum_{x_h\ne x_{h}^{t}}{\mu _{1:h-1}}(x_h) \frac{\mathrm{KL}\left( \mu _h(\cdot |x_h)||\mu _{h}^{t}(\cdot |x_h) \right) }{\eta \mu^{\star, h}_{1:h}(x_h, a_h)} .
\end{align*}

Thus in the $h$ layer we can optimize by setting 
$$
\mu _{h}^{t+1}(a_{h}|x_{h}^{t})=\mu _{h}^{t}(a_{h}|x_{h}^{t})\exp \left\{ -\left[ \eta \mu _{1:h}^{\star ,h}(x_{h}^{t},a_{h}) \widetilde{\ell }_{h}^{t}(x_{h}^{t},a_{h})-\frac{\mu _{1:h}^{\star ,h}(x_{h}^{t},a_{h})}{\mu _{1:h+1}^{\star ,h+1}(x_{h+1}^{t},a^t_{h+1})}\log Z_{h+1}^{t} \right] \ones\left\{ a_{h}=a_{h}^{t} \right\} -\log Z_{h}^{t} \right\} .
$$
For all non-visited $x_h\neq x_h^t$, by non-negativity of KL divergence, the object must be minimized at $\mu^{t+1}_h(\cdot|x_h)=\mu^t_h(\cdot|x_h)$. This is exactly the update rule in Algorithm~\ref{algorithm:balanced-omd-implement}.
\end{proof}

\subsection{Proof of Theorem~\ref{theorem:ixomd-regret}}
\label{appendix:proof-ixomd}
Decompose the regret as
\begin{align}
  & \quad \Reg^T= \max_{\mu^\dagger\in\Pi_{\max}} \sum_{t=1}^T \<\mu^t - \mu^\dagger, \ell^t \> \\
  & \le \underbrace{ \sum_{t=1}^T \<\mu^t, \ell^t - \wt{\ell}^t\>}_{\textrm{BIAS}^1}  + \underbrace{ \max_{\mu^\dagger\in\Pi_{\max}} \sum_{t=1}^T \<\mu^\dagger, \wt{\ell}^t - \ell^t\>}_{\textrm{BIAS}^2} + \underbrace{ \max_{\mu^\dagger\in\Pi_{\max}} \sum_{t=1}^T \<\mu^t - \mu^\dagger, \wt{\ell}^t\>}_{\textrm{REGRET}}.
\end{align}

We now state three lemmas that bound each of the three terms above. Their proofs are presented in Section~\ref{appendix:proof-bias1},~\ref{appendix:proof-bias2}, and~\ref{appendix:proof-regret_sample} respectively. Below, $\iota\defeq \log(3HXA/\delta)$ denotes a log factor.

\begin{lemma}[Bound on ${\rm BIAS}^1$]
  \label{lem:bias_1}
  With probability at least $1-\delta/3$, we have
  $$
  {\rm BIAS}^1 \le H\sqrt{2T\iota}+\gamma HT.
  $$
\end{lemma}

\begin{lemma}[Bound on ${\rm BIAS}^2$]
  \label{lem:bias_2}
  With probability at least $1-\delta/3$, we have
  $$
  {\rm BIAS}^2 \le XA\iota/\gamma.
  $$
\end{lemma}

\begin{lemma}[Bound on ${\rm REGRET}$]
  \label{lem:regret_sample}
  With probability at least $1-\delta/3$, we have
  $$
  {\rm REGRET} \le \frac{XA\log A}{\eta}+\eta H^3T+\frac{\eta H^2XA\iota}{\gamma}.
  $$
\end{lemma}

Putting the bounds together, we have that with probability at least $1-\delta$, 
$$
\Reg^T \le \frac{XA\log A}{\eta}+\eta H^3T+\frac{\eta H^2XA\iota}{\gamma}+H\sqrt{2T\iota}+\gamma HT+\frac{XA\iota}{\gamma}.
$$

Set $\eta =\sqrt{\frac{XA\log A}{H^3T}}$ and $\gamma =\sqrt{\frac{XA\iota}{TH}}$, we have 
$$
\Reg^T \le 6\sqrt{XAH^3T\iota }+HXA\iota.
$$
Additionally, recall the naive bound $\Reg^T\le HT$ on the regret (which follows as $\<\mu^t,\ell^t\>\in [0, H]$ for any $\mu\in\Pi_{\max}$, $t\in[T]$), we get
\begin{align*}
    \Reg^T \le \min\set{ 6\sqrt{XAH^3T\iota }+HXA\iota, HT }  \le HT\cdot \min\set{6\sqrt{XAH\iota/T }+XA\iota/T, 1 }.
\end{align*}
For $T>HXA \iota$, the min above is upper bounded by $7\sqrt{HXA \iota/T}$. For $T\le HXA\iota$, the min above is upper bounded by $1\le 7\sqrt{HXA\iota/T}$. Therefore, we always have
\begin{align*}
    \Reg^T \le HT\cdot 7\sqrt{HXA \iota/T} = 7\sqrt{H^3XA T\iota}.
\end{align*}
This is the desired result.
\qed

The rest of this section is devoted to proving the above three lemmas.

\subsection{A concentration result}
We begin by presenting a useful concentration result. This result is a variant of~\citep[Lemma 3]{kozuno2021model} and~\citep[Lemma 1]{neu2015explore} suitable to our loss estimator~\eqref{equation:loss-estimator} where the IX bonus on the denominator depends on $(x_h, a_h)$.
\begin{lemma}
\label{lem:explore_no_more}
  For some fixed $h \in [H]$, let $\alpha _{h}^{t}\left( x_h,a_h \right) \in \left[ 0,2\gamma \mu _{1:h}^{\star ,h}\left( x_h,a_h \right) \right] $ be $\mathcal{F}^{t-1}$-measurable random variable for each $\left( x_h,a_h \right) \in \mathcal{X}_h\times \mathcal{A}$. Then with probability $1-\delta$,
  $$
  \sum_{t=1}^T{\sum_{x_h,a_h}{\alpha _{h}^{t}\left( x_h,a_h \right) \left( \widetilde{\ell }_{h}^{t}\left( x_h,a_h \right) -\ell _{h}^{t}\left( x_h,a_h \right) \right)}}\le \log \left( 1/\delta \right) .
$$
\end{lemma}

\begin{proof}
Define the unbiased importance sampling estimator 
$$
\hat{\ell}_{h}^{t}:=\frac{1-r_{h}^{t}}{\mu _{1:h}^{t}(x_{h}^{t},a_{h}^{t})}\cdot \ones\left\{ x_h=x_{h}^{t},a_h=a_{h}^{t} \right\} .
$$
We first have
\begin{align*}
    \widetilde{\ell }_{h}^{t}\left( x_h,a_h \right) =&\frac{1-r_{h}^{t}}{\mu _{1:h}^{t}(x_h,a_h)+\gamma \mu _{1:h}^{\star ,h}(x_h,a_h)}\cdot \ones\left\{ x_h=x_{h}^{t},a_h=a_{h}^{t} \right\} 
\\
\le& \frac{1-r_{h}^{t}}{\mu _{1:h}^{t}(x_h,a_h)+\gamma \mu _{1:h}^{\star ,h}(x_h,a_h)\left( 1-r_{h}^{t} \right)}\cdot \ones\left\{ x_h=x_{h}^{t},a_h=a_{h}^{t} \right\} 
\\
\le& \frac{1}{2\gamma \mu _{1:h}^{\star ,h}(x_h,a_h)}\frac{2\gamma \mu _{1:h}^{\star ,h}(x_h,a_h)\left( 1-r_{h}^{t} \right) \ones\left\{ x_h=x_{h}^{t},a_h=a_{h}^{t} \right\} /\mu _{1:h}^{t}(x_h,a_h)}{1+\gamma \mu _{1:h}^{\star ,h}(x_h,a_h)\left( 1-r_{h}^{t} \right) \ones\left\{ x_h=x_{h}^{t},a_h=a_{h}^{t} \right\} /\mu _{1:h}^{t}(x_h,a_h)}
\\
=&\frac{1}{2\gamma \mu _{1:h}^{\star ,h}(x_h,a_h)}\frac{2\gamma \mu _{1:h}^{\star ,h}(x_h,a_h)\hat{\ell}_{h}^{t}(x_h,a_h)}{1+\gamma \mu _{1:h}^{\star ,h}(x_h,a_h)\hat{\ell}_{h}^{t}(x_h,a_h)}
\\
\overset{\left( i \right)}{\le}&\frac{1}{2\gamma \mu _{1:h}^{\star ,h}(x_h,a_h)}\log \left( 1+2\gamma \mu _{1:h}^{\star ,h}(x_h,a_h)\hat{\ell}_{h}^{t}(x_h,a_h) \right) ,
\end{align*}
where $(i)$ is because for any $z \ge 0$, $\frac{z}{1+z/2}\le \log \left( 1+z \right) $.

As a result, we have the following bound on the moment generating function:
\begin{align*}
   & \mathbb{E}\left\{ \exp \left\{ \sum_{x_h,a_h}{\alpha _{h}^{t}\left( x_h,a_h \right) \widetilde{\ell }_{h}^{t}\left( x_h,a_h \right)} \right\} |\mathcal{F}^{t-1} \right\} 
\\
\le& \mathbb{E}\left\{ \exp \left\{ \sum_{x_h,a_h}{\frac{\alpha _{h}^{t}\left( x_h,a_h \right)}{2\gamma \mu _{1:h}^{\star ,h}(x_h,a_h)}\log \left( 1+2\gamma \mu _{1:h}^{\star ,h}(x_h,a_h)\hat{\ell}_{h}^{t}(x_h,a_h) \right)} \right\} |\mathcal{F}^{t-1} \right\} 
\\
\overset{\left( i \right)}{\le}&\mathbb{E}\left\{ \exp \left\{ \sum_{x_h,a_h}{\log \left( 1+\alpha _{h}^{t}\left( x_h,a_h \right) \hat{\ell}_{h}^{t}(x_h,a_h) \right)} \right\} |\mathcal{F}^{t-1} \right\} 
\\
=&\mathbb{E}\left\{ \prod_{x_h,a_h}{\left( 1+\alpha _{h}^{t}\left( x_h,a_h \right) \hat{\ell}_{h}^{t}(x_h,a_h) \right)}|\mathcal{F}^{t-1} \right\} 
\\
\overset{\left( ii \right)}{=}&\mathbb{E}\left\{ 1+\sum_{x_h,a_h}{\alpha _{h}^{t}\left( x_h,a_h \right) \hat{\ell}_{h}^{t}(x_h,a_h)}|\mathcal{F}^{t-1} \right\} 
\\
=&1+\sum_{x_h,a_h}{\alpha _{h}^{t}\left( x_h,a_h \right) \ell _{h}^{t}(x_h,a_h)}
\\
\le& \mathbb{E}\left\{ \exp \left\{ \sum_{x_h,a_h}{\alpha _{h}^{t}\left( x_h,a_h \right) \ell _{h}^{t}(x_h,a_h)} \right\} |\mathcal{F}^{t-1} \right\} ,
\end{align*}
where $(i)$ is because $z\log \left( 1+z' \right) \le \log \left( 1+zz' \right) $ for any $0\le z \le 1$ and $z'>-1$, and $(ii)$ follows from the fact that for any $h$, at most one of $\hat{\ell}_{h}^{t}(x_h,a_h)$ is non-zero, so the cross terms disappear.

Repeating the above argument, 
\begin{align*}
    &\mathbb{E}\left\{ \exp \left\{ \sum_{t=1}^T{\sum_{x_h,a_h}{\alpha _{h}^{t}\left( x_h,a_h \right) \left( \widetilde{\ell }_{h}^{t}\left( x_h,a_h \right) -\ell _{h}^{t}\left( x_h,a_h \right) \right)}} \right\} \right\} 
\\
\le& \mathbb{E}\left\{ \exp \left\{ \sum_{t=1}^{T-1}{\sum_{x_h,a_h}{\alpha _{h}^{t}\left( x_h,a_h \right) \left( \widetilde{\ell }_{h}^{t}\left( x_h,a_h \right) -\ell _{h}^{t}\left( x_h,a_h \right) \right)}} \right\} \mathbb{E}\left\{ \exp \left\{ \sum_{x_h,a_h}{\alpha _{h}^{T}\left( x_h,a_h \right) \left( \widetilde{\ell }_{h}^{T}\left( x_h,a_h \right) -\ell _{h}^{T}\left( x_h,a_h \right) \right)} \right\} |\mathcal{F}^{T-1} \right\} \right\} 
\\
\le& \mathbb{E}\left\{ \exp \left\{ \sum_{t=1}^{T-1}{\sum_{x_h,a_h}{\alpha _{h}^{t}\left( x_h,a_h \right) \left( \widetilde{\ell }_{h}^{t}\left( x_h,a_h \right) -\ell _{h}^{t}\left( x_h,a_h \right) \right)}} \right\} \right\} 
\\
\le& \cdots \le 1.
\end{align*}
Therefore, we can apply the Markov inequality and get
\begin{align*}
    &\mathbb{P}\left\{ \sum_{t=1}^T{\sum_{x_h,a_h}{\alpha _{h}^{t}\left( x_h,a_h \right) \left( \widetilde{\ell }_{h}^{t}\left( x_h,a_h \right) -\ell _{h}^{t}\left( x_h,a_h \right) \right)}}>\log \left( 1/\delta \right) \right\} 
\\
=&\mathbb{P}\left\{ \exp \left\{ \sum_{t=1}^{T-1}{\sum_{x_h,a_h}{\alpha _{h}^{t}\left( x_h,a_h \right) \left( \widetilde{\ell }_{h}^{t}\left( x_h,a_h \right) -\ell _{h}^{t}\left( x_h,a_h \right) \right)}} \right\} >1/\delta \right\} 
\\
\le& \delta \cdot \mathbb{E}\left\{ \exp \left\{ \sum_{t=1}^T{\sum_{x_h,a_h}{\alpha _{h}^{t}\left( x_h,a_h \right) \left( \widetilde{\ell }_{h}^{t}\left( x_h,a_h \right) -\ell _{h}^{t}\left( x_h,a_h \right) \right)}} \right\} \right\} \le \delta .
\end{align*}
This is the desired result.
\end{proof}

\begin{corollary}
  \label{cor:explore_no_more}
  We have
  \begin{enumerate}[label=(\alph*)]
  \item For some fixed $h \in [H]$ and $(x_h,a_h)$, let $\alpha _{h}^{t}\left( x_h,a_h \right) \in \left[ 0,2\gamma \mu _{1:h}^{\star ,h}\left( x_h,a_h \right) \right] $ be $\mathcal{F}^{t-1}$-measurable random variable. Then with probability $1-\delta$,
    $$
    \sum_{t=1}^T{\alpha _{h}^{t}\left( x_h,a_h \right) \left( \widetilde{\ell }_{h}^{t}\left( x_h,a_h \right) -\ell _{h}^{t}\left( x_h,a_h \right) \right)}\le \log \left( 1/\delta \right) .
    $$
  \item For some fixed $h \in [H]$ and $x_h$, let $\alpha _{h}^{t}\left( x_h,a_h \right) \in \left[ 0,2\gamma \mu _{1:h}^{\star ,h}\left( x_h,a_h \right) \right] $ be $\mathcal{F}^{t-1}$-measurable random variable for each $a_h  \in  \mathcal{A}$. Then with probability $1-\delta$,
    $$
    \sum_{t=1}^T\sum_{a_h\in\cA}{\alpha _{h}^{t}\left( x_h,a_h \right) \left( \widetilde{\ell }_{h}^{t}\left( x_h,a_h \right) -\ell _{h}^{t}\left( x_h,a_h \right) \right)}\le \log \left( 1/\delta \right) .
    $$
  \end{enumerate}
\end{corollary}
\begin{proof}
For (a), using Lemma~\ref{lem:explore_no_more} with $\left( \alpha _{h}^{t} \right) '\left( x'_h,a'_h \right) =\alpha _{h}^{t}\left( x'_h,a'_h \right) \ones\left\{ x'_h=x_h,a'_h=a_h \right\} 
$,
\begin{align*}
    &\sum_{t=1}^T{\alpha _{h}^{t}\left( x_h,a_h \right) \left( \widetilde{\ell }_{h}^{t}\left( x_h,a_h \right) -\ell _{h}^{t}\left( x_h,a_h \right) \right)}
\\
=&\sum_{t=1}^T{\sum_{x'_h,a'_h}{\alpha _{h}^{t}\left( x_h,a_h \right) \ones\left\{ x'_h=x_h,a'_h=a_h \right\} \left[ \widetilde{\ell }^t(x'_h,a'_h)-\ell ^t(x'_h,a'_h) \right]}}\le \log \left( 1/\delta \right) .
\end{align*}
Claim (b) can proved similarly.
\end{proof}

\subsection{Proof of Lemma~\ref{lem:bias_1}}
\label{appendix:proof-bias1}
We further decompose $\textrm{BIAS}^1$ to two terms by 
$$
{\rm BIAS}^1 = \sum_{t=1}^T{\left< \mu ^t,\ell ^t-\widetilde{\ell }^t \right>}=\underset{\left( A \right)}{\underbrace{\sum_{t=1}^T{\left< \mu ^t,\ell ^t-\mathbb{E}\left\{ \widetilde{\ell }^t|\mathcal{F}^{t-1} \right\} \right>}}}+\underset{\left( B \right)}{\underbrace{\sum_{t=1}^T{\left< \mu ^t,\mathbb{E}\left\{ \widetilde{\ell }^t|\mathcal{F}^{t-1} \right\} -\widetilde{\ell }^t \right>}}}.
$$

To bound $(A)$, plug in the definition of loss estimator, 
\begin{align*}
    &\sum_{t=1}^T{\left< \mu ^t,\ell ^t-\mathbb{E}\left\{ \widetilde{\ell }^t|\mathcal{F}^{t-1} \right\} \right>}\\
    =&\sum_{t=1}^T{\sum_{h=1}^H{\sum_{x_h,a_h}{\mu _{1:h}^{t}(x_h,a_h)\left[ \ell_h ^t(x_h,a_h)-\frac{\mu _{1:h}^{t}(x_h,a_h)\ell_h ^t(x_h,a_h)}{\mu _{1:h}^{t}(x_h,a_h)+\gamma \mu _{1:h}^{\star ,h}(x_h,a_h)} \right]}}}
\\
=&\sum_{t=1}^T{\sum_{h=1}^H{\sum_{x_h,a_h}{\mu _{1:h}^{t}(x_h,a_h)\ell_h ^t(x_h,a_h)\left[ \frac{\gamma \mu _{1:h}^{\star ,h}(x_h,a_h)}{\mu _{1:h}^{t}(x_h,a_h)+\gamma \mu _{1:h}^{\star ,h}(x_h,a_h)} \right]}}}
\\
\le& \sum_{t=1}^T{\sum_{h=1}^H{\sum_{x_h,a_h}{\gamma \mu _{1:h}^{\star ,h}(x_h,a_h)\ell_h ^t(x_h,a_h)}}}
\\
\overset{\left( i \right)}{\le}& \sum_{t=1}^T{\sum_{h=1}^H{\gamma}=}\gamma HT,
\end{align*}
where $(i)$ is by using Corollary~\ref{cor:averge_loss_bound} with policy $\mu ^{\star ,h}$ for each layer $h$.

To bound $(B)$, first notice 
\begin{align*}
    \left< \mu ^t,\widetilde{\ell }^t \right> =&\sum_{h=1}^H{\sum_{x_h,a_h}{\mu _{1:h}^{t}(x_h,a_h)\frac{\left( 1-r_{h}^{t} \right) \ones\left\{ x_h=x_{h}^{t},a_h=a_{h}^{t} \right\}}{\mu _{1:h}^{t}(x_h,a_h)+\gamma \mu _{1:h}^{\star ,h}(x_h,a_h)}}}
\\
\le& \sum_{h=1}^H{\sum_{x_h,a_h}{\ones\left\{ x_h=x_{h}^{t},a_h=a_{h}^{t} \right\}}}=\sum_{h=1}^H{1}=H.
\end{align*}
Then by Azuma-Hoeffding, with probability at least $1-\delta/3$,
$$
\sum_{t=1}^T{\left< \mu ^t,\mathbb{E}\left\{ \widetilde{\ell }^t|\mathcal{F}^{t-1} \right\} -\widetilde{\ell }^t \right>}\le H\sqrt{2T\log(3/\delta)} \le H\sqrt{2T\iota}.
$$
Combining the bounds for (A) and (B) gives the desired result.
\qed

\subsection{Proof of Lemma~\ref{lem:bias_2}}
\label{appendix:proof-bias2}
We have
\begin{align*}
    & {\rm BIAS}^2 = \max_{\mu^\dagger\in\Pi_{\max}} \sum_{t=1}^T{\left< \mu ^{\dagger},\widetilde{\ell }^t-\ell ^t \right>}
\\
=& \max_{\mu^\dagger\in\Pi_{\max}} \sum_{t=1}^T{\sum_{h=1}^H{\sum_{x_h,a_h}{\mu _{1:h}^{\dagger}(x_h,a_h)\left[ \wt{\ell}_h^t(x_h,a_h)-\ell_h ^t(x_h,a_h) \right]}}}
\\
=& \max_{\mu^\dagger\in\Pi_{\max}}  \sum_{t=1}^T{\sum_{h=1}^H{\sum_{x_h,a_h}{\frac{\mu _{1:h}^{\dagger}(x_h,a_h)}{\gamma \mu _{1:h}^{\star ,h}(x_h,a_h)}\gamma \mu _{1:h}^{\star ,h}(x_h,a_h)\left[ \wt{\ell}_h^t(x_h,a_h)-\ell_h ^t(x_h,a_h) \right]}}}
\\
=& \max_{\mu^\dagger\in\Pi_{\max}}  \sum_{h=1}^H{\sum_{x_h,a_h}{\frac{\mu _{1:h}^{\dagger}(x_h,a_h)}{\gamma \mu _{1:h}^{\star ,h}(x_h,a_h)}\sum_{t=1}^T{\gamma \mu _{1:h}^{\star ,h}(x_h,a_h)\left[ \wt{\ell}_h^t(x_h,a_h)-\ell_h ^t(x_h,a_h) \right]}}}
\\
\overset{\left( i \right)}{\le} &\frac{\log \left( XA/\delta \right)}{\gamma}\sum_{h=1}^H \max_{\mu^\dagger\in\Pi_{\max}}  {\sum_{x_h,a_h}{\frac{\mu _{1:h}^{\dagger}(x_h,a_h)}{\mu _{1:h}^{\star ,h}(x_h,a_h)}}}
\\
\overset{\left( ii \right)}{\le}&\frac{\iota}{\gamma}\sum_{h=1}^H{X_hA}=XA\iota/\gamma ,
\end{align*}
where $(i)$ is by applying Corollary~\ref{cor:explore_no_more} for each $(x_h,a_h)$ pair and taking union bound, and $(ii)$ is by Lemma~\ref{lemma:balancing}.
\qed

\subsection{Proof of Lemma~\ref{lem:regret_sample}}
\label{appendix:proof-regret_sample}

We begin by stating the following lemma, which roughly speaking relates the task of bounding the regret to bounding the term $\left< \mu ,\widetilde{\ell}^{t} \right> + \frac{1}{ \eta\mu _{1:1}^{\star ,1}(x^t_1,a_1)}\log Z_{1}^{t}$.
\begin{lemma}
\label{lem:decompose_sample}
  For any policy $\mu \in \Pi _{\max} $, 
  $$
  \D( \mu \|\mu ^{t+1} ) -\D( \mu \|\mu ^t ) = \eta\left< \mu ,\widetilde{\ell}^{t} \right> + \frac{1}{ \mu _{1:1}^{\star ,1}(x^t_1,a_1)}\log Z_{1}^{t} .
$$
\end{lemma}

\begin{proof}
By definition of $\D$ and the conditional form update rule in Algorithm~\ref{alg:IXOMD},
\begin{align*}
    &\D( \mu \|\mu ^{t+1} ) -\D( \mu \|\mu ^t ) 
\\
=&\sum_{h=1}^H{\sum_{x_h,a_h}{\frac{\mu _{1:h}(x_h,a_h)}{\mu _{1:h}^{\star ,h}(x_h,a_h)}\log \frac{\mu _{h}^{t}(a_h|x_h)}{\mu _{h}^{t+1}(a_h|x_h)}}}
\\
=&\sum_{h=1}^H{\sum_{a_h}{\frac{\mu _{1:h}(x_{h}^{t},a_h)}{\mu _{1:h}^{\star ,h}(x_{h}^{t},a_h)}\log \frac{\mu _{h}^{t}(a_h|x_{h}^{t})}{\mu _{h}^{t+1}(a_h|x_{h}^{t})}}}
\\
=&\sum_{h=1}^H{\frac{\mu _{1:h}(x_{h}^{t},a_{h}^{t})}{\mu _{1:h}^{\star ,h}(x_{h}^{t},a_{h}^{t})}\left[ \eta \mu _{1:h}^{\star ,h}(x_{h}^{t},a_{h}^{t})\widetilde{\ell }_{h}^{\,t}-\frac{\mu _{1:h}^{\star ,h}(x_{h}^{t},a_{h}^{t})}{\mu _{1:h+1}^{\star ,h+1}(x_{h+1}^{t},a_{h+1}^{t})}\log Z_{h+1}^{t} \right] +}\sum_{h=1}^H{\sum_{a_h}{\frac{\mu _{1:h}(x_{h}^{t},a_h)}{\mu _{1:h}^{\star ,h}(x_{h}^{t},a_h)}\log Z_{h}^{t}}}
\\
=&\eta\sum_{h=1}^H{\mu _{1:h}(x_{h}^{t},a_{h}^{t})\widetilde{\ell }_{h}^{\,t}(x_{h}^{t},a_{h}^{t})}-\sum_{h=1}^H{\frac{\mu _{1:h}(x_{h}^{t},a_{h}^{t})}{\mu _{1:h+1}^{\star ,h+1}(x_{h+1}^{t},a_{h+1}^{t})}\log Z_{h+1}^{t}}+\sum_{h=1}^H{\frac{\mu _{1:h-1}(x_{h-1}^{t},a_{h-1}^{t})}{\mu _{1:h}^{\star ,h}(x_{h}^{t},a_{h}^{t})}\log Z_{h}^{t}}
\\
=&\eta \left< \mu ,\widetilde{\ell }^t \right> +\frac{1}{\mu _{1:1}^{\star ,1}(x_{1}^{t},a_1)}\log Z_{1}^{t}.
\end{align*}
\end{proof}

\paragraph{Additional notation}
We introduce the following notation for convenience throughout the rest of this subsection. Define
\begin{align*}
  \beta _{h}^{t}\defeq \eta \mu _{1:h}^{\star ,h}(x_{h}^{t},a_{h}^{t}).
\end{align*}
For simplicity, when there is no confusion, we write 
$$\mu _{h}^{t}\defeq \mu _{h}^{t}(a_h^t|x_h^t),\,\,\,\, \mu _{h:h'}^{t}\defeq \prod_{h''=h}^{h'}{\mu _{h''}^{t}}, $$
and 
$$\widetilde{\ell}_h^{t}:=\widetilde{\ell }_{h}^{t}\left( x_{h}^{t},a_{h}^{t} \right) =\frac{1-r_{h}^{t}}{\mu _{1:h}^{t}(x_{h}^{t},a_{h}^{t})+\gamma\mu^\star_{1:h}(x_{h}^{t},a_{h}^{t})}.$$

Define the normalized log-partition function as
\begin{align}
\label{equation:xiht}
  & \quad \Xi _{h}^{t}:= \frac{1}{ \beta _{h}^{t}}\log Z_{h}^{t} =\frac{1}{\beta _{h}^{t}}\log \left( 1-\mu _{h}^{t}+\mu _{h}^{t}\exp \left[ \beta _{h}^{t}\left( \Xi _{h+1}^{t}-\widetilde{\ell }_{h}^{t} \right) \right] \right).
\end{align}
Note that this value can be seen as an $H$-variate function of the loss estimator $\set{\widetilde{\ell}_h^{t}}_{h\in[H]}$. To make this dependence more clear, for any $\widetilde{\ell }\in [0, \infty)^H$, we define the function $\left\{ \Xi^t _h\left( \cdot \right) \right\} _{h=1}^{H}$ recursively by (overloading notation)
\begin{align*}
  \Xi^t _h\left( \widetilde{\ell } \right) = \Xi^t _h\left( \widetilde{\ell }_{h:H} \right)  \defeq \left\{
  \begin{aligned}
    & \log \left( 1-\mu _{h}^{t}+\mu _{h}^{t}\exp \left[ -\beta _{h}^{t}\widetilde{\ell }_h \right] \right) /\beta _{h}^{t}\,\,\,\, & {\rm if}~~h=H,\\
    & \log \left( 1-\mu _{h}^{t}+\mu _{h}^{t}\exp \left[ \beta _{h}^{t}\left( \Xi _{h+1}\left( \widetilde{\ell}_{h+1:H} \right) -\widetilde{\ell }_{h} \right) \right] \right) /\beta _{h}^{t}\,\,\,\, & {\rm otherwise}.\\
  \end{aligned}
  \right.
\end{align*}
With this definition, we have $\Xi _{h}^{t}=\Xi _h^t\left( \widetilde{\ell }^{t} \right)$ where $\wt{\ell}^t$ is the actual loss estimator. Note that, importantly, $\Xi^t_h(\wt{\ell}_{h:H})$ has a \emph{compositional structure}: It is a function of $\wt{\ell}_h$ ($h$-th entry of the loss) and $\Xi^t_{h+1}$ (which is itself a function of $\wt{\ell}_{h+1:H}$). This compositional structure is key to proving bounds on its gradients and Hessians.

The rest of this subsection is organized as follows. In Section~\ref{appendix:bound-xi1t}, we bound the gradients and Hessians of the function $\Xi_1^t(\cdot)$ in an entry-wise fashion, and then use the Mean-Value Theorem to give a bound on $\Xi_1^t=\Xi_1^t(\wt{\ell}^t)$ (Lemma~\ref{lem:bound_second_order_sample}). We then combine this result with Lemma~\ref{lem:decompose_sample} to prove the main lemma that bounds ${\rm REGRET}$ (Section~\ref{appendix:regret-main}).

\subsubsection{Bounding $\Xi_1^t$}
\label{appendix:bound-xi1t}
\begin{lemma}
\label{lem:Xi_negative_sample}
  For $\widetilde{\ell }\in [0,\infty) ^H$ and any $h \in [H]$,  $\Xi^t _h\left( \widetilde{\ell } \right)  \le 0$. Furthermore, $\Xi^t _h\left( 0 \right) =0$.
\end{lemma}
\begin{proof}
  We show the first claim by backward induction. For $h=H$, 
  $$
  \Xi^t _H\left( \widetilde{\ell }_H \right) =\log \left( 1-\mu _{H}^{t}+\mu _{H}^{t}\exp \left[ -\beta _{H}^{t}\widetilde{\ell }_H \right] \right) /\beta _{H}^{t}\le \log \left( 1-\mu _{H}^{t}+\mu _{H}^{t} \right) /\beta _{H}^{t}\le 0,
  $$
  because $\widetilde{\ell }_{H}^{t}  \ge 0$.
  
  Assume $\Xi^t _{h+1}\left( \widetilde{\ell } \right)  \le 0$, then for the previous step $h$,
  $$
  \Xi^t _h\left( \widetilde{\ell }_{h:H} \right) =\log \left( 1-\mu _{h}^{t}+\mu _{h}^{t}\exp \left[ \beta _{h}^{t}\left( \Xi^t _{h+1}\left( \widetilde{\ell }_{h+1:H} \right) -\widetilde{\ell }_h \right) \right] \right) /\beta _{h}^{t}\le \log \left( 1-\mu _{h}^{t}+\mu _{h}^{t} \right) /\beta _{h}^{t}\le 0.
  $$
  
  The second claim follows as all inequalities become equalities at $\wt{\ell}=0$.
\end{proof}

\begin{lemma}[Bounds on first derivatives]
 \label{lem:Xi_first_order_sample}
 For $\widetilde{\ell }\in [0,1] ^H$ and any $h  \in [H]$, the derivatives are bounded by 
 $$
  0\le \frac{\partial \Xi^t _h}{\partial \Xi^t _{h+1}}\le \mu _{h}^{t}~~~{\rm and}~~~-\mu _{h}^{t}\le \frac{\partial \Xi^t _h}{\partial \widetilde{\ell }_h}\le 0.
$$
 Furthermore, 
 $$
 \left. \frac{\partial \Xi^t _h}{\partial \widetilde{\ell }_{h'}} \right|_{\widetilde{\ell }=0}=\left\{
   \begin{aligned}
     & -\mu _{h:h'}^{t} & {\rm if}~h'\ge h, \\
     & 0 & \mathrm{otherwise}.
   \end{aligned}
   \right.
$$
\end{lemma}

\begin{proof}
  By chain rule and the compositional structure of the functions $\Xi_h^t(\cdot)$, for any $h' \ge h$,
  $$
  \frac{\partial \Xi^t _h}{\partial \widetilde{\ell }_{h'}}=\frac{\partial \Xi^t _h}{\partial \Xi^t _{h'}}\cdot \frac{\partial \Xi^t _{h'}}{\partial \widetilde{\ell }_{h'}}=\left( \prod_{h''=h}^{h'-1}{\frac{\partial \Xi^t _{h''}}{\partial \Xi^t _{h''+1}}} \right) \cdot \frac{\partial \Xi^t _{h'}}{\partial \widetilde{\ell }_{h'}}.
  $$

For any $h$, the derivatives are bounded by 
\begin{align*}
    \frac{\partial \Xi^t _h}{\partial \Xi^t _{h+1}}=&\frac{\mu _{h}^{t}\exp \left[ \beta _{h}^{t}\left( \Xi^t _{h+1}\left( \widetilde{\ell } \right) -\widetilde{\ell }_h \right) \right]}{1-\mu _{h}^{t}+\mu _{h}^{t}\exp \left[ \beta _{h}^{t}\left( \Xi^t _{h+1}\left( \widetilde{\ell } \right) -\widetilde{\ell }_h \right) \right]}\in \left[ 0,\mu _{h}^{t} \right], 
\\
\frac{\partial \Xi^t _h}{\partial \widetilde{\ell }_h}=-&\frac{\mu _{h}^{t}\exp \left[ \beta _{h}^{t}\left( \Xi^t _{h+1}\left( \widetilde{\ell } \right) -\widetilde{\ell }_h \right) \right]}{1-\mu _{h}^{t}+\mu _{h}^{t}\exp \left[ \beta _{h}^{t}\left( \Xi^t _{h+1}\left( \widetilde{\ell } \right) -\widetilde{\ell }_h \right) \right]}\in \left[ -\mu _{h}^{t},0 \right].
\end{align*}
 The inequalities hold because the function $f\left( z \right) =\frac{\mu _{h}^{t}z}{1-\mu _{h}^{t}+\mu _{h}^{t}z}=1-\frac{1-\mu _{h}^{t}}{1-\mu _{h}^{t}+\mu _{h}^{t}z}$ is increasing on $z \in [0,1]$, and $\exp \left[ \beta _{h}^{t}\left( \Xi^t _{h+1}\left( \widetilde{\ell } \right) -\widetilde{\ell }_h \right) \right] \in \left[ 0,1 \right] $ by Lemma~\ref{lem:Xi_negative_sample}.
 
 Putting them together, at $\widetilde{\ell }=0$, the derivative is just $\left. \frac{\partial \Xi^t _h}{\partial \widetilde{\ell }_{h'}} \right|_{\widetilde{\ell }^t=0}=-\mu _{h:h'}^{t}$ if $h' \ge h$. If $h'<h$, since $\Xi^t_h$ only depends on loss in the later layers, $\frac{\partial \Xi^t _h}{\partial \widetilde{\ell }_{h'}} |_{\widetilde{\ell }^t=0}=0$.
\end{proof}

\begin{lemma}[Bounds on second derivatives]
\label{lem:Xi_second_order_sample} 
For $\widetilde{\ell }\in [0,1] ^H$ and any $h\in [H]$, if $h' \ge h$ and $h'' \ge h$, the second-order derivatives are bounded by
$$
\frac{\partial ^2\Xi^t _{h}}{\partial \widetilde{\ell }_{h'}\partial \widetilde{\ell }_{h''}}\le \sum_{h'''=h}^{\min\{h', h''\}}{\beta _{h'''}^{t}\mu _{h:h'}^{t}\mu _{h'''+1:h''}^{t}}
= \sum_{h'''=h}^{\min\{h', h''\}}{\beta_{h'''}^{t}
\mu_{h:h'''}^t
\mu _{h'''+1:h'}^{t}\mu _{h'''+1:h''}^{t}}.
$$
Otherwise $\frac{\partial ^2\Xi^t _{h}}{\partial \widetilde{\ell }_{h'}\partial \widetilde{\ell }_{h''}} = 0$.
\end{lemma}

\begin{proof}
By symmetry of the second derivatives and the right-hand side with respect to $h'$ and $h''$, it suffices to prove the claim for $h''\ge h'$ only.

By chain rule and the compositional structure of the functions $\Xi_h^t(\cdot)$,
$$
\frac{\partial ^2\Xi^t _{h}}{\partial \widetilde{\ell }_{h'}\partial \widetilde{\ell }_{h''}}=\frac{\partial ^2\Xi^t _{h}}{\partial \Xi^t _{h'}\partial \widetilde{\ell }_{h''}}\cdot \frac{\partial \Xi^t _{h'}}{\partial \widetilde{\ell }_{h'}}+\frac{\partial \Xi^t _{h}}{\partial \Xi^t _{h'}}\cdot \frac{\partial ^2\Xi^t _{h'}}{\partial \widetilde{\ell }_{h'}\partial \widetilde{\ell }_{h''}}.
$$

If $h''=h'=h$, 
$$
\frac{\partial ^2\Xi^t _h}{\partial \widetilde{\ell }_{h}^{2}}=\beta _{h}^{t}\mu _{h}^{t}\exp \left[ \beta _{h}^{t}\left( \Xi^t _{h+1}\left( \widetilde{\ell } \right) -\widetilde{\ell }_h \right) \right] \frac{1-\mu _{h}^{t}}{\left\{ 1-\mu _{h}^{t}+\mu _{h}^{t}\exp \left[ \beta _{h}^{t}\left( \Xi^t _{h+1}\left( \widetilde{\ell } \right) -\widetilde{\ell }_h \right) \right] \right\} ^2}\le \beta _{h}^{t}\mu _{h}^{t}.
$$

If $h'=h, h''>h$,
$$
\frac{\partial ^2\Xi^t _h}{\partial \widetilde{\ell }_h\partial \widetilde{\ell }_{h''}}=-\frac{(1-\mu_h^t)\beta _{h}^{t}\mu _{h}^{t}\exp \left[ \beta _{h}^{t}\left( \Xi^t _{h+1}\left( \widetilde{\ell } \right) -\widetilde{\ell }_h \right) \right]}{  \paren{ 1-\mu _{h}^{t}+\mu _{h}^{t}\exp \left[ \beta _{h}^{t}\left( \Xi^t _{h+1}\left( \widetilde{\ell } \right) -\widetilde{\ell }_h \right) \right] }^2 }\cdot \frac{\partial \Xi^t _{h+1}}{\partial \widetilde{\ell }_{h''}}\le \beta _{h}^{t}\mu _{h:h''}^{t}.
$$

If $h<h'<h''$, we can compute the Hessian by induction. Notice once $h'>h$ we have 
$$
\frac{\partial \Xi^t _{h}}{\partial \widetilde{\ell }_{h'}}=\frac{\partial \Xi^t _{h}}{\partial \Xi^t _{h+1}}\cdot \frac{\partial \Xi^t _{h+1}}{\partial \widetilde{\ell }_{h'}}.
$$

Take second derivative, 
$$
\frac{\partial ^2\Xi^t _{h}}{\partial \widetilde{\ell }_{h'}\partial \widetilde{\ell }_{h''}}=\underset{\left( i \right)}{\underbrace{\frac{\partial \Xi^t _{h}}{\partial \Xi^t _{h+1}}\cdot \frac{\partial ^2\Xi^t _{h+1}}{\partial \widetilde{\ell }_{h'}\partial \widetilde{\ell }_{h''}}}}+\underset{\left( ii \right)}{\underbrace{\frac{\partial ^2\Xi^t _{h}}{\partial \Xi^t _{h+1}\partial \widetilde{\ell }_{h''}}\cdot \frac{\partial \Xi^t _{h+1}}{\partial \widetilde{\ell }_{h'}}}}.
$$

We first bound the second term, 
\begin{align*}
& \quad \left( ii \right) =\frac{ (1-\mu_h^t) \beta _{h}^{t}\mu _{h}^{t}\exp \left[ \beta _{h}^{t}\left( \Xi^t _{h+1}\left( \widetilde{\ell } \right) -\widetilde{\ell }_h \right) \right]}{ \paren{  1-\mu _{h}^{t}+\mu _{h}^{t}\exp \left[ \beta _{h}^{t}\left( \Xi^t _{h+1}\left( \widetilde{\ell } \right) -\widetilde{\ell }_h \right) \right] }^2 }\cdot \frac{\partial \Xi^t _{h+1}}{\partial \widetilde{\ell }_{h''}} \cdot \frac{\partial \Xi^t _{h+1}}{\partial \widetilde{\ell }_{h'}} \\
& \le \beta_h^t\mu_h^t \cdot \mu_{h+1:h''}^t \cdot \mu_{h+1:h'}^t \\
& \le \beta _{h}^{t}\mu _{h:h'}^{t}\mu _{h+1:h''}^{t}.
\end{align*}

The first term can be simplified to
$$
\left( i \right) \le \frac{\mu _{h}^{t}\exp \left[ \beta _{h}^{t}\left( \Xi^t _{h+1}\left( \widetilde{\ell } \right) -\widetilde{\ell }_h \right) \right]}{1-\mu _{h}^{t}+\mu _{h}^{t}\exp \left[ \beta _{h}^{t}\left( \Xi^t _{h+1}\left( \widetilde{\ell } \right) -\widetilde{\ell }_h \right) \right]}\frac{\partial ^2\Xi^t _{h+1}}{\partial \widetilde{\ell }_{h'}\partial \widetilde{\ell }_{h''}}\le \mu _{h}^{t}\frac{\partial ^2\Xi^t _{h+1}}{\partial \widetilde{\ell }_{h'}\partial \widetilde{\ell }_{h''}}.
$$

Now plug in $\frac{\partial ^2\Xi^t _{h'}}{\partial \widetilde{\ell }_{h'}\partial \widetilde{\ell }_{h''}}\le \beta _{h'}^{t}\mu _{h':h''}^{t}$ and backward induction from $h'$ to $h$ gives:
$$
\frac{\partial ^2\Xi^t _{h}}{\partial \widetilde{\ell }_{h'}\partial \widetilde{\ell }_{h''}}\le \sum_{h'''=h}^{h'}{\beta _{h'''}^{t}\mu _{h:h'}^{t}\mu _{h'''+1:h''}^{t}}.
$$

We can check this expression is also correct for the above special cases when $h'=h$. The second claim holds because $\Xi^t_h$ only depends on loss in the later layers.
\end{proof}

\begin{lemma}[Bound on $\Xi_1^t$]
  We have
  \label{lem:bound_second_order_sample}
 $$
 \Xi^t _{1}\le -\left< \mu ^t,\widetilde{\ell }^t \right> +\frac{\eta H}{2}\sum_{h=1}^H{\left( \sum_{h'=h}^H{\sum_{x_{h'},a_{h'}}{\mu _{1:h}^{\star ,h}\left( x_{h'}, a_{h'} \right) \mu _{h+1:h'}^{t}\left( x_{h'},a_{h'} \right) \widetilde{\ell }_{h'}^{t}\left( x_{h'},a_{h'} \right)}} \right)}.
 $$
\end{lemma}
\begin{proof}
We apply the Mean-value Theorem to function $\Xi^t _1\left( \widetilde{\ell } \right)$ at $\widetilde{\ell }=0$,
$$
\Xi^t _{1}=\Xi^t _1\left( \widetilde{\ell }^t \right) =\Xi^t _1\left( 0 \right) +\left< \left. \nabla _{\widetilde{\ell }}\Xi^t _1 \right|_{\widetilde{\ell }=0},\widetilde{\ell }^t \right> +\frac{1}{2} \left< \nabla ^2_{\widetilde{\ell }}\left. \Xi^t _1 \right|_{\widetilde{\ell }=\xi^t }\widetilde{\ell }^t,\widetilde{\ell }^t \right> ,
$$
where $\xi^t$ lies on the line segment between $0$ and $\widetilde{\ell }^t$.

By Lemma~\ref{lem:Xi_negative_sample}, the initial term is just zero. By Lemma~\ref{lem:Xi_first_order_sample}, the first-order term is just $-\left< \mu ^t,\widetilde{\ell }^t \right> $.

It thus remains to bound the second-order term. Applying the entry-wise upper bounds in Lemma~\ref{lem:Xi_second_order_sample} at $h=1$ (which hold uniformly at all nonnegative loss values, including $\xi^t$), we have
\begin{align*}
   \left< \nabla _{\widetilde{\ell }}^{2}\left. \Xi^t _1 \right|_{\widetilde{\ell }=\xi^t}\widetilde{\ell }^t,\widetilde{\ell }^t \right> =&\sum_{h=1}^H{\sum_{h'=1}^H{\left. \frac{\partial ^2\Xi^t _1}{\partial \widetilde{\ell }_h\partial \widetilde{\ell }_{h'}} \right|}}_{\widetilde{\ell }=\xi^t}\widetilde{\ell }_{h}^{t}\widetilde{\ell }_{h'}^{t}
\\
\overset{\left( i \right)}{\le}&\sum_{h=1}^H{\sum_{h'=1}^H{\sum_{h''=1}^{\min \{h,h'\}}{\beta _{h''}^{t}\mu _{1:h}^{t}\mu _{h''+1:h'}^{t}\widetilde{\ell }_{h}^{t}\widetilde{\ell }_{h'}^{t}}}}
\\
=&\sum_{h=1}^H{\mu _{1:h}^{t}\widetilde{\ell }_{h}^{t}\sum_{h'=1}^H{\sum_{h''=1}^{\min \{h,h'\}}{\beta _{h''}^{t}\mu _{h''+1:h'}^{t}\widetilde{\ell }_{h'}^{t}}}}
\\
\overset{\left( ii \right)}{\le} & H\underset{h\in \left[ H \right]}{\max}\sum_{h'=1}^H{\sum_{h''=1}^{\min \{h,h'\}}{\beta _{h''}^{t}\mu _{h''+1:h'}^{t}\widetilde{\ell }_{h'}^{t}}}
\\
=&H\sum_{h'=1}^H{\sum_{h''=1}^{h'}{\beta _{h''}^{t}\mu _{h''+1:h'}^{t}\widetilde{\ell }_{h'}^{t}}}
\\
=&H\sum_{h''=1}^H{\sum_{h'=h''}^{H}{\beta _{h''}^{t}\mu _{h''+1:h'}^{t}\widetilde{\ell }_{h'}^{t}}}
\\
=&\eta H\sum_{h''=1}^H{\left( \sum_{h'=h''}^H{\mu _{1:h''}^{\star ,h''}\left( x_{h'}^{t},a_{h'}^{t} \right) \mu _{h''+1:h'}^{t}\widetilde{\ell }_{h'}^{t}} \right)}
\\
\overset{\left( iii \right)}{=}&\eta H\sum_{h''=1}^H{\left( \sum_{h'=h''}^H{\sum_{x_{h'},a_{h'}}{\mu _{1:h''}^{\star ,h''}\left( x_{h'},a_{h'} \right) \mu _{h''+1:h'}^{t}\left( x_{h'},a_{h'} \right) \widetilde{\ell }_{h'}^{t}(x_{h'},a_{h'})}} \right)},
\end{align*}
where $(i)$ is by Lemma~\ref{lem:Xi_second_order_sample}; $(ii)$ follows from the bound
\begin{align*}
\sum_{h=1}^H \mu_{1:h}^t \wt{\ell}_h^t = \sum_{h=1}^H \mu^t_{1:h} \cdot \frac{1-r_h^t}{\mu_{1:h}^t + \gamma\mu^{\star,h}_{1:h}} \le H;
\end{align*}
and $(iii)$ is because $\wt{\ell}_{h'}^t(x_{h'}, a_{h'})=0$ at all $(x_{h'}, a_{h'})\neq (x_{h'}^t, a_{h'}^t)$. 
\end{proof}

\begin{lemma}
\label{lem:bound_Xi_sum}
 With probability at least $1-\delta/3$,
 $$
\sum_{t=1}^T{\Xi^t _{1}}\le -\sum_{t=1}^T{\left< \mu ^t,\widetilde{\ell }^t \right>}+\eta H^3T+\frac{\eta XAH^2\iota}{\gamma},
 $$
 where $\iota\defeq \log(H/\delta)$.
\end{lemma}

\begin{proof}
Using Lemma~\ref{lem:bound_second_order_sample} and take the summation with respect to $t \in [T]$ we have 
\begin{align}
\label{eq:xi_sum_decomposition}
    \sum_{t=1}^T{\Xi _{1}^{t}}\le -\sum_{t=1}^T{\left< \mu ^t,\widetilde{\ell }^t \right>}+\frac{\eta H}{2}\sum_{h=1}^H{ \sum_{h\prime=h}^H{\sum_{t=1}^T{\underset{\defeq\Delta _{h,h\prime}^{t}}{\underbrace{\sum_{x_{h\prime},a_{h\prime}}{\mu _{1:h}^{\star ,h}\left( x_{h\prime},a_{h\prime} \right) \mu _{h+1:h\prime}^{t}\left( x_{h\prime},a_{h\prime} \right) \widetilde{\ell } _{h\prime}^{t}\left( x_{h\prime},a_{h\prime} \right)}}}} }}.
\end{align}

Observe that the random variables $\Delta _{h,h\prime}^{t}$ satisfy the following:
\begin{itemize}
\item $\Delta _{h,h\prime}^{t}\le X_{h'}A/\gamma$ almost surely:
  \begin{align*}
    \Delta _{h,h\prime}^{t}=&\sum_{x_{h\prime},a_{h\prime}}{\mu _{1:h}^{\star ,h}\left( x_{h'},a_{h'} \right) \mu _{h+1:h\prime}^{t}\left( x_{h\prime},a_{h\prime} \right) \frac{\left( 1-r_{h'}^{t} \right) \ones\left\{ x_{h\prime}=x_{h'}^{t},a_{h\prime}=a_{h'}^{t} \right\}}{\mu _{1:h\prime}^{t}\left( x_{h\prime},a_{h\prime} \right) +\gamma \mu _{1:h'}^{\star ,h'}\left( x_{h\prime},a_{h\prime} \right)}}
\\
\le& \frac{1}{\gamma}\sum_{x_{h\prime},a_{h\prime}}{\frac{\mu _{1:h}^{\star ,h}\left( x_{h'},a_{h'} \right) \mu _{h+1:h\prime}^{t}\left( x_{h\prime},a_{h\prime} \right)}{\mu _{1:h'}^{\star ,h'}\left( x_{h\prime},a_{h\prime} \right)}}\overset{\left( i \right)}{\le}\frac{X_{h'}A}{\gamma},
  \end{align*}
where $(i)$ is by using Lemma~\ref{lemma:balancing} with the mixture of $\mu^{\star ,h}$ and $\mu^{t}$.
\item $\E[\Delta _{h,h\prime}^{t}|\cF_{t-1}]\le 1$, where $\cF_{t-1}$ is the $\sigma$-algebra containing all information after iteration $t-1$:
\begin{align*}
   \E[\Delta _{h,h\prime}^{t}|\cF_{t-1}] = & \sum_{x_{h\prime},a_{h\prime}}{\mu _{1:h}^{\star ,h}\left( x_{h'},a_{h'} \right) \mu _{h+1:h\prime}^{t}\left( x_{h\prime},a_{h\prime} \right) \ell _{h\prime}^{t}\left( x_{h\prime},a_{h\prime} \right)}\overset{\left( i \right)}{\le}1
,
\end{align*}
where $(i)$ is by using Corollary~\ref{cor:averge_loss_bound} with the mixture policy of $\mu^{\star ,h}$ and $\mu^t$.
\item The conditional variance $\E[(\Delta _{h,h\prime}^{t})^2|\cF_{t-1}]$ can be bounded as
\begin{align*}
    \E[(\Delta _{h,h\prime}^{t})^2|\cF_{t-1}] \overset{\left( i \right)}{=}&\sum_{x_{h\prime},a_{h\prime}}{\left[ \left( \mu _{1:h}^{\star ,h}\left( x_{h'},a_{h'} \right) \mu _{h+1:h\prime}^{t}\left( x_{h\prime},a_{h\prime} \right) \frac{\left( 1-r_{h'}^{t} \right) \ones\left\{ x_{h\prime}=x_{h'}^{t},a_{h\prime}=a_{h'}^{t} \right\}}{\mu _{1:h\prime}^{t}\left( x_{h\prime},a_{h\prime} \right) +\gamma \mu _{1:h'}^{\star ,h'}\left( x_{h\prime},a_{h\prime} \right)} \right) ^2 \right]}
\\
\le& \sum_{x_{h\prime},a_{h\prime}}{\left( \frac{\mu _{1:h}^{\star ,h}\left( x_{h'},a_{h'} \right) \mu _{h+1:h\prime}^{t}\left( x_{h\prime},a_{h\prime} \right)}{\mu _{1:h\prime}^{t}\left( x_{h\prime},a_{h\prime} \right) +\gamma \mu _{1:h'}^{\star ,h'}\left( x_{h\prime},a_{h\prime} \right)} \right)}^2\mu _{1:h\prime}^{t}\left( x_{h\prime},a_{h\prime} \right) 
\\
\le& \frac{1}{\gamma}\sum_{x_{h\prime},a_{h\prime}}{\frac{\mu _{1:h}^{\star ,h}\left( x_{h'},a_{h'} \right) \mu _{h+1:h\prime}^{t}\left( x_{h\prime},a_{h\prime} \right)}{\mu _{1:h'}^{\star ,h'}\left( x_{h\prime},a_{h\prime} \right)}}\overset{\left( ii \right)}{\le}\frac{X_{h'}A}{\gamma},
\end{align*}
where $(i)$ follows from the fact that for any $h$, at most one of indicators is non-zero, so the cross terms disappear and $(ii)$ is  using Corollary~\ref{cor:averge_loss_bound} with the mixture policy of $\mu^{\star ,h}$ and $\mu^t$.
\end{itemize}

Therefore, we can apply Freedman's inequality (Lemma~\ref{lemma:freedman}) and union bound to get that, with probability at least $1-\delta/3$, for some fixed $\lambda_{h,h'}\in(0, \gamma/X_{h'}A]$,
the following holds simultaneously for all $h,h'$:
$$
\sum_{t=1}^T{\Delta _{h,h\prime}^{t}}\le \frac{\lambda_{h,h\prime} X_{h'}AT}{\gamma}+\frac{2\log(H/\delta)}{\lambda_{h,h\prime}}+T,
$$
Take $\lambda_{h,h\prime} = \gamma/X_{h'}A$, we have 
$$
\sum_{t=1}^T{\Delta _{h,h\prime}^{t}}\le \frac{X_{h'}A\cdot 2\log(H/\delta)}{\gamma}+2T.
$$

Plug into equation~\eqref{eq:xi_sum_decomposition}, we have 
$$
\sum_{t=1}^T{\Xi _{1}^{t}}\le -\sum_{t=1}^T{\left< \mu ^t,\widetilde{\ell }^t \right>}+\eta H^3T+\frac{\eta H^2XA\iota}{\gamma},
$$
where $\iota\defeq \log(H/\delta)$ is a log factor.
\end{proof}

\subsubsection{Proof of main lemma}
\label{appendix:regret-main}
By Lemma~\ref{lem:decompose_sample}, for any policy $\mu ^{\dagger}\in\Pi_{\max}$, 
$$
\frac{1}{\eta}\left( \D( \mu ^{\dagger}\|\mu ^{t+1} ) -\D(\mu ^{\dagger}\|\mu ^t) \right) =\left< \mu ^{\dagger},\widetilde{\ell }^t \right> +\Xi _{1}^{t}.
$$
Taking the summation w.r.t. $t\in[T]$ and using Lemma~\ref{lem:bound_Xi_sum}, we have with probability at least $1-\delta/3$, the following holds simultaneously over all $\mu^\dagger\in\Pi_{\max}$:
\begin{align*}
    \frac{1}{\eta}\left( \D( \mu ^{\dagger}\|\mu ^T) -\D( \mu ^{\dagger}\|\mu ^1 \right)) =&\sum_{t=1}^T{\left< \mu ^{\dagger},\widetilde{\ell }^t \right>}+\sum_{t=1}^T{\Xi _{1}^{t}}
\\
\le& \sum_{t=1}^T{\left< \mu ^{\dagger}-\mu ^t,\widetilde{\ell }^t \right>}+\eta H^3T+\frac{\eta H^2XA\iota}{\gamma}.
\end{align*}
Rerranging the terms we have 
\begin{align*}
    \max_{\mu^\dagger \in\Pi_{\max}}\sum_{t=1}^T{\left< \mu ^t-\mu ^{\dagger},\widetilde{\ell }^t \right>}\le& \max_{\mu^\dagger\in\Pi_{\max}} \frac{1}{\eta}\left( \D( \mu ^{\dagger}\|\mu ^1 ) -\D( \mu ^{\dagger}\|\mu ^T ) \right) + \eta H^3T+\frac{\eta H^2XA\iota}{\gamma}
\\
\le& \max_{\mu^\dagger\in\Pi_{\max}} \frac{1}{\eta}\D( \mu ^{\dagger}\|\mu ^1
     ) +\eta H^3T+\frac{\eta H^2XA\iota}{\gamma}\\
  \le & \frac{XA\log A}{\eta} + \eta H^3T+\frac{\eta H^2XA\iota}{\gamma},
\end{align*}
where the last inequality above follows by recalling that $\mu^1$ is taken to be the uniform policy ($\mu_h^1 (a_h|x_h) = 1  / A$ for all $(h, x_h, a_h)$) in Algorithm~\ref{alg:IXOMD}, and applying the bound on the balanced dilated KL (Lemma~\ref{lemma:bound-balanced-dilated-kl}). This proves Lemma~\ref{lem:regret_sample}.

%% file: Sections_arxiv/proof-lower-bound.tex
\subsection{Proof of Theorem~\ref{theorem:lower-bound}}
\label{appendix:proof-lower-bound}

Both the regret and PAC lower bounds follow from a direct reduction to stochastic multi-armed bandits. For completeness, we first state the lower bound for stochastic bandits~\citep[Exercise 15.4 \& Exercise 33.1]{lattimore2020bandit} as follows. Below, $c$ is an absolute constant. 
\begin{proposition}[Lower bound for stochastic bandits]
  \label{proposition:bandit-lower-bound}
  Let $K\ge 2$ denote the number of arms.
  \begin{enumerate}[label=(\alph*), topsep=0pt, itemsep=0pt]
  \item (Regret lower bound) Suppose $T\ge K$. For any bandit algorithm that plays policy $\mu^t\in\Delta([K])$ (either deterministic or random) in round $t\in[T]$, there exists some $K$-armed stochastic bandit problem with Bernoulli rewards with mean vector $r\in[0,1]^K$, on which the algorithm suffers from the following lower bound on the expected regret:
    \begin{align*}
      \E\brac{ \max_{\mu^\dagger\in \Delta([K])}\sum_{t=1}^T \<\mu^\dagger - \mu^t, r\> } \ge c\cdot \sqrt{KT}.
    \end{align*}
  \item (PAC lower bound) For any bandit algorithm that plays for $t$ rounds and outputs some policy $\what{\mu}\in\Delta([K])$, there exists some $K$-armed stochastic bandit problem with Bernoulli rewards with some mean vector $r\in[0,1]^K$, on which policy $\what{\mu}$ is at least $\eps$ away from optimal:
    \begin{align*}
      \E\brac{ \max_{\mu^\dagger\in\Delta([K])}\<\mu^\dagger - \what{\mu}, r\> } \ge \eps,
    \end{align*}
    unless $T\ge cK/\eps^2$.
  \end{enumerate}  
\end{proposition}

We now construct a class of IIEFGs with $X_H=A^{H-1}$ (the minimal possible number of infosets), and show that any algorithm that solves this class of games will imply an algorithm for stochastic bandits with $A^H$ arms with the same regret/PAC bounds, from which Theorem~\ref{theorem:lower-bound} follows.

Our construction is as follows: For any $A\ge 2$ and $H\ge 1$, we let $S_h=A^{h-1}$ for all $h\in[H]$ (in particular, $S_1=1$) and $B=1$ (so that there is no opponent effectively). By the tree structure, each state is thus uniquely determined by all past actions $s_h=(a_1,\dots,a_{h-1})$, and the transition is deterministic: $((a_1,\dots,a_{h-1}), a_h)\in\cS_h\times\cA$ transits to $(a_1,\dots, a_h)\in\cS_{h+1}$ with probability one. Further, we let $x_h=x(s_h)=s_h$, so that there is no partial observability, and thus $\cX_h=\cS_h$ for all $h$. Only the $H$-th layer yields a Bernoulli reward with some mean $r_{a_{1:H}}\defeq \E[r_H(a_{1:H-1}, a_H)]\in[0,1]$, for all $a_{1:H}\in\cX_H$. The reward is zero within all previous layers.

Under this model, the expected reward under any policy $\mu\in\Pi_{\max}$ can be succinctly written as
\begin{align*}
  \<\mu, r\> = \sum_{(x_H, a_H)\in\cX_H\times\cA} \mu_{1:H}(x_H, a_H) \E[r_H(x_h, a_H)] = \sum_{a_{1:H}\in \cA^H} \mu_{1:H}(a_{1:H}) r_{a_{1:H}}.
\end{align*}
This expression coincides with the expression for the expected reward of an $A^H$-armed stochastic bandit problem.

Now, for any algorithm $\Alg$ achieving regret $\Reg^T$ on IIEFGs, we claim we can use it to design an algorithm for solving any $A^H$-armed stochastic bandit problem with Bernoulli rewards, and achieve the same regret. Indeed, given any $A^H$-armed bandit problem, we rename its arms as a sequence $a_{1:H}=(a_1,\dots,a_H)\in\cA^H$. Now, we instantiate an instance of $\Alg$ on a simulated IIEFG with the above structure. Whenever $\Alg$ plays policy $\mu^t\in\Pi_{\max}$, we query an arm $a_{1:H}$ using policy $\mu^t_{1:H}(\cdot)\in\Delta(\cA^H)$ in the bandit problem. Then, upon receiving the reward $r^t$ from the bandit problem, we give the feedback that the game transitted to infoset $a_{1:H}$ and yielded reward $r^t$. By the above equivalence, the regret $\Reg^T$ within this simulated game is exactly the same as the regret for the bandit problem.

Therefore, for $T\ge A^H$, we can apply Proposition~\ref{proposition:bandit-lower-bound}(a) to show that for any such $\Alg$, there exists one such IIEFG, on which
\begin{align*}
  \E\brac{\Reg^T} \ge c\cdot \sqrt{A^HT} = c\sqrt{X_HAT} \ge c\sqrt{XAT},
\end{align*}
where the last inequality follows from the fact that $X\le X_H(1+1/A+1/A^2+\cdots)\le X_H/ (1-1/A)\le 2X_H$ by perfect recall. This shows part (a).

Part (b) (PAC lower bound) follows similarly from Proposition~\ref{proposition:bandit-lower-bound}(b). Using the same reduction, we can show for any algorithm that controls both players and outputs policy $(\what{\mu}, \what{\nu})\in\Pi_{\max}\times\Pi_{\min}$, there exists one such game of the above form (where only the max player affects the game) where the algorithm suffers from the PAC lower bound
\begin{align*}
  \E\brac{ \negap(\what{\mu}, \what{\nu}) } = \E\brac{ \max_{\mu\in\Pi_{\max}} V^{\mu^\dagger, \what{\nu}} - V^{\what{\mu}, \what{\nu}} } \ge \epsilon
\end{align*}
unless $T\ge cXA/\eps^2$. The symmetrical construction for the min player implies that there exists some game on which $\E\brac{\negap(\what{\mu}, \what{\nu})}\ge \epsilon$ unless $T\ge cYB/\eps^2$.

Therefore, if $T< c(XA+YB)/(2\eps^2)$, at least one of $T\ge cXA/\eps^2$ and $T\ge cYB/\eps^2$ has to be false, for which we obtain a game where the expected duality gap is at least $\eps$. This shows part (b).
\qed

%% file: Sections_arxiv/proof-cfr.tex
\section{Proofs for Section~\ref{section:cfr}}
\label{appendix:proof-cfr}

\subsection{Counterfactual regret decomposition}\label{sec:CFR_decomposition}

Define the immediate counterfactual regret at any $x_h\in\cX_h$, $h\in[H]$ as
\begin{align}
  \label{equation:immediate-counterfactual-regret}
  \Reg_{h}^{\imm, T}(x_{h}) = \max_{\mu_{h}^\dagger(\cdot|x_{h})} \sum_{t=1}^T \< \mu_{h}^t(\cdot | x_{h}) - \mu_{h}^\dagger(\cdot | x_{h}) , \L_{h}^t(x_{h}, \cdot)\>,
\end{align}
where $L_h^t(\cdot, \cdot)$ is the counterfactual loss function defined in~\eqref{equation:counterfactual-loss}:
\begin{align*}
  \L_h^t(x_h, a_h) \defeq \l_h^t(x_h, a_h) + \sum_{h'=h+1}^H \sum_{(x_{h'}, a_{h'})\in \cC_{h'}(x_h, a_h)\times \cA} \mu^t_{(h+1):h'}(x_{h'}, a_{h'}) \l_{h'}^t(x_{h'}, a_{h'}).
\end{align*}

\begin{lemma}[Counterfactual regret decomposition]
  \label{lemma:cfr-regret-decomposition}
  We have $\wt{\Reg}^T\le \sum_{h=1}^H \Reg_{h}^T$, where
  \begin{align*}
    & \quad \Reg_{h}^T \defeq \sum_{x_{1}\in\cX_{1}} \max_{a_{1}\in\cA} \cdots \sum_{x_{h-1}\in \cC(x_{h-2}, a_{h-2})} \max_{a_{h-1}\in\cA} \sum_{x_{h}\in \cC(x_{h-1}, a_{h-1})} \Reg_{h}^{\imm, T}(x_{h}), \\
    & = \max_{\mu\in\Pi_{\max}} \sum_{x_{h}\in\cX_{h}} \mu_{1:(h-1)}(x_{h-1}, a_{h-1}) \cdot \Reg_{h}^{\imm, T}(x_{h}).
  \end{align*}
\end{lemma}
\begin{proof}
  The bound $\wt{\Reg}^T\le \sum_{h=1}^H \Reg_h^T$ with the sum-max form expression for $\wt{\Reg}_h^T$ has already implicitly appeared in the proof of~\citep[Theorem 3]{zinkevich2007regret}, albeit with their slightly different formulation of extensive-form games (turn-based games with reward only in the last round). For completeness, here we provide a proof under our formulation.

  We first show the bound with the $\mu$ form expression for $\Reg_h^T$, which basically follows by a performance decomposition argument. We have
  \begin{align*}
    & \quad \wt{\Reg}^T = \max_{\mu^\dagger\in\Pi_{\max}}\sum_{t=1}^T \<\mu^t - \mu^\dagger, \l^t\> \\
    & = \max_{\mu^\dagger\in\Pi_{\max}} \sum_{t=1}^T \sum_{h=1}^H \<\mu^\dagger_{1:h-1}\mu^t_{h:H} - \mu^\dagger_{1:h}\mu^t_{h+1:H}, \l^t\> \\
    & \le \sum_{h=1}^H \underbrace{\max_{\mu^\dagger\in\Pi_{\max}} \sum_{t=1}^T \<\mu^\dagger_{1:h-1}\mu^t_{h:H} - \mu^\dagger_{1:h}\mu^t_{h+1:H}, \l^t\>}_{\defeq \Reg_h^T}.
  \end{align*}
  Note that each term $\Reg_h^T$ measures the performance difference between $\mu^\dagger_{1:h-1}\mu^t_{h:H}$ and $\mu^\dagger_{1:h}\mu^t_{h+1:H}$:
  \begin{align*}
    & \quad \Reg_h^T = \max_{\mu^\dagger\in\Pi_{\max}} \sum_{t=1}^T \E_{s_h\sim \mu^\dagger_{1:h-1}\times \nu^t} \brac{ \E_{a_h\sim \mu^t(\cdot|x_h)}\brac{ \sum_{h'=1}^H r_{h'}  }  - \E_{a_h\sim \mu^\dagger(\cdot|x_h)}\brac{ \sum_{h'=1}^H r_{h'}  } } \\
    & \stackrel{(i)}{=} \max_{\mu^\dagger\in\Pi_{\max}} \sum_{t=1}^T \E_{s_h\sim \mu^\dagger_{1:h-1}\times \nu^t} \brac{ \E_{a_h\sim \mu^t(\cdot|x_h)}\brac{ \sum_{h'=h}^H r_{h'}  }  - \E_{a_h\sim \mu^\dagger(\cdot|x_h)}\brac{ \sum_{h'=h}^H r_{h'}  } } \\
    & \stackrel{(ii)}{=} \max_{\mu^\dagger\in\Pi_{\max}} \sum_{t=1}^T \sum_{x_h\in\cX_h} \mu^\dagger_{1:h-1}(x_{h-1}, a_{h-1}) \cdot \<\mu^t_h(\cdot|x_h) - \mu^\dagger_h(\cdot|x_h), L_h^t(x_h, \cdot)\> \\
    & = \max_{\mu^\dagger\in\Pi_{\max}} \sum_{x_h\in\cX_h} \mu^\dagger_{1:h-1}(x_{h-1}, a_{h-1}) \cdot \Reg^{\imm, T}_h(x_h).
  \end{align*}
  Above, (i) follows as the rewards for the first $h-1$ steps are the same for the two expectations; (ii) follows by definition of the counterfactual loss function (cumulative loss multiplied by the opponent and environment's policy / transition probabilities, as well as the max player's own policy from step $h$ onward). The claim  (with the $\mu$ form expression) thus follows by renaming the dummy variable $\mu^\dagger$ as $\mu$.

  To verify that the second expression is equivalent to the first expression, it suffices to notice that the max over $\mu_{1:h-1}\in\Pi_{\max}$ consists of separable optimization problems over $\mu_{h'}(\cdot|x_{h'})$ over all $x_{h'}\in\cX_{h'}$, $h'\le h-1$, due to the perfect recall assumption (different $(x_{h'}, a_{h'})$ leads to disjoint subtrees). Therefore, we can rewrite the above as
  \begin{align*}
    &  \Reg_h^T = \sum_{x_1\in\cX_1} \max_{\mu_1(\cdot|x_1)\in\Delta(\cA)} \sum_{a_1\in\cA} \mu_1(a_1|x_1) \sum_{x_2\in \cC(x_1, a_1)} \cdots \\
    & \qquad \sum_{x_{h-1}\in\cC(x_{h-2}, a_{h-2})} \max_{\mu_{h-1}(\cdot|x_{h-1})\in\Delta(\cA)} \sum_{a_{h-1}\in\cA} \mu_{h-1}(a_{h-1}|x_{h-1})  \sum_{x_h\in\cC(x_{h-1}, a_{h-1})} \Reg^{\imm, T}_h(x_h).
  \end{align*}
  Further noticing (backward recursively) that each max over the action distribution is achieved at a single action yields the claimed sum-max form expression.
\end{proof}

\subsection{Proof of Theorem~\ref{theorem:cfr}}
\label{appendix:proof-theorem-cfr}
We now prove our main theorem on the regret of the CFR algorithm.

By Lemma~\ref{lemma:cfr-regret-decomposition}, we have $\wt{\Reg}^T\le \sum_{h=1}^H \Reg_{h}^T$, where for any $h\in[H]$ we have
\begin{align*}
  & \quad \Reg_{h}^T = \max_{\mu\in\Pi_{\max}} \sum_{x_{h}\in\cX_{h}} \mu_{1:(h-1)}(x_{h-1}, a_{h-1}) \Reg_{h}^{\imm, T}(x_{h}) \\
  & = \max_{\mu\in\Pi_{\max}} \sum_{x_{h}\in\cX_{h}} \mu_{1:(h-1)}(x_{h-1}, a_{h-1}) \max_{\mu_{h}^\dagger(\cdot|x_{h})} \sum_{t=1}^T \< \mu_{h}^t(\cdot | x_{h}) - \mu_{h}^\dagger(\cdot | x_{h}), \L_{h}^t(x_{h}, \cdot)\> \\
  & \le \max_{\mu\in\Pi_{\max}} \sum_{x_{h}\in\cX_{h}} \mu_{1:(h-1)}(x_{h-1}, a_{h-1}) \underbrace{\max_{\mu_{h}^\dagger(\cdot|x_{h})} \sum_{t=1}^T \< \mu_{h}^t(\cdot | x_{h}) - \mu_{h}^\dagger(\cdot | x_{h}), \wt{\L}_{h}^t(x_{h}, \cdot)\>}_{\defeq \wt{\Reg}_{h}^{\imm, T}(x_{h})} \\
  & \qquad + \max_{\mu\in\Pi_{\max}} \sum_{x_{h}\in\cX_{h}} \mu_{1:(h-1)}(x_{h-1}, a_{h-1}) \sum_{t=1}^T \< \mu_{h}^t(\cdot | x_{h}), \L_{h}^t(x_{h}, \cdot) - \wt{\L}_{h}^t(x_{h}, \cdot) \>  \\
  & \qquad + \max_{\mu\in\Pi_{\max}} \sum_{x_{h}\in\cX_{h}} \mu_{1:(h-1)}(x_{h-1}, a_{h-1}) \max_{\mu_{h}^\dagger(\cdot|x_{h})} \sum_{t=1}^T \< \mu_{h}^\dagger(\cdot | x_{h}), \wt{\L}_{h}^t(x_{h}, \cdot) - \L_{h}^t(x_{h}, \cdot) \>  \\
  & \stackrel{(i)}{=} \underbrace{\max_{\mu\in\Pi_{\max}} \sum_{x_{h}\in\cX_{h}} \mu_{1:(h-1)}(x_{h-1}, a_{h-1}) \wt{\Reg}_{h}^{\imm, T}(x_{h})}_{\defeq {\rm REGRET}_h} \\
  & \qquad + \underbrace{\max_{\mu\in\Pi_{\max}}  \sum_{(x_{h}, a_{h})\in\cX_{h}\times \cA} \mu_{1:(h-1)}(x_{h-1}, a_{h-1}) \sum_{t=1}^T \mu_{h}^t(a_{h} | x_{h})\brac{ \L_{h}^t(x_{h}, a_{h}) - \wt{\L}_{h}^t(x_{h}, a_{h})} }_{\defeq {\rm BIAS}_h^1} \\
  & \qquad + \underbrace{\max_{\mu\in\Pi_{\max}} \sum_{(x_{h}, a_{h})\in\cX_{h}\times \cA} \mu_{1:h}(x_{h}, a_{h}) \sum_{t=1}^T \brac{\wt{\L}_{h}^t(x_{h}, a_{h}) - \L_{h}^t(x_{h}, a_{h})} }_{\defeq {\rm BIAS}_h^2} \\
  & = {\rm REGRET}_h + {\rm BIAS}^1_h + {\rm BIAS}^2_h.
\end{align*}
Above, the simplification of the ${\rm BIAS}^2_h$ part in (i) uses the fact that the inner max over $\mu_{h}^\dagger(\cdot|x_{h})$ and the outer max over $\mu_{1:(h-1)}$ are separable and thus can be merged into a single max over $\mu_{1:h}$.

We now state three lemmas that bound each term above. Their proofs are deferred to Sections~\ref{appendix:proof-cfr-bias1}-\ref{appendix:proof-cfr-regret}.

\begin{lemma}[Bound on ${\rm BIAS}^1_h$]
  \label{lemma:cfr-bias1}
  For any sequence of opponents' policies $\nu^t\in\cF_{t-1}$, using the estimator $\wt{\L}_{h}$ in~\eqref{equation:L-estimator}, with probability $1-\delta/10$, we have
  \begin{align*}
    \sum_{h=1}^H {\rm BIAS}^1_h \le 2\sqrt{H^3XA T \iota} + HX \iota,
  \end{align*}
  where $\iota=\log(10X/\delta)$.
\end{lemma}

\begin{lemma}[Bound on ${\rm BIAS}^2_h$]
  \label{lemma:cfr-bias2}
  For any sequence of opponents' policies $\nu^t\in\cF_{t-1}$, using the estimator $\wt{\L}_{h}$ in~\eqref{equation:L-estimator}, with probability $1-\delta/10$, we have
  \begin{align*}
    \sum_{h=1}^H {\rm BIAS}^2_h \le 2\sqrt{H^3XAT \iota} + HXA\iota,
  \end{align*}
  where $\iota=\log(10XA/\delta)$.
\end{lemma}

\begin{lemma}[Bound on ${\rm REGRET}_h$]
  \label{lemma:cfr-regret}
  Choosing $\eta=\sqrt{XA\iota/(H^3T)}$, we have that with probability at least $1-\delta/10$ (over the randomness within the loss estimator $\wt{\L}^t_{h}$),
  \begin{align*}
    \sum_{h=1}^H {\rm REGRET}_h \le 2\sqrt{H^3XAT \iota} + \sqrt{H X^3A^3 \iota^3/(4T)},
  \end{align*}
  where $\iota=\log(10XA/\delta)$.
\end{lemma}

Combining Lemma~\ref{lemma:cfr-bias1},~\ref{lemma:cfr-bias2}, and~\ref{lemma:cfr-regret}, we obtain the following: Choosing $\eta=\sqrt{XA\iota/(H^3T)}$, with probability at least $1-3\delta/10\ge 1-\delta$, we have
\begin{align*}
  & \quad \wt{\Reg}^T \le \sum_{h=1}^H \Reg_{h}^T \le \sum_{h=1}^H {\rm REGRET}_h + \sum_{h=1}^H {\rm BIAS}^1_h + \sum_{h=1}^H {\rm BIAS}^2_h \\
  & \le 6\sqrt{H^3XAT\iota} + 2HXA\iota + \sqrt{HX^3A^3\iota^3/(4T)}.  
\end{align*}
Additionally, recall the naive bound $\wt{\Reg}^T\le HT$ on the regret (which follows as $\<\mu^t,\ell^t\>\in [0, H]$ for any $\mu\in\Pi_{\max}$, $t\in[T]$), we get
\begin{align*}
    & \quad \wt{\Reg}^T \le \min\set{ 6\sqrt{H^3XA T\iota} + 2HXA\iota + \sqrt{HX^3A^3\iota^3/4T}, HT } \\
    & \le HT\cdot \min\set{6\sqrt{HXA \iota/T} + 2XA\iota/T + \sqrt{X^3A^3\iota^3/(4HT^3)}, 1 }.
\end{align*}
For $T>HXA \iota$, the min above is upper bounded by $9\sqrt{HXA \iota/T}$. For $T\le HXA\iota$, the min above is upper bounded by $1\le 9\sqrt{HXA\iota/T}$. Therefore, we always have
\begin{align*}
    \wt{\Reg}^T \le HT\cdot 9\sqrt{HXA \iota/T} = 9\sqrt{H^3XA T\iota}.
\end{align*}
This is the desired result.
\qed

\subsection{Proof of Lemma~\ref{lemma:cfr-bias1}}
\label{appendix:proof-cfr-bias1}

Rewrite ${\rm BIAS}^1_h$ as
\begin{align}
  & \quad {\rm BIAS}^1_h = \max_{\mu\in\Pi_{\max}} \sum_{(x_{h}, a_{h})\in\cX_{h}\times \cA} \frac{\mu_{1:(h-1)}(x_{h-1}, a_{h-1})}{\mu^{\star, h}_{1:(h-1)}(x_{h-1}, a_{h-1})} \nonumber \\
  & \qquad\qquad\qquad  \cdot \sum_{t=1}^T \mu^{\star, h}_{1:(h-1)}(x_{h-1}, a_{h-1}) \mu_{h}^t(a_{h}|x_{h}) \cdot \brac{ \L_{h}^t(x_{h}, a_{h}) - \wt{\L}_{h}^t(x_{h}, a_{h}) } \nonumber \\
  & = \max_{\mu\in\Pi_{\max}} \sum_{x_{h} \in\cX_{h}} \frac{\mu_{1:(h-1)}(x_{h-1}, a_{h-1})}{\mu^{\star, h}_{1:(h-1)}(x_{h-1}, a_{h-1})} \nonumber \\
  & \cdot \sum_{t=1}^T \underbrace{\sum_{a_{ h} \in \cA} \frac{\mu_{h}^t(a_{h}|x_{h})}{\mu^{\star, h}_{ h}(a_{h}|x_{h})} \brac{ \mu^{\star, h}_{1:h}(x_{h}, a_{h}) \L_{h}^t(x_{h}, a_{h}) - \paren{ H-h+1 - \sum_{h'=h}^H r_{h'}^{\th} }\indic{ (x_{h}^{\th}, a_{h}^{\th})=(x_{h}, a_{h})} }}_{\defeq \wt{\Delta}_t^{x_{h}}}. \label{equation:cfr-bias1}
\end{align}
Observe that the random variables $\wt{\Delta}_t^{x_{h}}$ satisfy the following:
\begin{itemize}
\item $\wt{\Delta}_t^{x_{h}}\le H$ almost surely:
  \begin{align*}
    & \quad \wt{\Delta}_t^{x_{h}} \le \sum_{a_{ h} \in \cA} \frac{\mu_{h}^t(a_{h}|x_{h})}{\mu^{\star, h}_{ h}(a_{h}|x_{h})} \cdot \mu^{\star, h}_{1:h}(x_{h}, a_{h}) \L_{h}^t(x_{h}, a_{h}) \\
    & = \sum_{a_{ h} \in \cA} \mu_{h}^t(a_{h}|x_{h}) \mu^{\star, h}_{1:(h-1)}(x_{h-1}, a_{h-1}) \L_{h}^t(x_{h}, a_{h}) \le H.
  \end{align*}
  Above, the last bound follows from Lemma~\ref{lemma:pnu}(a).
\item $\E[\wt{\Delta}_t^{x_{h}}|\cF_{t-1}]=0$, where $\cF_{t-1}$ is the $\sigma$-algebra containing all information after iteration $t-1$;
\item The conditional variance $\E[(\wt{\Delta}_t^{x_{h}})^2|\cF_{t-1}]$ can be bounded as
  \begin{align*}
    & \quad \E\brac{ \paren{\wt{\Delta}_t^{x_{h}}}^2 \Big| \cF_{t-1}} \\
    & \le \E\brac{ \sum_{a_{ h} \in \cA} \paren{ \frac{\mu_{h}^t(a_{h}|x_{h})}{\mu^{\star, h}_{ h}(a_{h}|x_{h})} }^2 \cdot \paren{ H-h+1 - \sum_{h'=h}^H r_{h'}^{\th} }^2 \indic{ (x_{h}^{\th}, a_{h}^{\th})=(x_{h}, a_{h})} \Big| \cF_{t-1} }\\
    & \le H^2 \sum_{a_{ h} \in \cA} \paren{ \frac{\mu_{h}^t(a_{h}|x_{h})}{\mu^{\star, h}_{ h}(a_{h}|x_{h})} }^2 \cdot \P^{\mu^{\star, h}_{1:h}, \nu^t}\paren{ (x_{h}^{\th}, a_{h}^{\th})=(x_{h}, a_{h}) } \\
    & = H^2 \sum_{a_{ h} \in \cA} \paren{ \frac{\mu_{h}^t(a_{h}|x_{h})}{\mu^{\star, h}_{ h}(a_{h}|x_{h})} }^2 \cdot \mu^{\star,h}_{1:h}(x_{h}, a_{h}) \cdot p^{\nu^t}_{1:h}(x_{h}) \\
    & = H^2 \sum_{a_{ h} \in \cA} \underbrace{ \paren{ \frac{\mu_{h}^t(a_{h}|x_{h})}{\mu^{\star, h}_{ h}(a_{h}|x_{h})}} }_{\le A} \cdot \mu^{\star,h}_{1:h-1}(x_{h-1}, a_{h-1}) \cdot \mu_{h}^t(a_{h}|x_{h})  p^{\nu^t}_{1:h}(x_{h}) \\
    & \le H^2A \cdot \sum_{a_{ h} \in \cA} \mu^{\star,h}_{1:h-1}(x_{h-1}, a_{h-1}) \cdot \mu_{h}^t(a_{h}|x_{h})  p^{\nu^t}_{1:h}(x_{h}).
  \end{align*}
\end{itemize}
Therefore, we can apply Freedman's inequality (Lemma~\ref{lemma:freedman}) and union bound to get that, for any fixed $\lambda\in(0, 1/H]$, with probability at least $1-\delta/10$, the following holds simultaneously for all $(h, x_{h})$:
\begin{align*}
  \sum_{t=1}^T \wt{\Delta}_t^{x_{h}} \le \lambda H^2A \sum_{a_{ h} \in \cA} \mu^{\star, h}_{1:h-1}(x_{h-1}, a_{h-1}) \cdot \sum_{t=1}^T \mu_{h}^t(a_{h}|x_{h}) p^{\nu^t}_{1:h}(x_{h}) + \frac{\iota}{\lambda},
\end{align*}
where $\iota\defeq \log(10X/\delta)$ is a log factor.
Plugging this bound into~\eqref{equation:cfr-bias1} yields that, for all $h\in[H]$,
\begin{align*}
  & \quad {\rm BIAS}^1_h = \max_{\mu\in\Pi_{\max}} \sum_{x_{h} \in\cX_{h} } \frac{\mu_{1:(h-1)}(x_{h-1}, a_{h-1})}{\mu^{\star, h}_{1:(h-1)}(x_{h-1}, a_{h-1})} \cdot \sum_{t=1}^T \wt{\Delta}_t^{x_{h}} \\
  & \le \max_{\mu\in\Pi_{\max}} \sum_{x_{h}\in\cX_{h}} \frac{\mu_{1:(h-1)}(x_{h-1}, a_{h-1})}{\mu^{\star, h}_{1:(h-1)}(x_{h-1}, a_{h-1})} \cdot \Bigg[ \lambda H^2A \sum_{a_{ h} \in \cA} \mu^{\star, h}_{1:h-1}(x_{h-1}, a_{h-1}) \cdot \sum_{t=1}^T \mu_{h}^t(a_{h}|x_{h}) p^{\nu^t}_{1:h}(x_{h}) + \frac{\iota}{\lambda} \Bigg] \\
  & \le \lambda H^2A \cdot \max_{\mu\in\Pi_{\max}} \sum_{(x_{h}, a_{h})\in\cX_{h}\times \cA} \mu_{1:h-1}(x_{h-1}, a_{h-1}) \sum_{t=1}^T \mu_{h}^t(a_{h}|x_{h})  p^{\nu^t}_{1:h}(x_{h}) \\
  & \qquad + \frac{\iota}{\lambda} \cdot \max_{\mu\in\Pi_{\max}} \sum_{x_{h} \in\cX_{h}} \frac{\mu_{1:(h-1)}(x_{h-1}, a_{h-1})}{\mu^{\star, h}_{1:(h-1)}(x_{h-1}, a_{h-1})} \\
  & \stackrel{(i)}{=} \lambda H^2 AT + \frac{\iota}{\lambda} \cdot \frac{1}{A} \max_{\mu\in\Pi_{\max}} \sum_{(x_{h}, a_{h})\in\cX_{h}\times \cA} \frac{(\mu_{1:(h-1)}\mu^{\rm unif}_h)(x_{h}, a_{h})}{\mu^{\star, h}_{1:h}(x_{h}, a_{h})} \\
  & \stackrel{(ii)}{=} \lambda H^2 AT + \frac{\iota}{\lambda} \cdot X_{h}.
\end{align*}
Above, (i) used the fact that $\sum_{(x_{h},a_{h})\in\cX_{h}\times\cA} \mu_{1:h-1}(x_{h-1}, a_{h-1}) \mu_{h}^t(a_{h}|x_{h}) p^{\nu^t}_{1:h}(x_{h}) = 1$ for any $\mu\in\Pi_{\max}$ and any $t\in[T]$ (Lemma~\ref{lemma:pnu}(a)), as well as the fact that $\mu^{\star, h}_{h}(a_{h}|x_{h})=\mu^{\rm unif}_h(a_{h}|x_{h})\defeq 1/A$; (ii) used the balancing property of $\mu^{\star, h}_{1:h}$ (Lemma~\ref{lemma:balancing}).
Combining the bounds for all $h\in[H]$, we get that with probability at least $1-\delta/10$,
\begin{align*}
  \sum_{h=1}^H {\rm BIAS}^1_h \le \lambda H^3AT + \frac{X \iota}{\lambda}.
\end{align*}
Choosing
\begin{align*}
  \lambda=\min\set{ \sqrt{\frac{X\iota}{H^3AT}}, \frac{1}{H}} \le \frac{1}{H},
\end{align*}
we obtain the bound
\begin{align*}
  \sum_{h=1}^H {\rm BIAS}^1_h \le 2\sqrt{H^3 XAT \iota} + H X\iota.
\end{align*}
This is the desired result.
\qed

\subsection{Proof of Lemma~\ref{lemma:cfr-bias2}}
\label{appendix:proof-cfr-bias2}

The proof strategy is similar to Lemma~\ref{lemma:cfr-bias1}. We can rewrite ${\rm BIAS}^2_h$ as
\begin{align}
  & \quad {\rm BIAS}^2_h = \max_{\mu\in\Pi_{\max}} \sum_{(x_{h}, a_{h})\in\cX_{h}\times \cA} \frac{\mu_{1:h}(x_{h}, a_{h})}{\mu^{\star, h}_{1:h}(x_{h}, a_{h})} \cdot \sum_{t=1}^T \mu^{\star, h}_{1:h}(x_{h}, a_{h}) \brac{ \wt{\L}_{h}^t(x_{h}, a_{h}) - \L_{h}^t(x_{h}, a_{h}) } \nonumber \\
  & = \max_{\mu\in\Pi_{\max}} \sum_{(x_{h}, a_{h})\in\cX_{h}\times \cA} \frac{\mu_{1:h}(x_{h}, a_{h})}{\mu^{\star, h}_{1:h}(x_{h}, a_{h})} \cdot \nonumber \\
  & \qquad \sum_{t=1}^T \underbrace{\brac{ \paren{ H-h+1 - \sum_{h'=h}^H r_{h'}^{\th} }\indic{ (x_{h}^{\th}, a_{h}^{\th})=(x_{h}, a_{h})} - \mu^{\star, h}_{1:h}(x_{h}, a_{h}) \L_{h}^t(x_{h}, a_{h}) }}_{\defeq \Delta_t^{x_{h}, a_{h}}}, \label{equation:cfr-bias2}
\end{align}
where the last equality used the definition of the loss estimator $\wt{\L}_h^t(x_h, a_h)$ in~\eqref{equation:L-estimator}.

Observe that the random variables $\Delta_t^{x_{h}, a_{h}}$ satisfy the following:
\begin{itemize}
\item $\Delta_t^{x_{h}, a_{h}}\le H$ almost surely.
\item $\E[\Delta_t^{(x_{h}, a_{h})}|\cF_{t-1}]=0$, where $\cF_{t-1}$ is the $\sigma$-algebra containing all information after iteration $t-1$. This follows as the episode was sampled using $\mu^{\th}=\mu^{\star, h}_{1:h}\mu^t_{h+1:H}$, as well as the definition of $\L_h^t(x_h, a_h)$ in~\eqref{equation:counterfactual-loss}.
\item The conditional variance $\E[(\Delta_t^{(x_{h}, a_{h})})^2|\cF_{t-1}]$ can be bounded as
  \begin{align*}
    & \quad \E\brac{ \paren{\Delta_t^{(x_{h}, a_{h})}}^2 \Big| \cF_{t-1}} \le \E\brac{ \paren{ H-h+1 - \sum_{h'=h}^H r_{h'}^{\th} }^2 \indic{ (x_{h}^{\th}, a_{h}^{\th})=(x_{h}, a_{h})} \Big| \cF_{t-1} }\\
    & \le H^2 \P^{\mu^{\star, h}_{1:h}, \nu^t}\paren{ (x_{h}^{\th}, a_{h}^{\th})=(x_{h}, a_{h}) } \\
    & = H^2 \mu^{\star,h}_{1:h}(x_{h}, a_{h}) \cdot p^{\nu^t}_{1:h}(x_{h}).
  \end{align*}
\end{itemize}
Therefore, we can apply Freedman's inequality (Lemma~\ref{lemma:freedman}) and union bound to get that, for any fixed $\lambda\in(0, 1/H]$, with probability at least $1-\delta/10$, the following holds simultaneously for all $(h, x_{h}, a_{h})$:
\begin{align*}
  \sum_{t=1}^T \Delta_t^{(x_{h}, a_{h})} \le \lambda H^2 \mu^{\star, h}_{1:h}(x_{h}, a_{h}) \cdot \sum_{t=1}^T p^{\nu^t}_{1:h}(x_{h}) + \frac{\iota}{\lambda},
\end{align*}
where $\iota\defeq \log(10XA/\delta)$ is a log factor. Plugging this bound into~\eqref{equation:cfr-bias2} yields that, for all $h\in[H]$, 
\begin{align*}
  & \quad {\rm BIAS}^2_h = \max_{\mu\in\Pi_{\max}} \sum_{(x_{h}, a_{h})\in\cX_{h}\times \cA} \frac{\mu_{1:h}(x_{h}, a_{h})}{\mu^{\star, h}_{1:h}(x_{h}, a_{h})} \cdot \sum_{t=1}^T \Delta_t^{x_{h}, a_{h}} \\
  & \le \max_{\mu\in\Pi_{\max}} \sum_{(x_{h}, a_{h})\in\cX_{h}\times \cA} \frac{\mu_{1:h}(x_{h}, a_{h})}{\mu^{\star, h}_{1:h}(x_{h}, a_{h})} \cdot \brac{ \lambda H^2 \mu^{\star, h}_{1:h}(x_{h}, a_{h}) \cdot \sum_{t=1}^T p^{\nu^t}_{1:h}(x_{h}) + \frac{\iota}{\lambda} } \\
  & \le \lambda H^2 \cdot \max_{\mu\in\Pi_{\max}} \sum_{(x_{h}, a_{h})\in\cX_{h}\times \cA} \mu_{1:h}(x_{h}, a_{h}) \sum_{t=1}^T p^{\nu^t}_{1:h}(x_{h}) + \frac{\iota}{\lambda} \cdot \max_{\mu\in\Pi_{\max}} \sum_{(x_{h}, a_{h})\in\cX_{h}\times \cA} \frac{\mu_{1:h}(x_{h}, a_{h})}{\mu^{\star, h}_{1:h}(x_{h}, a_{h})} \\
  & \stackrel{(i)}{=} \lambda H^2 T + \frac{\iota}{\lambda} \cdot X_{h} A.
\end{align*}
Above, (i) used the fact that $\sum_{(x_{h},a_{h})\in\cX_{h}\times\cA} \mu_{1:h}(x_{h}, a_{h}) p^{\nu^t}_{1:h}(x_{h}) = 1$ for any $\mu\in\Pi_{\max}$ and any $t\in[T]$ (Lemma~\ref{lemma:pnu}(a)), as well as the balancing property of $\mu^{\star, h}_{1:h}$ (Lemma~\ref{lemma:balancing}). Combining the bounds for all $h\in[H]$, we get that with probability at least $1-\delta/10$,
\begin{align*}
  \sum_{h=1}^H {\rm BIAS}^2_h \le \lambda H^3T + \frac{XA \iota}{\lambda}.
\end{align*}
Choosing
\begin{align*}
  \lambda=\min\set{ \sqrt{\frac{XA\iota}{H^3T}}, \frac{1}{H}} \le \frac{1}{H},
\end{align*}
we obtain the bound
\begin{align*}
  \sum_{h=1}^H {\rm BIAS}^2_h \le  2\sqrt{H^3XAT \iota} + HXA\iota.
\end{align*}
This is the desired result.
\qed

\subsection{Proof of Lemma~\ref{lemma:cfr-regret}}
\label{appendix:proof-cfr-regret}

Recall that for all $(h, x_h)$, we have implemented Line~\ref{line:cfr-md} of Algorithm~\ref{algorithm:cfr} as the \md~algorithm (Algorithm~\ref{algorithm:md}) with learning rate $\eta \mu^{\star,h}_{1:h}(x_{h}, a)$ and loss vector $\set{\wt{\L}_{h}^t(x_{h}, a)}_{a\in\cA}$ (cf.~\eqref{equation:cfr-hedge}). Therefore, applying the standard regret bound for \md~(Lemma~\ref{lemma:md}), we get (below $a\in\cA$ is arbitrary)
\begin{equation}\label{eqn:reg_imm_upper_bound_MD}
\begin{aligned}
  & \quad \wt{\Reg}_{h}^{\imm, T}(x_{h}) = \max_{\mu^\dagger_{h}(\cdot | x_{h})} \sum_{t=1}^T \< \mu_{h}^t(\cdot | x_{h}) - \mu_{h}^\dagger(\cdot | x_{h}), \wt{\L}_{h}^t(x_{h}, \cdot)\> \\
  & \le \frac{\log A}{\eta \mu^{\star, h}_{1:h}(x_{h}, a) } + \frac{\eta }{2} \cdot \sum_{t=1}^T \sum_{a_{h}\in\cA} \mu^{\star,h}_{1:h}(x_{h}, a_{h}) \cdot \mu_{h}^t(a_{h}|x_{h}) \paren{ \wt{\L}_{h}^t(x_{h}, a_{ h}) }^2 \\
  & \stackrel{(i)}{=} \frac{\log A}{\eta \mu^{\star, h}_{1:h}(x_{h}, a) } \\
  & \qquad + \frac{\eta }{2} \cdot \sum_{t=1}^T \sum_{a_{h}\in\cA} \mu^{\star,h}_{1:h}(x_{h}, a_{h})\mu_{h}^t(a_{h}|x_{h}) \cdot \frac{\paren{H-h+1 - \sum_{h'=h}^H r_{h'}^{\th} }^2 \indic{(x_{h}^{\th}, a_{h}^{\th}) = (x_{h}, a_{h}) }}{ \paren{\mu^{\star, h}_{1:h}(x_{h}, a_{h})}^2 } \\
  & \le \frac{\log A}{\eta \mu^{\star, h}_{1:h}(x_{h}, a) } + \frac{\eta H^2}{2} \cdot \sum_{t=1}^T \sum_{a_{h}\in\cA} \mu_{h}^t(a_{h}|x_{h}) \cdot \frac{ \indic{(x_{h}^{\th}, a_{h}^{\th}) = (x_{h}, a_{h}) } }{\mu^{\star, h}_{1:h}(x_{h}, a_{h})}.
\end{aligned}
\end{equation}
Above, (i) used the form of $\wt{\L}_{h}^t$ in~\eqref{equation:L-estimator}. Plugging this into the definition of ${\rm REGRET}_h$, we have
\begin{equation}\label{eqn:regret_h_bound_MD}
\begin{aligned}
  & \quad {\rm REGRET}_h = \max_{\mu\in\Pi_{\max}} \sum_{x_{h}\in\cX_{h}} \mu_{1:(h-1)}(x_{h-1}, a_{h-1}) \wt{\Reg}_{h}^{\imm, T} (x_{h}) \\
  & \le \underbrace{\max_{\mu\in\Pi_{\max}} \sum_{x_{h}\in\cX_{h}} \mu_{1:(h-1)}(x_{h-1}, a_{h-1}) \cdot \frac{\log A}{\eta \mu^{\star, h}_{1:h}(x_{h}, a)}}_{{\rm I}_h} \\
  & \qquad + \underbrace{\max_{\mu\in\Pi_{\max}} \sum_{x_{h}\in\cX_{h}} \mu_{1:(h-1)}(x_{h-1}, a_{h-1}) \cdot \frac{\eta H^2}{2} \cdot \sum_{t=1}^T \sum_{a_{h}\in\cA} \mu_{h}^t(a_{h}|x_{h}) \cdot \frac{ \indic{(x_{h}^{\th}, a_{h}^{\th}) = (x_{h}, a_{h}) } }{\mu^{\star, h}_{1:h}(x_{h}, a_{h})}}_{{\rm II}_h}.
\end{aligned}
\end{equation}

We first calculate term ${\rm I}_h$. We have
\begin{align*}
  & \quad {\rm I}_h \stackrel{(i)}{=} \frac{\log A}{\eta}\cdot \max_{\mu\in\Pi_{\max}} \sum_{(x_{h}, a_{h})\in\cX_{h}\times \cA} \frac{1}{A_h}\cdot \frac{\mu_{1:(h-1)}(x_{h-1}, a_{h-1})}{\mu^{\star, h}_{1:h}(x_{h}, a_{h})} \\
  & = \frac{\log A}{\eta}\cdot \max_{\mu\in\Pi_{\max}} \sum_{(x_{h}, a_{h})\in\cX_{h}\times \cA} \frac{(\mu_{1:(h-1)}\mu^{\rm unif}_h)(x_{h}, a_{h})}{\mu^{\star, h}_{1:h}(x_{h}, a_{h})} \\
  & \stackrel{(ii)}{=} \frac{\log A}{\eta} \cdot X_{h}A = \frac{X_{h}A\log A}{\eta},
\end{align*}
where (i) follows by splitting the sum over $a_{h}$ and using the fact that $\mu^{\star, h}_{ 1:h}(x_{h}, a)$ does not depend on $a$; (ii) follows from the balancing property of $\mu^{\star, h}_{ 1:h}$ (Lemma~\ref{lemma:balancing}).

Next, we bound term ${\rm II}_h$. We have
\begin{align}
  & \quad {\rm II}_h = \frac{\eta H^2}{2} \max_{\mu\in\Pi_{\max}} \sum_{x_{h}\in\cX_{h}} \mu_{1:(h-1)}(x_{h-1}, a_{h-1}) \cdot \sum_{t=1}^T \sum_{a_{h}\in\cA} \mu_{h}^t(a_{h}|x_{h}) \cdot \frac{ \indic{(x_{h}^{\th}, a_{h}^{\th}) = (x_{h}, a_{h}) } }{\mu^{\star, h}_{1:h}(x_{h}, a_{h})} \nonumber \\
  & = \frac{\eta H^2}{2} \max_{\mu\in\Pi_{\max}} \sum_{x_{h}\in\cX_{h}} \frac{\mu_{1:(h-1)}(x_{h-1}, a_{h-1})}{ \mu^{\star, h}_{1:h}(x_{h}, a) } \cdot \sum_{t=1}^T \underbrace{\sum_{a_{h}\in\cA} \mu_{h}^t(a_{h}|x_{h}) \cdot\indic{(x_{h}^{\th}, a_{h}^{\th}) = (x_{h}, a_{h}) } }_{\defeq \wb{\Delta}_t^{x_{h}}}. \label{equation:cfr-REGRET}
\end{align}
The last equality above used the fact that $\mu^{\star, h}_{1:h}(x_{h}, a_{h})$ does not depend on $a_{h}$ (cf.~\eqref{equation:balanced-policy}).

Observe that the random variables $\wb{\Delta}_t^{x_{h}}$ satisfy the following:
\begin{itemize}
\item $\wb{\Delta}_t^{x_{h}}\in [0, 1]$ almost surely;
\item $\E[\wb{\Delta}_t^{x_{h}}|\cF_{t-1}]=  \sum_{a_{h}\in\cA} \mu^{\star, h}_{1:h}(x_{h}, a_{h})\cdot \mu_{h}^t(a_{h}|x_{h}) p^{\nu^t}_{1:h}(x_{h})$, where $\cF_{t-1}$ is the $\sigma$-algebra containing all information after iteration $t-1$;
\item The conditional variance $\Var[\wb{\Delta}_t^{x_{h}}|\cF_{t-1}]$ can be bounded as
  \begin{align*}
    & \quad \Var\brac{\wb{\Delta}_t^{x_{h}} \Big| \cF_{t-1}} \le \E\brac{\paren{\wb{\Delta}_t^{x_{h}}}^2 \Big| \cF_{t-1}} \\
    & = \E\brac{ \sum_{a_{h}\in\cA} \paren{ \mu_{h}^t(a_{h}|x_{h}) }^2 \indic{ (x_{h}^{\th}, a_{h}^{\th})=(x_{h}, a_{h})} \Big| \cF_{t-1} }\\
    & = \sum_{a_{h}\in\cA} \paren{\mu_{h}^t(a_{h}|x_{h}) }^2 \cdot \P^{\mu^{\star, h}_{1:h}\times \nu^t}\paren{ (x_{h}^{\th}, a_{h}^{\th})=(x_{h}, a_{h}) } \\
    & = \sum_{a_{h}\in\cA} \mu^{\star,h}_{1:h}(x_{h}, a_{h}) \cdot \paren{\mu_{h}^t(a_{h}|x_{h}) }^2 \cdot p^{\nu^t}_{1:h}(x_{h}).
  \end{align*}
\end{itemize}
Therefore, we can apply Freedman's inequality (Lemma~\ref{lemma:freedman}) and a union bound to obtain that, for any $\lambda\in(0,1]$, with probability at least $1-\delta/10$, the following holds simultaneously for all $(h, x_{h})$:
\begin{align*}
  & \quad \sum_{t=1}^T \wb{\Delta}_t^{x_{h}} - \sum_{t=1}^T \sum_{a_{h}\in\cA} \mu^{\star, h}_{1:h}(x_{h}, a_{h})\cdot \mu_{h}^t(a_{h}|x_{h}) p^{\nu^t}_{1:h}(x_{h}) \\
  & \le \lambda \cdot \sum_{t=1}^T \sum_{a_{h}\in\cA} \mu^{\star,h}_{1:h}(x_{h}, a_{h}) \cdot \paren{\mu_{h}^t(a_{h}|x_{h}) }^2 \cdot p^{\nu^t}_{1:h}(x_{h}) + \frac{\iota}{\lambda},
\end{align*}
where $\iota\defeq \log(10X/\delta)$ is a log factor. Plugging this bound into~\eqref{equation:cfr-REGRET} yields that, for all $h\in[H]$, 
\begin{align*}
  & \quad {\rm II}_h \le \frac{\eta H^2}{2}\cdot \max_{\mu\in\Pi_{\max}} \sum_{x_{h}\in\cX_{h}} \frac{\mu_{1:(h-1)}(x_{h-1}, a_{h-1})}{ \mu^{\star, h}_{1:h}(x_{h}, a) } \cdot \sum_{t=1}^T \sum_{a_{h}\in\cA} \mu^{\star, h}_{1:h}(x_{h}, a_{h})\cdot \mu_{h}^t(a_{h}|x_{h}) p^{\nu^t}_{1:h}(x_{h}) \\
  & \qquad +  \frac{\eta H^2}{2}\cdot \max_{\mu\in\Pi_{\max}} \sum_{x_{h}\in\cX_{h}} \frac{\mu_{1:(h-1)}(x_{h-1}, a_{h-1})}{ \mu^{\star, h}_{1:h}(x_{h}, a) } \\
  & \qquad\qquad \qquad \qquad \qquad \cdot \brac{ \lambda \sum_{t=1}^T \sum_{a_{h}\in\cA} \mu^{\star,h}_{1:h}(x_{h}, a_{h}) \cdot \paren{\mu_{h}^t(a_{h}|x_{h}) }^2 \cdot p^{\nu^t}_{1:h}(x_{h}) + \frac{\iota}{\lambda} } \\
  & \stackrel{(i)}{\le} \frac{\eta H^2}{2}\cdot \max_{\mu\in\Pi_{\max}} \sum_{(x_{h}, a_{h})\in\cX_{h}\times \cA} \mu_{1:(h-1)}(x_{h-1}, a_{h-1})\cdot \sum_{t=1}^T \mu_{h}^t(a_{h}|x_{h}) p^{\nu^t}_{1:h}(x_{h}) \\
  & \qquad + \frac{\eta H^2}{2}\cdot \max_{\mu\in\Pi_{\max}} \sum_{(x_{h}, a_{h})\in\cX_{h}\times \cA} \mu_{1:(h-1)}(x_{h-1}, a_{h-1}) \cdot \lambda \sum_{t=1}^T \paren{\mu_{h}^t(a_{h}|x_{h}) }^2 \cdot p^{\nu^t}_{1:h}(x_{h}) \\
  & \qquad + \frac{\eta H^2}{2} \cdot \frac{\iota}{\lambda} \cdot \max_{\mu\in\Pi_{\max}} \sum_{x_{h}\in\cX_{h}} \frac{\mu_{1:(h-1)}(x_{h-1}, a_{h-1})}{ \mu^{\star, h}_{1:h}(x_{h}, a) } \\
  & \stackrel{(ii)}{\le} \frac{\eta H^2}{2}(1+\lambda) \cdot \max_{\mu\in\Pi_{\max}} \sum_{(x_{h}, a_{h})\in\cX_{h}\times \cA} \mu_{1:(h-1)}(x_{h-1}, a_{h-1})\cdot \sum_{t=1}^T \mu_{h}^t(a_{h}|x_{h}) p^{\nu^t}_{1:h}(x_{h}) \\
  & \qquad + \frac{\eta H^2}{2} \cdot \frac{\iota}{\lambda} \cdot \max_{\mu\in\Pi_{\max}} \sum_{(x_{h}, a_{h})\in\cX_{h}\times \cA} \frac{(\mu_{1:(h-1)}\mu^{\rm unif}_h)(x_{h}, a_{h})}{ \mu^{\star, h}_{1:h}(x_{h}, a_{h}) } \\
  & \stackrel{(iii)}{=} \frac{\eta H^2}{2}(1+\lambda)T + \frac{\eta H^2}{2} \cdot \frac{\iota}{\lambda} \cdot X_{h}A.
\end{align*}
Above, (i) used again the fact that $\mu^{\star, h}_{1:h}(x_{h}, a)=\mu^{\star, h}_{1:h}(x_{h}, a_{h})$ for any $a, a_{h}\in\cA$; (ii) used the fact that $\mu_{h}^t(a_{h}|x_{h})\le 1$; (iii) used the fact that $\sum_{(x_{h}, a_{h})\in \cX_{h}\times \cA} (\mu_{1:(h-1)}\mu_{h}^t)(x_{h}, a_{h}) p^{\nu^t}_{1:h}(x_{h})=1$ for any $\mu\in\Pi_{\max}$ and any $t\in[T]$ (Lemma~\ref{lemma:pnu}(a)), as well as the balancing property of $\mu^{\star, h}_{1:h}$ (Lemma~\ref{lemma:balancing}).

Combining the bounds for ${\rm I}_h$ and ${\rm II}_h$, we obtain that
\begin{align*}
  & \quad \sum_{h=1}^H {\rm REGRET}_h \le \sum_{h=1}^H ({\rm I}_h + {\rm II}_h) \\
  & \le \sum_{h=1}^H \brac{ \frac{X_{h}A\log A}{\eta} + \frac{\eta H^2}{2}(1+\lambda)T + \frac{\eta H^2X_{h}A\iota}{2\lambda} } \\
  & \le \frac{XA \iota}{\eta} + \frac{\eta H^3}{2}T + \frac{\eta H^2}{2}\brac{\lambda\cdot HT + \frac{XA \iota}{\lambda}},
\end{align*}
where we have redefined the log factor $\iota\defeq \log(10XA/\delta)$. Choosing $\lambda=1$,
the above can be upper bounded by 
\begin{align*}
  \frac{XA\iota}{\eta} + \eta H^3T + \frac{\eta H^2 XA\iota}{2}.
\end{align*}
Further choosing $\eta=\sqrt{XA\iota/(H^3T)}$, we obtain the bound
\begin{align*}
  \sum_{h=1}^H {\rm REGRET}_h \le 2\sqrt{H^3XAT\iota} + \sqrt{HX^3A^3\iota^3/(4T)}.
\end{align*}
This is the desired result.
\qed

%% file: Sections_arxiv/full-feedback.tex
\section{Results for full feedback setting}
\label{appendix:full-feedback}

We now consider the full feedback setting in which the algorithm can receive the true loss $\set{\L_h^t(x_h, a_h)}_{h, x_h, a_h}$ within each round. Define
\begin{align*}
    \lone{\Pi_{\max}} \defeq \max_{\mu\in\Pi_{\max}} \sum_{h=1}^H \sum_{(x_h, a_h)\in\cX_h\times\cA} \mu_{1:h}(x_h, a_h).
\end{align*}
Note that we have the bound $\lone{\Pi_{\max}}\le X$ (as $\sum_{a_h} \mu_{1:h}(x_h, a_h) =\mu_{1:h-1}(x_{h-1}, a_{h-1}) \le 1$ for all $(h, x_h)$), but $\lone{\Pi_{\max}}$ may in general be much smaller than $X$~\citep{zhou2020lazy}.

\subsection{Regret bound for CFR algorithm}

\begin{algorithm}[h]
  \small
  \caption{CFR with Hedge (full feedback setting)}
  \label{algorithm:cfr-full-feedback}
  \begin{algorithmic}[1]
    \REQUIRE Learning rate $\eta>0$.
    \STATE Initialize policies $\mu_{h}^1(a_{h}|x_{h})\setto 1/A$ for all $h\in[H]$, $(x_h, a_h)\in\cX_h\times \cA$.
    \FOR{iteration $t=1,\dots,T$}
    \FOR{all $h\in[H]$ and $x_{h}\in\mc{X}_{h}$}
    \STATE Receive loss $\L_h^t(x_h, a_h)$ for all $a_h\in\cA$. 
    \STATE Update policy at $x_h$ using Hedge:
    \begin{align*}
       \mu_h^{t+1}(a|x_h) \propto_a \mu_{h}^t(a|x_{h}) \cdot e^{-\eta \cdot \L^t_{h}(x_{h}, a)}.
    \end{align*}
    \label{line:cfr-full-feedback-md}
    \ENDFOR
    \ENDFOR
  \end{algorithmic}
\end{algorithm}

We present a ``vanilla" CFR algorithm for the full-feedback case in Algorithm~\ref{algorithm:cfr-full-feedback} (with regret minimizers are instantiated as Hedge). The algorithm is similar as Balanced CFR for the bandit-feedback case (Algorithm~\ref{algorithm:cfr}), yet is simpler as it does not need the sampling step, and uses a constant learning rate for all infosets. 

The following sharp regret bound for CFR in the full-feedback case has appeared in~\citep[Lemma 2]{zhou2020lazy}\footnote{modulo the (minor) difference of them using Regret Matching instead of Hedge}. For completeness, we provide a full statement and proof under our notation.

\begin{theorem}[Regret of CFR in full feedback setting]
  \label{theorem:cfr-full-feedback}
  For the full feedback setting, Algorithm~\ref{algorithm:cfr-full-feedback} with learning rate $\eta=\sqrt{2\lone{\Pi_{\max}}\log A/(H^3T)}$ achieves regret bound
  \begin{align*}
    \Reg^T \le \sqrt{2H^3\lone{\Pi_{\max}} T \log A}.
  \end{align*}
\end{theorem}

\begin{proof}
Note that Algorithm~\ref{algorithm:cfr-full-feedback} updates the policy at each $x_h$ using the Hedge algorithm (Algorithm~\ref{algorithm:md}) with loss vector $\set{L_h^t(x_h, a)}_{a\in\cA}$ and learning rate $\eta$. Therefore, plugging the standard regret bound for Hedge (Lemma~\ref{lemma:md}) into the counterfactual regret decomposition (Lemma~\ref{lemma:cfr-regret-decomposition}), we get 
\begin{align*}
  & \quad \Reg^T \stackrel{(i)}{\le} \max_{\mu\in\Pi_{\max}} \sum_{h=1}^H \sum_{x_h\in\cX_h} \mu_{1:h-1}(x_{h-1}, a_{h-1}) \Reg^{\imm, T}_h(x_h) \\
  & = \max_{\mu\in\Pi_{\max}} \sum_{h=1}^H \sum_{x_h\in\cX_h} \mu_{1:h-1}(x_{h-1}, a_{h-1}) \cdot \max_{\mu_{h}^\dagger(\cdot|x_{h})} \sum_{t=1}^T \< \mu_{h}^t(\cdot | x_{h}) - \mu_{h}^\dagger(\cdot | x_{h}), \L_{h}^t(x_{h}, \cdot)\> \\
  & \le \max_{\mu\in\Pi_{\max}} \sum_{h=1}^H \sum_{x_h\in\cX_h} \mu_{1:h-1}(x_{h-1}, a_{h-1}) \cdot \brac{ \frac{\log A}{\eta} + \frac{\eta}{2}\sum_{t=1}^T\sum_{a\in\cA_h} \mu_h^t(a|x_h) L_h^t(x_h, a)^2 } \\
  & \stackrel{(ii)}{\le} \max_{\mu\in\Pi_{\max}} \sum_{h=1}^H \sum_{x_h\in\cX_h} \mu_{1:h-1}(x_{h-1}, a_{h-1}) \cdot \brac{ \frac{\log A}{\eta} + \frac{\eta}{2}\sum_{t=1}^T\sum_{a\in\cA_h} \mu_h^t(a|x_h) \cdot p^{\nu^t}_{1:h}(x_h)^2 H^2 } \\
  & \stackrel{(iii)}{\le} \max_{\mu\in\Pi_{\max}} \sum_{h, x_h} \mu_{1:h-1}(x_{h-1}, a_{h-1}) \cdot \frac{\log A}{\eta} +   \frac{\eta H^2 }{2}\max_{\mu\in\Pi_{\max}} \sum_{h, x_h} \mu_{1:h-1}(x_{h-1}, a_{h-1}) \sum_{t=1}^T\sum_{a\in\cA_h} \mu_h^t(a|x_h) \cdot p^{\nu^t}_{1:h}(x_h) \\
  & \stackrel{(iv)}{\le} \frac{\lone{\Pi_{\max}}\log A}{\eta} +  \frac{\eta H^2 }{2}\max_{\mu\in\Pi_{\max}} \sum_{h=1}^H  \sum_{t=1}^T \sum_{(x_h, a)\in\cX_h\times\cA} \paren{\mu_{1:h-1} \mu_h^t}(x_h, a) \cdot p^{\nu^t}_{1:h}(x_h) \\
  & \stackrel{(v)}{=} \frac{\lone{\Pi_{\max}}\log A}{\eta} +  \frac{\eta H^3 T}{2}.
\end{align*}
Above, (i) uses a stronger form of Lemma~\ref{lemma:cfr-regret-decomposition} (with $\max_{\mu\in\Pi_{\max}}$ outside of the $\sum_{h=1}^H$, which could also be extracted from the proof of Lemma~\ref{lemma:cfr-regret-decomposition}); (ii) used the bound for $L_h^t(x_h, a)$ in Lemma~\ref{lemma:counterfactual-loss-bound}(b); (iii) pushed the max into each of the two parts, and used $p^{\nu^t}_{1:h}(x_h)\le 1$ (Lemma~\ref{lemma:pnu}(b)); (iv) used the fact that $\sum_{h, x_h}\mu_{1:h-1}(x_{h-1}, a_{h-1}) = \sum_{h, x_h, a_h} \mu_{1:h}(x_h, a_h)\le \lone{\Pi_{\max}}$ for any $\mu\in\Pi_{\max}$; (v) used Lemma~\ref{lemma:pnu}(a).

Choosing $\eta=\sqrt{2\lone{\Pi_{\max}}\log A/(H^3T)}$ yields the desired regret bound.

\end{proof}

%% file: Sections_arxiv/regret-matching.tex
\section{Balanced CFR with regret matching}
\label{appendix:cfr-rm}

In this section, we consider instantiating Line~\ref{line:cfr-md} of Algorithm~\ref{algorithm:cfr} using the following Regret Matching algorithm:
\begin{align}
  \label{equation:cfr-rm-update}
  \begin{aligned}
    & \mu_h^{t+1}(a|x_h) = \frac{\brac{R^t_{x_h}(a)}_+}{ \sum_{a'\in\cA} \brac{R^t_{x_h}(a')}_+ }, \\
    & {\rm where}~~R^t_{x_h}(a) \defeq \sum_{\tau=1}^t \<\mu_h^\tau(\cdot|x_h), \wt{\L}^\tau_h(x_h, \cdot)\> - \wt{\L}^\tau_h(x_h, a)~~~\textrm{for all}~a\in\cA.
  \end{aligned}
\end{align}

We now present the main theoretical guarantees for Balanced CFR with regret matching. The proof of Theorem~\ref{theorem:cfr-rm} can be found in Section~\ref{appendix:proof-cfr-rm}.
\begin{theorem}[``Regret'' bound for Balanced CFR with Regret Matching]
  \label{theorem:cfr-rm}
  Suppose the max player plays Algorithm~\ref{algorithm:cfr} where each $R_{x_h}$ is instantiated as the Regret Matching algorithm~\eqref{equation:cfr-rm-update}. Then the policies $\mu^t$ achieve the following regret bound with probability at least $1-\delta$:
  \begin{align*}
    \wt{\Reg}^T \defeq \max_{\mu^\dagger\in\Pi_{\max}} \sum_{t=1}^T \<\mu^t - \mu^\dagger, \ell^t \> \le \cO(\sqrt{H^3XA^2T\iota}),
  \end{align*}
  where $\iota=\log(10XA/\delta)$ is a log factor. Further, each round plays $H$ episodes against $\nu^t$ (so that the total number of episodes played is $HT$). 
\end{theorem}

We then have the following corollary directly by the regret-to-Nash conversion (Proposition~\ref{proposition:online-to-batch}).
\begin{corollary}[Learning Nash using Balanced CFR with Regret Matching]
  \label{corollary:cfr-rm-pac}
  Letting both players play Algorithm~\ref{algorithm:cfr} in a self-play fashion against each other for $T$ rounds, where each $R_{x_h}$ is instantiated as the Regret Matching algorithm~\eqref{equation:cfr-rm-update}. Then, for any $\eps>0$, the average policy $(\wb{\mu}, \wb{\nu})=(\frac{1}{T}\sum_{t=1}^T\mu^t, \frac{1}{T}\sum_{t=1}^T\nu^t)$ achieves $\negap(\wb{\mu}, \wb{\nu})\le \eps$ with probability at least $1-\delta$, as long as
  \begin{align*}
    T \ge \cO(H^3(XA^2+YB^2) \iota/\eps^2),
  \end{align*}
  where $\iota\defeq \log(10(XA+YB)/\delta)$ is a log factor. The total amount of episodes played is at most
  \begin{align*}
    2H\cdot T = \cO(H^4(XA^2+YB^2) \iota/\eps^2).
  \end{align*}
\end{corollary}

\subsection{Proof of Theorem~\ref{theorem:cfr-rm}}
\label{appendix:proof-cfr-rm}

The proof is similar as Theorem \ref{theorem:cfr}, except for plugging in the regret bound for Regret Matching instead of Hedge.

First, by Lemma~\ref{lemma:cfr-regret-decomposition}, we have $\wt{\Reg}^T\le \sum_{h=1}^H \Reg_{h}^T$, where for any $h\in[H]$ we have
\begin{equation}\label{eqn:CFR_RM_R_h_T_decomposition}
\begin{aligned}
  & \quad \Reg_{h}^T \le \underbrace{\max_{\mu\in\Pi_{\max}} \sum_{x_{h}\in\cX_{h}} \mu_{1:(h-1)}(x_{h-1}, a_{h-1}) \wt{\Reg}_{h}^{\imm, T}(x_{h})}_{\defeq {\rm REGRET}_h} \\
  & \qquad + \underbrace{\max_{\mu\in\Pi_{\max}}  \sum_{(x_{h}, a_{h})\in\cX_{h}\times \cA} \mu_{1:(h-1)}(x_{h-1}, a_{h-1}) \sum_{t=1}^T \mu_{h}^t(a_{h} | x_{h})\brac{ \L_{h}^t(x_{h}, a_{h}) - \wt{\L}_{h}^t(x_{h}, a_{h})} }_{\defeq {\rm BIAS}_h^1} \\
  & \qquad + \underbrace{\max_{\mu\in\Pi_{\max}} \sum_{(x_{h}, a_{h})\in\cX_{h}\times \cA} \mu_{1:h}(x_{h}, a_{h}) \sum_{t=1}^T \brac{\wt{\L}_{h}^t(x_{h}, a_{h}) - \L_{h}^t(x_{h}, a_{h})} }_{\defeq {\rm BIAS}_h^2} \\
  & = {\rm REGRET}_h + {\rm BIAS}^1_h + {\rm BIAS}^2_h,
\end{aligned}
\end{equation}
where the definition of $\wt{\Reg}_{h}^{\imm, T}(x_{h})$, $\L_{h}^t(x_{h}, a_{h})$ are at the beginning of Section \ref{sec:CFR_decomposition} and the definition of $\wt{\L}_{h}^t(x_{h}, a_{h})$ are given by Algorithm \ref{algorithm:cfr}.  

To upper bound ${\rm BIAS}^1_h$ and ${\rm BIAS}^2_h$, we use the same strategy as the proof of Lemma \ref{lemma:cfr-bias1} and \ref{lemma:cfr-bias2} (whose proofs are independent of the regret minimizer), so that we have the same bound as in Lemma \ref{lemma:cfr-bias1} and \ref{lemma:cfr-bias2}: with probability at least $1 - \delta/5$, we have
\begin{equation}\label{eqn:BIAS_12_RM}
\sum_{h = 1}^H {\rm BIAS}^1_h \le 2 \sqrt{H^3 X A T \iota} + H X \iota, ~~~~  \sum_{h = 1}^H {\rm BIAS}^2_h \le 2 \sqrt{H^3 X A T \iota} + H X A \iota,
\end{equation}
where $\iota = \log(10 X A / \delta)$. 

To upper bound ${\rm REGRET}_h$, we use the same strategy as the proof of Lemma \ref{lemma:cfr-regret} as in Section \ref{appendix:proof-cfr-regret}. First, applying the regret bound for Regret Matching (Lemma~\ref{lemma:regmatch} \& Remark~\ref{rmk:regmatch}), we get (below $a\in\cA$ is arbitrary, and $\eta > 0$ is also arbitrary)
\begin{equation}\label{eqn:reg_imm_upper_bound_RM}
\begin{aligned}
  & \quad \wt{\Reg}_{h}^{\imm, T}(x_{h}) = \max_{\mu^\dagger_{h}(\cdot | x_{h})} \sum_{t=1}^T \< \mu_{h}^t(\cdot | x_{h}) - \mu_{h}^\dagger(\cdot | x_{h}), \wt{\L}_{h}^t(x_{h}, \cdot)\> \\
  & \le \frac{1}{\eta \mu^{\star, h}_{1:h}(x_{h}, a) } + \frac{\eta }{2} \cdot \sum_{t=1}^T \sum_{a_{h}\in\cA} \mu^{\star,h}_{1:h}(x_{h}, a_{h}) \cdot A \bar \mu_{h}^t(a_{h}|x_{h}) \paren{ \wt{\L}_{h}^t(x_{h}, a_{ h}) }^2 \\
  & \le \frac{1}{\eta \mu^{\star, h}_{1:h}(x_{h}, a) } + \frac{\eta H^2}{2} \cdot \sum_{t=1}^T \sum_{a_{h}\in\cA} A \cdot \bar \mu_{h}^t(a_{h}|x_{h}) \cdot \frac{ \indic{(x_{h}^{\th}, a_{h}^{\th}) = (x_{h}, a_{h}) } }{\mu^{\star, h}_{1:h}(x_{h}, a_{h})},
\end{aligned}
\end{equation}
where $\bar \mu_h^t(a_h \vert x_h) = (\mu_h^t(a_h \vert x_h) + (1/A))/2$ is a probability distribution over $[A]$. Comparing the right hand side of Eq. (\ref{eqn:reg_imm_upper_bound_RM}) with the right hand side of Eq. (\ref{eqn:reg_imm_upper_bound_MD}), we can see that there is only one difference which is $A \cdot \bar \mu_{h}^t$ versus $\mu_h^t$. Plugging this into the definition of ${\rm REGRET}_h$, we have
\begin{equation}\label{eqn:regret_h_bound_RM}
\begin{aligned}
  & \quad {\rm REGRET}_h = \max_{\mu\in\Pi_{\max}} \sum_{x_{h}\in\cX_{h}} \mu_{1:(h-1)}(x_{h-1}, a_{h-1}) \wt{\Reg}_{h}^{\imm, T} (x_{h}) \\
  & \le \underbrace{\max_{\mu\in\Pi_{\max}} \sum_{x_{h}\in\cX_{h}} \mu_{1:(h-1)}(x_{h-1}, a_{h-1}) \cdot \frac{1}{\eta \mu^{\star, h}_{1:h}(x_{h}, a)}}_{{\rm I}_h} \\
  & \qquad + \underbrace{\max_{\mu\in\Pi_{\max}} \sum_{x_{h}\in\cX_{h}} \mu_{1:(h-1)}(x_{h-1}, a_{h-1}) \cdot \frac{\eta H^2}{2} \cdot \sum_{t=1}^T \sum_{a_{h}\in\cA} A \cdot \bar \mu_{h}^t(a_{h}|x_{h}) \cdot \frac{ \indic{(x_{h}^{\th}, a_{h}^{\th}) = (x_{h}, a_{h}) } }{\mu^{\star, h}_{1:h}(x_{h}, a_{h})}}_{{\rm II}_h}.
\end{aligned}
\end{equation}
Comparing Eq. (\ref{eqn:regret_h_bound_RM}) with Eq. (\ref{eqn:regret_h_bound_MD}), we can see that ${\rm I}_h$ in Eq. (\ref{eqn:regret_h_bound_RM}) is the same as ${\rm I}_h$ in Eq. (\ref{eqn:regret_h_bound_MD}), and ${\rm II}_h$ in Eq. (\ref{eqn:regret_h_bound_RM}) and (\ref{eqn:regret_h_bound_MD}) only have one difference which is also $A \cdot \bar \mu_h^t$ versus $\mu_h^t$. Using the same argument as in the former proof, we have 
\[
{\rm I}_h = \frac{X_h A}{\eta}. 
\]
Furthermore, using the same argument as in the former proof, we can show that the upper bound of ${\rm II}_h$ in Eq. (\ref{eqn:regret_h_bound_RM}) is at most $A$ times the upper bound of ${\rm II}_h$ in Eq. (\ref{eqn:regret_h_bound_MD}). This gives for any $\lambda \in (0, 1)$, with probability at least $1 - \delta / 10$, we have
\[
{\rm II}_h \le \frac{\eta H^2 A}{2}(1+\lambda)T + \frac{\eta H^2}{2} \cdot \frac{\iota}{\lambda} \cdot X_{h}A^2.
\]
Combining the bounds for ${\rm I}_h$ and ${\rm II}_h$, we obtain that
\begin{align*}
  & \quad \sum_{h=1}^H {\rm REGRET}_h \le \sum_{h=1}^H ({\rm I}_h + {\rm II}_h) \le \frac{XA}{\eta} + \frac{\eta H^3 A}{2}T + \frac{\eta H^2 A}{2}\brac{\lambda\cdot HT + \frac{XA \iota}{\lambda}},
\end{align*}
Choosing $\lambda=1$ and choosing $\eta=\sqrt{X\iota/(H^3T)}$, with probability at least $1 - \delta / 10$, we obtain the bound
\begin{align}\label{eqn:REGRET_h_bound_RM}
  \sum_{h=1}^H {\rm REGRET}_h \le 2\sqrt{H^3XA^2T\iota} + \sqrt{HX^3A^4\iota^3/(4T)}.
\end{align}
This bound is $\sqrt{A}$ times larger than the bound of $\sum_{h=1}^H {\rm REGRET}_h $ as in Lemma \ref{lemma:cfr-regret}. 

Combining Eq. (\ref{eqn:CFR_RM_R_h_T_decomposition}), (\ref{eqn:BIAS_12_RM}) and (\ref{eqn:REGRET_h_bound_RM}), we obtain the following: with probability at least $1-3\delta/10\ge 1-\delta$, we have
\begin{align*}
  & \quad \wt{\Reg}^T \le \sum_{h=1}^H \Reg_{h}^T \le \sum_{h=1}^H {\rm REGRET}_h + \sum_{h=1}^H {\rm BIAS}^1_h + \sum_{h=1}^H {\rm BIAS}^2_h \\
  & \le 6\sqrt{H^3XA^2T\iota} + 2HXA\iota + \sqrt{HX^3A^4\iota^3/(4T)}.  
\end{align*}
Additionally, recall the naive bound $\wt{\Reg}^T\le HT$ on the regret (which follows as $\<\mu^t,\ell^t\>\in [0, H]$ for any $\mu\in\Pi_{\max}$, $t\in[T]$), we get
\begin{align*}
    & \quad \wt{\Reg}^T \le \min\set{ 6\sqrt{H^3XA^2 T\iota} + 2HXA\iota + \sqrt{HX^3A^4\iota^3/4T}, HT } \\
    & \le HT\cdot \min\set{6\sqrt{HXA^2 \iota/T} + 2XA\iota/T + \sqrt{X^3A^4\iota^3/(4HT^3)}, 1 }.
\end{align*}
For $T>HXA^2 \iota$, the min above is upper bounded by $9\sqrt{HXA^2 \iota/T}$. For $T\le HXA^2\iota$, the min above is upper bounded by $1\le 9\sqrt{HXA^2\iota/T}$. Therefore, we always have
\begin{align*}
    \wt{\Reg}^T \le HT\cdot 9\sqrt{HXA^2 \iota/T} = 9\sqrt{H^3XA^2 T\iota}.
\end{align*}
This is the desired result.
\qed

%% file: Sections_arxiv/multi-player-full.tex
\section{Results for multi-player general-sum games}
\label{appendix:multi-player}

We introduce multi-player general-sum games with imperfect information and show when all the players run Algorithm~\ref{alg:IXOMD} or Algorithm~\ref{algorithm:cfr} independently, the average policy is an approximate Coarse Correlated Equilibrium policy.

\subsection{Multi-player general-sum games}
\label{appendix:multi-player-definition}

Here we define an $m$-player general-sum IIEFG with tree structure and perfect recall. Our definition is parallel to the POMG formulation for two-player zero-sum IIEFGs in Section~\ref{sec:prelim}. 

\paragraph{Partially observable Markov games}
We consider finite-horizon, tabular, $m$-player general-sum Markov Games with partial observability. Formally,  it can be described as a $\POMG(H, \cS, \{\cX_i\}_{i=1}^m, \{\cA_i\}_{i=1}^m, \P, \{r_i\}_{i=1}^m)$, where
\begin{itemize}[wide, itemsep=0pt, topsep=0pt]
\item $H$ is the horizon length;
\item $\cS=\bigcup_{h\in[H]} \cS_h$ is the (underlying) state space;
\item $\cX_i=\bigcup_{h\in[H]} \cX_{i,h}$ is the space of infosets for the \emph{$i$-th player} with $|\cX_{i,h}|=X_{i,h}$ and $X_i\defeq \sum_{h=1}^H X_{i,h}$. At any state $s_h\in\cS_h$, the $i$-th player only observes the infoset $x_{i,h}=x_i(s_h)\in\cX_{i,h}$, where $x_i:\cS\to\cX_i$ is the emission function for the $i$-th player;
\item $\cA_i$ is the action spaces for the $i$-th player with $\abs{\cA_i}=A_i$. For any $h$, we define the joint action of $m$ players by $\mathbf{a}_h:=(a_{1,h},\cdots, a_{m,h})$ and the set of joint actions by $\cA:=\cA_1 \times \cdots \times \cA_m$.
\item $\P=\{p_1(\cdot)\in\Delta(\cS_1)\} \cup \{p_h(\cdot|s_h, \mathbf{a}_h)\in \Delta(\cS_{h+1})\}_{(s_h, \mathbf{a}_h)\in \cS_h\times \cA,~h\in[H-1]}$ are the transition probabilities, where $p_1(s_1)$ is the probability of the initial state being $s_1$, and $p_h(s_{h+1}|s_h, \mathbf{a}_h)$ is the probability of transitting to $s_{h+1}$ given state-action $(s_h, a_h, b_h)$ at step $h$;
\item $r_i=\set{r_{i,h}(s_h, \mathbf{a}_h)\in[0,1]}_{(s_h, \mathbf{a}_h)\in\cS_h\times\cA}$ are the (random) reward functions with mean $\wb{r}_{i,h}(s_h, \mathbf{a}_h)$.
\end{itemize}

\paragraph{Policies, value functions}
A policy for the $i$-th player is denoted by $\pi_i=\set{\pi_{i,h}(\cdot|x_{i,h})\in\Delta(\cA_i)}_{h\in[H], x_{i,h}\in\cX_{i,h}}$, where $\pi_{i,h}(a_{i,h}|x_{i,h})$ is the probability of taking action $a_{i,h}\in\cA_i$ at infoset $x_{i,h}\in\cX_{i,h}$. A trajectory for the $i$-th player takes the form $(x_{i,1}, a_{i,1}, r_{i,1}, x_{i,2}, \dots, x_{i,H}, a_{i,H}, r_{i,H})$, where $a_{i,h}\sim \pi_{i,h}(\cdot|x_{i,h})$, which depends on both the other (unseen) players' policy and underlying state transition. 

We use $\pi$ to denote the joint policy. Notice although the marginals are $\pi_i$, $\pi$ is not necessarily a product policy. When $\pi$ is indeed a product policy, we have $\pi = \pi_1 \times \cdots \times \pi_m $. We also use $\pi_{-i}$ to denote the joint product policy excluding the $i$-th player. The overall game value of the $i$-th player for any joint policy $\pi$ is denoted by $V^{\pi}_i\defeq \E_{\pi}\brac{ \sum_{h=1}^H r_{i,h}(s_h,\mathbf{a}_h) }$.

\paragraph{Tree structure and perfect recall}
As before, we assume
\begin{itemize}
    \item \emph{Tree structure}: for any $h$ and $s_h\in\cS_h$, there exists a unique history $(s_1, \mathbf{a}_{1}, \dots, s_{h-1}, \mathbf{a}_{h-1})$ of past states and (joint) actions that leads to $h$.
    \item \emph{Perfect recall}: For any $h$ and any infoset $x_{i,h}\in\cX_{i,h}$ for the $i$-th player, there exists a unique history $(x_{i,1}, a_{i,1}, \dots, x_{i,h-1}, a_{i,h-1})$ of past infosets and $i$-th player's actions that leads to $x_{i,h}$.
\end{itemize}

Given above conditions, under any product policy $\pi$, the probability of reaching state-action $(s_h, \mathbf{a}_h)$ at step $h$ takes the form
\begin{align}
  \P^{\pi}(s_h, \mathbf{a}_h) = p_{1:h}(s_h)\prod_{i=1}^m{\pi _{i,1:h}\left( x_{i,h},a_{i,h} \right)},
\end{align}
where $\set{s_{h'}, \mathbf{a}_{h'}}_{h'\le h-1}$ are the histories uniquely determined from $s_h$ and $x_{i,h'}=x_i(s_{h'})$. We have also defined the sequence-form transition probability as
\begin{align*}
  p_{1:h}(s_h)\defeq p_1(s_1)\prod_{h\prime\le h-1}{p_{h\prime}}(s_{h\prime+1}|s_{h\prime},\mathbf{a}_{h\prime}),
\end{align*}
and the \emph{sequence-form policies} as
\begin{align*}
  \pi_{i,1:h}\left( x_{i,h},a_{i,h} \right) \defeq \prod_{h\prime=1}^h{\pi_{i,h'}\left( a_{i,h'}|x_{i,h'} \right)}.
\end{align*}

\paragraph{Regret and CCE}
Similar as how regret minimization in two-player zero-sum games leads to an approximate Nash equilibrium (Proposition~\ref{proposition:online-to-batch}), in multi-player general-sum games, regret minimization is known to lead to an approximate NFCCE. Let $\set{\pi^t}_{t=1}^T$ denote a sequence of joint policies (for all players) over $T$ rounds. The regret of the $i$-th player is defined by
\begin{align*}
  \Reg_{i}^T \defeq \max_{\pi_i^\dagger\in\Pi_{i}} \sum_{t=1}^T \paren{V_i^{\pi_i^\dagger, \pi_{-i}^t} - V_i^{\pi^t}}.
\end{align*}
where $\Pi_{i}$ denotes the set of all possible policies for the $i$-th player. 

Using online-to-batch conversion, it is a standard result that sub-linear regret for all the players ensures that the average policy $\wb{\pi}$ is an approximate NFCCE~\citep{celli2019computing}.

\begin{proposition}[Regret-to-CCE conversion for multi-player general-sum games]
  \label{proposition:online-to-batch-cce}
  Let the average policy $\wb{\pi}$ be defined as playing a policy within $\set{\pi^t}_{t=1}^T$ uniformly at random, then we have
  \begin{align*}
    \ccegap(\wb{\pi}) = \frac{\max_{i\in[m]} \Reg_{i}^T}{T}.
  \end{align*}
\end{proposition}
We include a short justification for this standard result here for completeness.

\begin{proof}
By definition of $\wb{\pi}$, we have for any $i \in [m]$ and $\pi_i^\dagger \in \Pi_{i}$ that
$$
V_i^{\pi_i^\dagger, \wb{\pi}_{-i}} - V_i^{\wb{\pi}} = \frac{1}{T}\sum_{t=1}^T \paren{V_i^{\pi_i^\dagger, \pi_{-i}^t} - V_i^{\pi^t}}.
$$
Taking the max over $\pi_i^\dagger \in \Pi_{i}$ and $i\in[m]$ on both sides yields the desired result.
\end{proof}

\subsection{Proof of Theorem~\ref{theorem:nfcce}}
\label{appendix:proof-nfcce}

It is straightforward to see that the regret guarantees for Balanced OMD (Theorem~\ref{theorem:ixomd-regret}) and Balanced CFR (Theorem~\ref{theorem:cfr}) also hold in multi-player general-sum games (e.g. by modeling all other players as a single opponent). Therefore, the regret-to-CCE conversion in Proposition~\ref{proposition:online-to-batch-cce} directly implies that, letting $\wb{\pi}$ denote the joint policy of playing a uniformly sampled policy within $\set{\pi^t}_{t=1}^T$, we have for Balanced OMD that
\begin{align*}
  \ccegap(\wb{\pi}) \le \cO\paren{ \frac{\max_{i\in[m]} \sqrt{H^3X_iA_i\iota T}}{T} } = \cO\paren{ \sqrt{H^3\paren{\max_{i\in[m]}X_iA_i}\iota /T} },
\end{align*}
with probability at least $1-\delta$, where $\iota\defeq \log(3H\sum_{i=1}^m X_iA_i/\delta)$ is a log factor. Choosing $T\ge \tO\paren{H^3\paren{\max_{i\in[m]}X_iA_i}\iota/\eps^2}$ ensures that the right-hand side is at most $\eps$. This shows part (a). A similar argument can be done for the Balanced CFR algorithm to show part (b).
\qed